\documentclass[]{article}

\usepackage[a4paper, total={6in, 8in}]{geometry}

\usepackage{times}  
\usepackage{helvet}  
\usepackage{courier}  
\usepackage[hyphens]{url}  
\usepackage{graphicx} 
\usepackage{caption} 
\usepackage{microtype}
\usepackage{graphicx}
\usepackage{subfigure}
\usepackage{booktabs} 

\usepackage{hyperref}

\usepackage{amsmath}
\usepackage{amssymb}
\usepackage{mathtools}
\usepackage{amsthm}

\usepackage{color}
\usepackage{xspace}

\usepackage{pgffor}

\usepackage[nameinlink]{cleveref}

\usepackage{algorithm}
\usepackage{algorithmic}

\usepackage{adjustbox}

\usepackage{nicefrac}

\usepackage{newfloat}
\usepackage{listings}
\DeclareCaptionStyle{ruled}{labelfont=normalfont,labelsep=colon,strut=off} 
\floatstyle{ruled}
\newfloat{listing}{tb}{lst}{}
\floatname{listing}{Listing}
\newcommand{\diff}{\mathop{}\!\mathrm{d}}								
\newcommand{\eps}{{\varepsilon}}										
\newcommand{\ball}{{\mathcal B}}										
\newcommand{\diam}{{\text{diam}}}										
\newcommand{\image}{{\text{Image}}}										
\DeclareMathOperator*{\argmin}{\text{arg\,min}}							
\newcommand{\inv}{\ensuremath{^{-1}}}									
\newcommand{\nrm}[1]{\left\lVert#1\right\rVert}							
\newcommand{\abs}[1]{\left\lvert#1\right\rvert}							
\newcommand{\floor}[1]{\left\lfloor#1\right\rfloor}						
\newcommand{\cl}[1]{\text{cl}\left({#1}\right)}							
\newcommand{\E}[2][{}]{\mathbb E_{#1}\left[#2\right]}					
\newcommand{\PP}[1]{\mathbb P\left[#1\right]}							
\newcommand{\low}{{\text{low}}}											
\newcommand{\up}{{\text{up}}}											
\newcommand{\ra}{\rightarrow}											
\newcommand{\la}{\leftarrow}											
\newcommand{\ie}{\unskip, i.\,e.,\xspace}								
\newcommand{\eg}{\unskip, e.\,g.,\xspace}								
\newcommand{\sut}{\text{s.\,t.\,}}										
\newcommand{\wrt}{w.\,r.\,t. \xspace}									

\let\emptyset\varnothing
\newcommand{\N}{{\mathbb{N}}}											
\newcommand{\Z}{{\mathbb{Z}}}											
\newcommand{\R}{{\mathbb{R}}}											
\newcommand{\state}{s}													
\newcommand{\State}{S}													
\newcommand{\states}{\mathbb S}											
\newcommand{\action}{a}													
\newcommand{\Action}{A}													
\newcommand{\actions}{\mathbb A}										
\newcommand{\Traj}{Z}													
\newcommand{\policy}{\pi}												
\newcommand{\policies}{\Pi}												
\newcommand{\transit}{p}												
\newcommand{\reward}{r}													
\newcommand{\Value}{V}													
\newcommand{\Advan}{\mathcal A}											
\newcommand{\W}{\ensuremath{\mathbb{W}}}								
\newcommand{\act}{{\text{act}}}											
\newcommand{\crit}{{\text{crit}}}										
\newcommand{\G}{\ensuremath{\mathbb{G}}}								
\newcommand{\K}{\ensuremath{\mathcal{K}}\xspace}						
\newcommand{\KL}{\ensuremath{\mathcal{KL}}\xspace}						
\newcommand{\Kinf}{\ensuremath{\mathcal{K}_{\infty}}\xspace}			
\newcommand{\T}{\mathcal T}												
\newcommand{\normpdf}{\ensuremath{\mathcal N}}							
\newcommand{\loss}{\mathcal L}											
\newcommand{\replay}{\mathcal R}										
\newcommand{\spc}{{\,\,}}												

\newtheorem{dfn}{Definition}

\newtheorem{prp}{Proposition}

\newtheorem{thm}{Theorem}
\newtheorem{crl}{Corollary}
\newtheorem{rem}{Remark}
\newtheorem{exm}{Example}

\newenvironment{prf}[1][\unskip]{
\textbf{Proof #1.~}}
{\hfill$\blacksquare$}

\newtheorem*{dfn*}{Definition}
\newtheorem*{asm*}{Assumption}
\newtheorem*{prp*}{Proposition}
\newtheorem*{lem*}{Lemma}
\newtheorem*{thm*}{Theorem}
\newtheorem*{crl*}{Corollary}
\newtheorem*{rem*}{Remark}

\newcommand{\Treplay}{{T_\replay}}										
\newcommand{\lrcrit}{\alpha_\crit}										
\newcommand{\lract}{\alpha_\act}										
\newcommand{\eplen}{T}													

\newcommand{\Noise}{B}													
\newcommand{\noise}{b}													

\newcommand{\robj}{\reward}												
\newcommand{\goaldist}[1][{}]{d_{\G#1}}									

\newcommand{\relprob}{P_{\text{relax}}}									
\newcommand{\relfact}{\lambda_{\text{relax}}} 							

\renewcommand{\Traj}{S}

\definecolor{dgreen}{rgb}{0.0, 0.5, 0.0}

\definecolor{gray}{rgb}{0.5, 0.5, 0.5}

\newcommand{\alglinelabel}{%
  \addtocounter{ALC@line}{-1}
  \refstepcounter{ALC@line}
  \label
}

\makeatletter
\newcommand{\pushright}[1]{\ifmeasuring@#1\else\omit\hfill$\displaystyle#1$\fi\ignorespaces}
\newcommand{\pushleft}[1]{\ifmeasuring@#1\else\omit$\displaystyle#1$\hfill\fi\ignorespaces}
\makeatother

\newcommand{\mquad}[1]{%
    \foreach \n in {1,...,#1} {%
        \qquad%
    }%
}

\usepackage[
	backend=biber,
	maxnames=2,
	style=authoryear,
	sorting=ydnt,
	giveninits=true,
	  doi=false,
	  url=false,
	  isbn=false	
	]{biblatex}
\newcommand{\citep}{\parencite}
\usepackage{subfiles}
\usepackage{titling}  

\title{An agent design with goal reaching guarantees for enhancement of learning}
\author{
    Pavel Osinenko\textsuperscript{\rm 1},
    Grigory Yaremenko\textsuperscript{\rm 1},
    Georgiy Malaniya\textsuperscript{\rm 1}, \\
    Anton Bolychev\textsuperscript{\rm 1},
    Alexander Gepperth\textsuperscript{\rm 2} \\
	\\
    \textsuperscript{\rm 1}Skoltech, Moscow, Russia\\
    \textsuperscript{\rm 2}Fulda University of Applied Sciences, Germany\\
    Corresponding author: \\
    Pavel Osinenko, \\
    \texttt{e-mail: p.osinenko@yandex.ru}
}

\begin{document}

\maketitle

\begin{abstract}
Reinforcement learning is commonly concerned with problems of maximizing accumulated rewards in Markov decision processes.
Oftentimes, a certain goal state or a subset of the state space attain maximal reward.
In such a case, the environment may be considered solved when the goal is reached.
Whereas numerous techniques, learning or non-learning based, exist for solving environments, doing so optimally is the biggest challenge.
Say, one may choose a reward rate which penalizes the action effort.
Reinforcement learning is currently among the most actively developed frameworks for solving environments optimally by virtue of maximizing accumulated reward, in other words, returns.
Yet, tuning agents is a notoriously hard task as reported in a series of works.
Our aim here is to help the agent learn a near-optimal policy efficiently while ensuring a goal reaching property of some basis policy that merely solves the environment.
We suggest an algorithm, which is fairly flexible, and can be used to augment practically any agent as long as it comprises of a critic.
A formal proof of a goal reaching property is provided.
Comparative experiments on several problems under popular baseline agents provided an empirical evidence that the learning can indeed be boosted while ensuring goal reaching property.
\end{abstract}

\tableofcontents

\section{Background and problem statement}
\label{sec_problem}

Consider the following Markov decision process (MDP):
\begin{equation}
	\label{eqn_mdp}
	\left(\states, \actions, \transit, \robj \right),
\end{equation}
where:
\begin{enumerate}
\item $\states$ is the \textit{state space}, assumed as a finite-dimensional Banach space of all states of the given environment;
\item $\actions$ is the \textit{action space}, that is a set of all actions available to the agent, assumed to be a compact topological space;
\item $\transit : \states \times \actions \times \states \ \rightarrow \ \R$ is the \textit{transition probability density function} of the environment, that is such a function that $\transit(\bullet \mid \state_{t}, \action_{t})$ is the probability density of the state $s_{t + 1}$ at step $t+1$ conditioned on the current state $\state_{t}$ and current action $\action_{t}$;
\item $\robj : \states \times \actions \rightarrow \mathbb{R}$ is the reward function of the problem, that is a function that takes a state $\state_{t}$ and an action $\action_{t}$ and returns the immediate reward $\robj_{t}$ incurred upon the agent if it were to perform action $\action_{t}$ while in state $\state_{t}$.
\end{enumerate}

Let $(\Omega, \Sigma, \mathbb P)$ be a probability space underlying \eqref{eqn_mdp}, and $\mathbb E$ be the respective expected value operator.
The problem of reinforcement learning is to find a policy $\policy$ from some space of admissible policies $\policies$ that maximizes 
\begin{equation}
	\label{eqn_value}
	\Value^{\policy}(\state) := \E[\Action \sim \policy]{\sum_{t = 0}^{\infty}\gamma^{t}\robj(\State_{t}, \Action_{t}) \mid \State_0=\state}, \state \in \states	
\end{equation}
for some $\gamma \in (0, 1]$ called a discount factor.
The problem may be concerned with a designated initial state $\state$, a distribution thereof, or even the whole state space.
A policy may be taken as a probability density function or as an ordinary function (cf. Markov policy).
The policy $\policy^*$ that solves the stated problem is commonly referred to as the optimal policy.
An agent will be referred to as a finite routine that generates actions from the observed states.

In some problems, the reward function may have the property that the environment reach the maximal reward being in a certain state or a set thereof \eg the optimal robot routing where the maximal reward is reached at the target state.
For instance, one may assume the reward to be zero at the origin, say, $\state=0$ or some set $\G \subset \states$ and zero action, and strictly negative outside.
When $\G$ is reached by the agent, the environment may be considered \eg solved.
Actually, even if the maximal reward is not reached exactly, a sufficiently large value thereof may be considered acceptable by the user.
So if \eg the target is $\state=0$ and $\robj(0, 0) = 0$, $\robj<0$ otherwise, one may wish to consider environment solved when \eg $\State \in \G, \G \subset \states, 0 \in \G$ in some suitable statistical sense \eg in probability.
It should be noted that the user needs not to actually specify $\G$, just the reward function is sufficient to state the problem, whence a set $\G$ occurs naturally.
In general, one may drop some state components from the reward thus restricting goal reaching to a subset of state variables. 
We do not focus on these details here and assume, and, for simplicity, $\G$ to be a compact neighborhood of the origin of the subspace spanned by the state variables of interest.
Notice though, that the environment may not necessarily be solved optimally, while respecting the problem of maximizing the value $\Value$.
For instance, a robot may reach the target pose by some agent that produces suboptimal route and possibly unnecessary action effort.
Thus, the problem \eqref{eqn_mdp} may be considered as the problem of reaching (and not leaving) $\G$ optimally in the described context.
Just reaching $\G$ may be done by various classical techniques such as proportional-derivative-integral regulator \citep{Johnson2005PIDcontrol,Wang2020PIDcontrolsys}, sliding-mode regulator \citep{Perruquetti2002SlidingModeCo,vaidyanathan2017applications}, flatness-based regulators \citep{Zhu2006Flatnessbased}, energy-based regulators \citep{Spong1996Energybasedco}, funnel regulators \citep{Berger2021Funnelcontrol}, gain schedulers \cite{Rotondo2017Advancesgains} etc., but doing so optimally is a non-trivial task for which reinforcement learning is currently the most promising approach.
Now, let the goal $\G$ be a compact neighborhood of the origin, and the distance to it be denoted by $\goaldist(\state) := \inf\limits_{\state' \in \G} \nrm{\state - \state'}$.
Let $\policies_0 \subset \policies$ denote the set of policies which satisfy the following $\eta$-improbable goal reaching property, $\eta \in [0,1)$, for $\policy_0 \in \policies$:
\begin{equation}
	\label{eqn_introstab}
		\forall \state_0 \in \states \spc \PP{\goaldist(\State_t) \xrightarrow{t \ra \infty} 0 \mid \Action_t \sim \policy_0} \ge 1 - \eta.
\end{equation}
Here $\goaldist(\State_t) \xrightarrow{t \rightarrow \infty} 0$ denotes that $0$ is a limit point of $\goaldist(\State_t)$.
That is, for all starting states, $\policies_0$ ensures that the probability of failing to reach the goal is no greater than $\eta$. 
Different probabilistic conditions can also be considered (see \Cref{sec_thms}).
Notice $\policies_0$ is problem-specific.
Furthermore, under the properties of the reward function stated above \ie being zero on $\G$ under zero action and strictly negative otherwise, and if $\gamma=1$, the optimal policy $\policy^*$ is in $\policies_0$.
Hence, a generic $\policy_0 \in \policies_0$ is suboptimal by default and a reinforcement learning agent, once reached the performance of such a $\policy_0$ in terms of the value $\Value^{\policy_0}$, will actually seek to improve on it.
Our quest here is to find a way to boost the learning of the said agent while preserving the goal reaching property of $\policy_0$.


\subsection{Contribution}
\label{sub_contrib}

This work presents a method of reinforcement learning, that once achieved a goal-reaching property or, alternatively, provided a policy $\policy_0 \in \policies_0$, preserves the said property while learning, thus enhancing it.
Sample efficiency is also improved due to avoidance of trials that fail probabilistic goal reaching condition and hence fail to improve the value.
A rigorous mathematical analysis of the said property's preservation is provided.
Although not necessarily seen as a direct alternative to such popular baselines as deterministic deep policy gradient (DDPG), soft actor-critic (SAC), Twin Delayed Deep Deterministic (TD3) agent, proximal policy optimization (PPO), the demonstration of superiority in learning dynamics is given compared to the said agents, and also to REINFORCE and vanilla policy gradient (VPG) on six environments: inverted pendulum, pendulum, two-tank system, non-holonomic three-wheel robot, omnidirectional robot (omnibot) and lunar lander.
The details follow in the main text below.
The code for this work is available at \url{https://github.com/osinenkop/regelum-calf}.

\section{Related work}
\label{sec_relwork}

Common reinforcement learning approaches to the above-stated problem include various tabular dynamic programming, episodic Monte-Carlo and online methods (see \eg \cite{Sutton2018ReinforcementL,Bertsekas2019Reinforcementl,Lewis2013Reinforcementl,Vrabie2012OptimalAdaptiv,Bouzy2006MonteCarloGo,Lazaric2007Reinforcementl,Vodopivec2017montecarlotre} for overviews).
Applications range from robotics \citep{Kumar2016Optimalcontrol,Borno2013TrajectoryOpti,Tassa2012Synthesisstabi,Surmann2020DeepReinforcem,Akkaya2019Solvingrubiks} to games such as Go, chess, Shogi (also known as Japanese chess) \citep{Silver2016Masteringgame,Silver2018generalreinfor}, and even complex video games such as StarCraft II \citep{Vinyals2019Grandmasterlev}.

Some simple policy gradients \citep{Baxter2001Infinitehorizo,Sutton1999PolicyGradient,Williams1992Simplestatisti,Kakade2001naturalpolicy,Peters2006Policygradient}, like REINFORCE, are almost direct adaptations of the stochastic approximation theory, whereas advanced methods employ sophisticated step size adjustments, critic learning, batch sampling techniques etc.

Among the most profound foundational methods of reinforcement learning are DDPG \citep{Silver2014DeterministicP} and PPO \citep{Schulman2017ProximalPolicy}.
Implementation tweaks such as value, gradient, and reward clipping, as well as reward and layer scaling, have been thoroughly studied in the context of trust region policy optimization and PPO \citep{Engstrom2019Implementationa}.
It has been shown that without these tweaks, PPO's performance is similar to that of trust region policy optimization.
In the experiments of this work, the described tweaks were employed, particularly layer scaling for the actor neural network output.
Furthermore, generalized advantage estimation (GAE) in VPG and PPO was used to address common pitfalls of value loss clipping \citep{Schulman2018HighDimensiona, Andrychowicz2020Whatmatterspo, Huang2024PPOClipAttain}.
Large batch sizes have consistently been reported to produce better agent neural network convergence up until recently \citep{Nikulin2023QEnsembleOffl, Akimov2023LetOfflineRL, Nikulin2023AntiExploratio, You2020LargeBatchOpt, Ginsburg2018LargeBatchTra}.
In all studied baselines of this work, samples from the replay buffers were taken large (larger than 256), as recommended.
Besides GAE, various tuning recommendations were also employed \citep{Furuta2021CoAdaptationA, Engstrom2020Implementation, Paine2020Hyperparameter, Yang2020OfflinePolicy, Henderson2019DeepReinforcem}.
Policy iteration format critic learning, inspired by \cite{Song2019Vmpopolicy}, was reported to help stabilize policy gradients without the need for excessive entropy tuning.
This approach was employed in the studied baselines as well.
Building upon DDPG advancements, the TD3 \citep{Fujimoto2018AddressingFunc} policy gradient algorithm enhances DDPG by incorporating techniques like delayed critic updates and target policy smoothing to mitigate the overestimation bias inherent in actor-critic methods.
SAC \citep{Haarnoja2018SoftActorCritic} emerged as a robust variation of DDPG, integrating entropy maximization to boost exploration.

Exploration in reinforcement learning presents its own challenges, as highlighted by \cite{Wang2020Strivingsimpli}.
They show that squashed actions can lead to poor exploration.
In the current work, this was avoided by using a truncated normal distribution for action selection to ensure more natural exploration behaviors.
This decision is informed by the discussion in \citep{Fujita2018Clippedaction}, which suggested that without such a distribution, one might encounter problematic Q-function estimations.
The significance of $N$-step returns in enhancing learning with large replay buffers should also be noted \citep{Fedus2020Revisitingfund}.
Their research suggested that uncorrected $N$-step returns could be beneficial for replay strategy.

It is notoriously hard to tune reinforcement learning agents in general, whereas implementation details play a particular role.
The CleanRL \citep{Huang2022CleanRl} Python package offers reliable and well-recognized implementations of popular reinforcement learning algorithms.
In this work, pure PyTorch agents were implemented following the recipes stated above besides agents taken from the CleanRL package with the best ones included in the results section.
In particular, TD3 and SAC agents were selected from CleanRL.
Furthermore, REINFORCE and VPG were not present in CleanRL, whereas we studied them as well.
Up-to-date performance on some popular problems relevant to this study should be mentioned.
For instance, PPO required about 500k agent-environment interaction steps on lunar lander to converge as reported \eg by \cite{Huang2024PPOClipAttain,Park2024UnveilingSigni}. 
For the inverted pendulum, around 80--100k steps total were needed for SAC to converge as reported in \citep{He2022Representation}.
The studies in \Cref{sec_results} indicate good alignment with these numbers, with the newly suggested agent greatly outperforming the baselines in learning dynamics, while beating or performing equally on final values.

\section{Suggested approach}
\label{sec_approach}

Recall the problem setup of \Cref{sec_problem}.
Let, for the sake of motivation, the reward function be negative-definite \ie $\robj(0,0)=0, \robj<0$ otherwise.
In general, it holds, by the virtue of the Hamilton-Jacobi-Bellman (HJB) equation, that $\forall t \ge 0 \spc \Value^*_{t+1} - \Value^*_t = -\robj^*_t$, where the optimal value satisfies $\Value^*_t := \E[\Action_t \sim \policy^*(\bullet \mid \state_t), \State_{t+1} \sim \transit( \bullet \mid \state_t, \Action_t ) ]{\robj(\state_t, \Action_t) + \Value^{\policy^*} \left( \State_{t+1} \right)}$, and $\robj^*_t := \E[\Action_t \sim \policy^*]{\robj(\state_t, \Action_t)}$.
Here, it was assumed that the problem was undiscounted \ie $\gamma = 1$ for motivational purposes.
Essentially, the HJB dictates that any specified goal $\G \supset 0$ is reached, generally in a probabilistic sense, by the agent following the optimal policy because (1) $\Value^*$ is finite for any initial state, (2) $\Value^*_t$ is strictly increasing due to $r$ being negative-definite.
If a model \ie a critic, is employed \eg a $w$-weighted deep neural network $\hat \Value^w(\state)$, then due to imperfections of learning, the goal reaching property may be lost, yet it is desired.
It is tempting to ask on retaining the said property.
Unfortunately, no such guarantee can be given under learned critic and policy models in practice.
The hypothesis of this work is that if the agent is supported by a policy $\policy_0 \in \policies_0$ (see \Cref{sec_problem}), such a guarantee can be provided, which in turn improves the learning.
As discussed in \Cref{sec_problem}, just finding a policy $\policy_0 \in \policies_0$ is not a difficult task which can be accomplished by various techniques \citep{Johnson2005PIDcontrol,Perruquetti2002SlidingModeCo,vaidyanathan2017applications,Zhu2006Flatnessbased,Berger2021Funnelcontrol}.
The challenge is to find the optimal policy $\policy^*$, which is the aim of the agent.
The idea here is to build on top of a nominal agent, which could in turn be based on any reinforcement learning algorithm with a critic, and to prioritize those critic updates that exhibit the property of the kind $\hat \Value^w(\state_{t+1}) - \hat \Value^{w^\dagger}(\state_t^\dagger) > 0$, where the weight tensor $w^\dagger, \state_t^\dagger$ are yet to be determined.
Namely, the critic update is done so as to try to satisfy the said property.
Should this be the case, the optimized weights are assigned to $w^\dagger$ and the current state is assigned to $\state_t^\dagger$.
Subsequently, the action generated by the actor's policy $\policy$ is applied to the environment.
If the critic update did not satisfy the condition $\hat \Value^w(\state_{t+1}) - \hat \Value^{w^\dagger}(\state_t^\dagger) > 0$, then the action generated by $\policy_0$ is applied.
Notice that even if there are to policies $\policy_0, \policy'_0$ with a goal reaching property \ie $\policy_0, \policy'_0 \in \policies_0$, it is not the case that an arbitrary switching between $\policy_0$ and $\policy'_0$ has the goal reaching property.
Hence, combining policies is generally not trivial.
We overcome this difficulty in the herein presented approach by ``freezing'' the critic $\hat \Value^{w^\dagger}$ if it failed to satisfy $\hat \Value^w(\state_{t+1}) - \hat \Value^{w^\dagger}(\state_t) > 0$.
In \Cref{sub_contrib}, we stated that such an approach was not to be necessarily seen as a direct alternative to the existing baselines.
This is because, as long as the agent comprises a critic part, the current approach puts no restrictions on the choice of the actor and critic loss functions $\loss_\act$ and $\loss_\crit$, respectively; optimization routines; schedules to update (online or episodic Monte-Carlo); models; replay buffer accumulation and sampling algorithms etc.
Now, besides $\hat \Value^w(\state_{t+1}) - \hat \Value^{w^\dagger}(\state_t) > 0$, the critic should also satisfy the property $- \hat \kappa_\up (\nrm{\state_t}) \le \hat \Value^{w}(\state_t) \le - \hat \kappa_\low(\nrm{\state_t})$, where $\hat \kappa_\up, \hat \kappa_\low: \R_{\ge0} \ra \R_{\ge0}$ are two functions which are zero at zero, monotonically increasing and tending to infinity.
Their choice is fairly loose \eg
\begin{equation}
\begin{aligned}
	\label{eqn_calf_quadkapps}
	& \hat \kappa_\low(x) = C_\low x^{2}, \ \hat \kappa_\up(x) = C_\up x^{2}, \\
	& 0 < C_\low < C_\up \text{ arbitrary},
\end{aligned}
\end{equation}
is acceptable.
One may \eg square the output of the critic network and add a small regularization while putting any a priori fixed lower and upper bound on the weights \ie the weights lie in some compact $\W$.
The reason is that we do not want the critic to output values arbitrarily large in magnitude and, respectively, small for non-zero states (equivalently, states outside $\G$).
Finally, the condition $\hat \Value^w(\state_{t+1}) - \hat \Value^{w^\dagger}(\state_t^\dagger) > 0$ should be cast into a non-strict one with any $\bar \nu > 0$ in place of the zero.
In \Cref{alg_calfstate}, we present a fairly generic, online form of the suggested approach.

\clearpage

\begin{algorithm}
	\begin{algorithmic}[1]
		\STATE {\bfseries Setup:} MDP, nominal agent details \eg networks, actor loss function $\loss_\act$, critic loss function $\loss_\crit$, and ${\bar \nu} > 0, \hat \kappa_\low, \hat \kappa_\up, \policy_0 \in \policies_0, \text{relaxation factor }1 > \relfact \geq 0$
		\STATE \textbf{Initialize}: $\state_0, w_0 \in \W$ \sut
		$$
		-\hat \kappa_\up(\nrm{\state_0}) \leq \hat \Value^{w_0}(\state_0) \leq -\hat \kappa_\low(\nrm{\state_0})
		$$
		\STATE $w^\dagger \gets w_{0}, \state^\dagger \gets \state_0, \relprob \gets \relfact$ \alglinelabel{algline_calfinit}  
		\FOR {$t := 1, \dots \infty$}
			\STATE Take action sampled from $\policy_{t-1}(\bullet \mid \state_{t-1})$, get state $\state_t$
			\STATE Try critic update \alglinelabel{algline_calfcrit}
			\[
				\begin{array}{lrl}
					w^* \gets & \argmin \limits_{w \in \mathbb W} & \loss_\crit(w) \\
					          & \sut                              & \substack{\hat \Value^w(\state_t) - \hat \Value^{w^\dagger}(\state^\dagger) \ge \bar \nu \\ -\hat \kappa_\up(\nrm{\state_t}) \le \hat \Value^{w}(\state_t) \le -\hat \kappa_\low(\nrm{\state_t})} \\
				\end{array}
			\]
			\IF{ solution $w^*$ found} \alglinelabel{algline_calfcheck}
				\STATE $\state^\dagger \gets \state_t, w^\dagger \gets w^*$
			\ENDIF	
			
			\STATE $q \gets $ sampled uniformly from $[0, 1]$.
			\IF{ solution $w^*$ found or $q < \relprob$} 
				\STATE	Update policy: \alglinelabel{algline_calfact}
				\[
					\policy_{t}(\bullet \mid \state_{t}) \la \argmin\limits_{\policy \in \policies} \loss_\act(\policy)
				\]		
				If desired, $\policy_{t}$ may be taken to produce a random action with probability $\eps>0$		 
				\ELSE
				\STATE $\policy_{t}(\bullet \mid \state_{t}) \gets \policy_0(\bullet \mid \state_t)$	
			\ENDIF	
			\STATE $\relprob \gets \relfact \relprob$
		\ENDFOR
	\end{algorithmic}
	\caption{Suggested goal reaching agent (state-valued critic).}
	\label{alg_calfstate}
\end{algorithm}

\pagebreak

The main goal reaching result related to \Cref{alg_calfstate} is formulated in \Cref{thm_calfstabmean}.
Let $\State^{\policy}_t(\state_0)$ denote the state trajectory emanating from $\state_0$ under a policy $\policy \in \policies$ \ie $\State^{\policy}_0(\state_0) = \state_0$ and for $t > 0$ it holds that $\State^{\policy}_{t}(\state_0) \sim \transit(\bullet \ | \ \State^{\policy}_{t - 1}(\state_0), \Action_{t - 1}(\state_0))$, with $\Action_{t - 1}(\state_0) \sim \policy(\bullet \mid \State^{\policy}_{t - 1}(\state_0))$. 

\begin{thm}
	\label{thm_calfstabmean}
	Consider the problem \eqref{eqn_value} under the MDP \eqref{eqn_mdp}.
	Let $\policy_0 \in \policies_0$ have the following goal reaching property for $\G \subset \states$ \ie
	\begin{equation}
	\label{eqn_stabilization}
		\forall \state_0 \in \states \spc \PP{\goaldist(\State^{\policy_0}_t(\state_0)) \xrightarrow{t \ra \infty} 0} \ge 1 - \eta, \eta \in [0,1).
	\end{equation}
	Let $\policy_t$ be produced by Algorithm 1 for all $t \ge 0$.
	Then, a similar goal reaching property is preserved under $\policy_t$ \ie
	\begin{equation}
	\label{eqn_stabilizationcalf}
		\forall \state_0 \in \states \spc \PP{\goaldist(\State^{\policy_t}_t(\state_0)) \xrightarrow{t \ra \infty} 0} \ge 1 - \eta.
	\end{equation}	
\end{thm}

\begin{figure*}[t]
\centering
\centerline{\includegraphics[width=\textwidth]{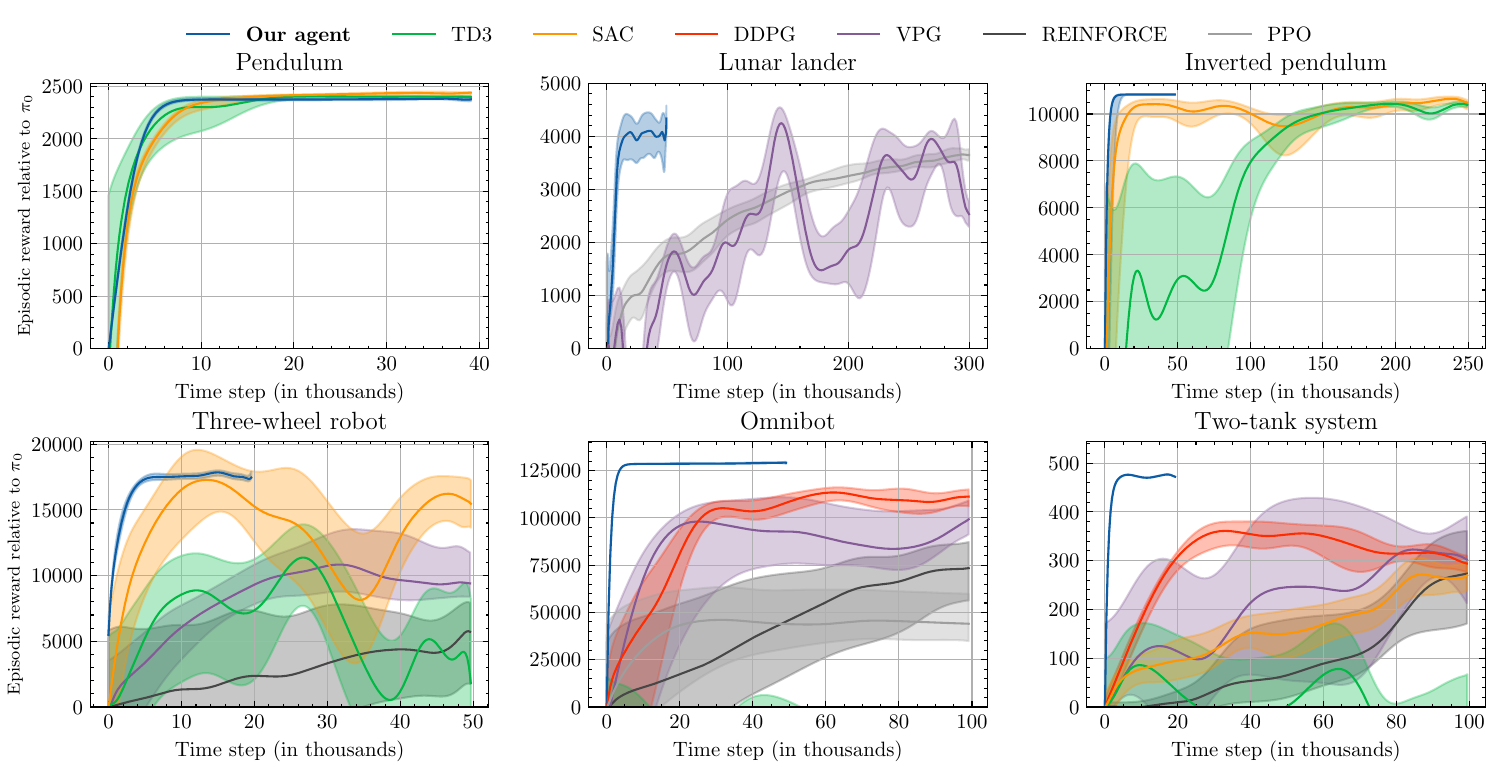}}
\caption{
	The plots show smoothed learning curves, representing accumulated episodic reward versus the number of environment
	steps. 
	The plots represent median performance over 10 random seeds relative to the baseline policy $\policy_0$, with the accumulated
	reward of $\policy_0$ subtracted for clarity. 
	Each plot is smoothed using a rolling median followed by Bezier interpolation. 
	Plots are truncated starting when policy gradient algorithms (PPO, DDPG, VPG, REINFORCE, SAC, TD3) reach $\policy_0$ performance in value, while for our agent, full learning plots are shown. 
	Full plots for all agents are located in \Cref{sec_rawresults}.
}
\label{fig_learncurves}
\end{figure*}

\begin{rem}
	\label{rem_Hoeffding}
	The policy $\policy_0$ may be either derived by techniques such as proportional-derivative-integral regulator \citep{Johnson2005PIDcontrol,Wang2020PIDcontrolsys}, sliding-mode regulator \citep{Perruquetti2002SlidingModeCo,vaidyanathan2017applications}, flatness-based regulators \citep{Zhu2006Flatnessbased}, energy-based regulators \citep{Spong1996Energybasedco}, funnel regulators \citep{Berger2021Funnelcontrol}, gain schedulers \citep{Rotondo2017Advancesgains} etc., or using environment runs over a distribution of start states $\state_0$ while employing statistical
	confidence bounds by \eg Hoeffding’s inequality in terms of the number of runs \citep{Hertneck2018Learningapprox}.
\end{rem}

\begin{rem}
	\label{rem_initrelprob}
	One may set $\relprob$ to some initial value in $[0, 1)$ instead of the assignment $\relprob \gets \relfact$ in line \ref{algline_calfinit} of \Cref{alg_calfstate}.
	Such an initial $\relprob$ and/or $\relfact$ may be considered episode-to-episode learnable parameters besides.
\end{rem}

\begin{rem}
	\label{rem_calfupdates}
	In \Cref{alg_calfstate}, the critic and actor updates (lines \ref{algline_calfcrit}, \ref{algline_calfact}) are fairly flexible.
	They may be varied to the user's choice without affecting the claim of \Cref{thm_calfstabmean}.
	In particular, if desired:
	\begin{enumerate}
		\item the constraints in line \ref{algline_calfcrit} may be shifted to the loss function $\loss_\crit$ as penalties;
		\item the constraints in line \ref{algline_calfcrit} may be omitted altogether and just checked in line \ref{algline_calfcheck} instead;
		\item the critic and policy updates (lines \ref{algline_calfcrit}, \ref{algline_calfact}) may be arbitrarily event triggered.
		In the given listing, the updates are done at every step $t$, but the they may also be triggered upon episode ends (cf. Monte-Carlo updates).
	\end{enumerate}
\end{rem}

\begin{rem}
	\label{rem_calfnogoal}
	Notice no direct specification of the goal $\G$ is required in \Cref{alg_calfstate}. 
\end{rem}

\begin{rem}
	\label{rem_calfrelaxprob}
	The variable $\relprob$ determines the probability of skipping an invocation of $\pi_0$ when it would otherwise have taken place.
	At each time step this probability decays with rate $\relfact < 1$, also referred to as ``the relaxation factor''.
	Larger values of $\relfact$ give the agent greater freedom during learning.
	While the choice of $\relfact$ does not in any way affect the goal reaching property, large values of $\relfact$ may cause large worst-case reaching times.
\end{rem}

\begin{rem}
	\label{rem_calfstabclaim}
	The claim of \Cref{thm_calfstabmean} can be rephrased in the following way: $\state_t$ will eventually get arbitrarily close to $\G$ with probability no less than that of $\policy_0$.
	This result is valuable, because it indicates that even a poorly trained reinforcement learning agent will still reach the goal with guarantee in all learning episodes.
	The hypothesis, which was supported by the experiments of this work, is that such a guarantee helps boost learning.
\end{rem}

\begin{rem}
	\label{rem_policymodel}
	In practice, the policy will usually be represented as some $\theta$-weighted model \eg softmax or Gaussian, in which case the policy update would actually mean the weight $\theta_t$ update.
	In case of a Markov policy, one directly updates the action $\action_t$.
\end{rem}

\section{Simulation experiments}
\label{sec_results}

Experiments were conducted on six environments -- pendulum, inverted pendulum, two-tank system, non-holonomic three-wheel robot, omnidirectional robot (omnibot) and lunar lander -- under seven agents, besides $\policy_0$: REINFORCE, VPG, DDPG, PPO, SAC, TD3 and the benchmarked agent by \Cref{alg_calfstate}.
We used dense quadratic rewards for each problem.
The MDPs for the tested environments were generated from discretized differential equations of dynamics. 
It was aimed to maximize the sum of rewards per episode. 
All results are presented in terms of the final sum of rewards when starting from a designated initial condition or a vicinity thereof. 
DDPG, VPG, REINFORCE, PPO, agent by \Cref{alg_calfstate} were implemented using PyTorch and CasADi \citep{Andersson2018CasADisoftware}, while SAC and TD3 were based on CleanRL \citep{Huang2022CleanRl}.
For the pendulum environment, we used the Gymnasium pendulum as a basis with modified parameters, in particular, mass and length, to simulate a pendulum close to a physical pendulum test stand by Quanser Inc. \citep{Inc2024QuanserRotary}.
It was observed that TD3 and SAC converged to a near-optimal policy in approximately same number of time steps as with the default Gymnasium environment.
The agent by \Cref{alg_calfstate} beat the baselines considerably in learning dynamics.
In experiments with the inverted pendulum, we removed episode termination upon pole fall as compared to the default Gymnasium environment thus rendering the problem slightly more challenging. 
The SAC algorithm required approximately 80,000 time steps to converge, whereas it took around 20,000 time steps to achieve convergence on the default Gymnasium environment \citep{Huang2022CleanRl}. 
Nonetheless, our agent achieved near-optimal performance in considerably fewer than 20,000 time steps.
A policy $\policy_0$ for each environment was designed by energy-based and proportional-derivative-integral regulators.
We aimed to pretrain every benchmarking agent (REINFORCE, VPG, DDPG, PPO, TD3, SAC) to achieve at least the value of $\policy_0$ in each experiment.
The learning curves of the agents, which succeeded in this task, are shown in \Cref{fig_learncurves}.
The accumulated rewards are shown relative to the one achieved upon using $\policy_0$.
Hence, the plots essentially show how much better than $\policy_0$ the agents performed.
Detailed descriptions of each problem, technical details of the agents' implementations and detailed evaluation data can be found in Technical Appendix, Sections C, D and E.
We observed learning convergence for lunar lander by PPO and VPG in about 500k agent-environment interaction steps -- similar results were reported \eg in \citep{Huang2024PPOClipAttain,Park2024UnveilingSigni}.
For the inverted pendulum environment, it was reported that SAC required less than 100k time steps to achieve convergence \citep{Huang2022CleanRl}.
Overall, our agent demonstrated superior or comparable final value across all problems, while greatly beating the baselines in learning dynamics, even when the baselines were pretrained to reproduce the value of $\policy_0$. 
\Cref{fig_learncurves} shows all learning interactions with the environment of our agent.

The code for this work is available at \href{https://github.com/osinenkop/regelum-calf}{https://github.com/osinenkop/regelum-calf}.



\subsection{Limitations}
\label{sub_limitations}

The reference to a policy $\policy_0$ is, strictly speaking, a limitation of \Cref{alg_calfstate}.
However, the are reasons to recommend \Cref{alg_calfstate}.
First, finding $\policy_0$ is commonly not difficult and can be done in a systemic way following one of the established techniques of \eg proportional-derivative-integral regulation, sliding-mode regulation, flatness-based regulation, energy-based regulation, funnel regulation, gain scheduling etc.
The biggest challenge is to optimally solve the environment, which is the target problem of reinforcement learning.
Second, it is generally not possible to guarantee goal reaching by an arbitrary agent in each episode hence many of those may be wasted.
It is asserted here that some insight is needed within the agent and a minimalistic one is the reference to $\policy_0$.
Notice that even just a combination of two goal reaching policies, say, $\policy_0, \policy'_0$, is not in general goal reaching as well.
\Cref{alg_calfstate} comprises of a non-trivial integration of $\policy_0$ into the agent by carefully saving the last successful critic update.
Third, our experiments on six problems showed that even pretrained agents, which achieved the value of the respective $\policy_0$, were mostly beaten by \Cref{alg_calfstate}.
It is thus claimed that at least for the problems, where a policy $\policy_0$ is accessible, agent design as per \Cref{alg_calfstate} may be seen beneficial.

\section{Conclusion}

This work presented an approach to improve agent learning via ensuring a goal reaching property for the stated return maximization problem.
While simply reaching a designated goal (or a goal that arises naturally from the stated reward rate) is not difficult and can be done via numerous techniques, including  proportional-derivative-integral regulator, sliding-mode regulator, flatness-based regulators, energy-based regulators, funnel regulators, gain schedulers etc., doing so optimally is the main challenge addressed by reinforcement learning.
The purpose of this work was to study how goal reaching guarantee could help the agent learning.
A fairly flexible algorithm for this sake was presented and benchmarked on six problems providing an empirical evidence for better learning under the said guarantee.
The algorithm should not necessarily be seen as a direct alternative to, say, proximal policy optimization, and can be built on top of a nominal agent.
Formal analysis of the algorithm was provided.

\setcounter{secnumdepth}{1} 

\begin{center}
{\LARGE
	\textbf{Appendix} to \\ ``An agent design \\ with goal reaching guarantees \\for enhancement of learning''
}

\vspace{1cm}

{\Large
    Pavel Osinenko\textsuperscript{\rm 1},
    Grigory Yaremenko\textsuperscript{\rm 1},
    Georgiy Malaniya\textsuperscript{\rm 1}, \\
    Anton Bolychev\textsuperscript{\rm 1},
    Alexander Gepperth\textsuperscript{\rm 2} \\
}

\vspace{2cm}

{\Large
    \textsuperscript{\rm 1}Skoltech, Moscow, Russia\\
    \textsuperscript{\rm 2}Fulda University of Applied Sciences, Germany\\
    Corresponding author: Pavel Osinenko, \texttt{e-mail: p.osinenko@yandex.ru}
}
\end{center}

\appendix


\part{Technical appendix}
\section{Formal analysis of the approach}
\label{sec_thms}

We will use Python-like array notation \eg $[0:T] = \{0, \dots, T-1\}$ or $\state_{0:T} = \{\state_0, \dots, \state_{T-1}\}$.
In particular, indexing as $0:\infty$ will refer to an infinite sequence starting at index zero.
The subscript $\ge 0$ in number set notation will indicate that only non-negative numbers are meant.
The notation ``$\cl{\bullet}$'', when referring to a set, will mean the closure.
For a set $\mathcal X$ in any normed space the number $\diam(\mathcal X)$ will mean the set diameter \ie $\sup\limits_{x, y \in \mathcal X} \nrm{x - y}$.
The notation $\ball_v(x)$ will mean a closed ball of radius $v$ around $x$ in a normed space.
When $x=0$, we will simply write $\ball_v$.
Let $\K, \Kinf: \R_{\ge0} \ra \R_{\ge0}$ denote the spaces of continuous, strictly increasing, zero-at-zero functions, and, additionally, unbounded in case of $\Kinf$.
We will also use the common notation of capital vs. small letters to distinguish between random variables and the values they attend when it is clear from context.
The notation ``$\bullet \sim \bullet$'' means the first argument is sampled from the distribution being the second argument.
To declare a variable, we will use the $:=$ sign and for a dynamic assignment \eg to an action in an algorithm, we will use the left arrow.
The notation ``$\floor{\bullet}$'' will mean taking the floor function \ie the largest integer not greater than the given number. 
The negation of an event will be denoted ``$\neg \bullet$''.

\newcommand{\Crefthmcalfstabmean}{Theorem 1}
\newcommand{\refalglinecalfcrit}{line 6:}
\newcommand{\Crefalgcalfstate}{Algorithm 1}
\newcommand{\Crefcrlcalfstabevent}{Corollary 1}

\newcommand{\reach}{\mathcal Z}											


\numberwithin{equation}{section}

\subsection{Recalls and definitions}
\label{sub_recalls}

Let us recall Algorithm 1.

\setcounter{algorithm}{0}

\begin{algorithm*}
	\begin{algorithmic}[1]
		\STATE {\bfseries Setup:} MDP, nominal agent details \eg networks, actor loss function $\loss_\act$, critic loss function $\loss_\crit$, and ${\bar \nu} > 0, \hat \kappa_\low, \hat \kappa_\up, \policy_0 \in \policies_0, \text{relaxation factor }1 > \relfact \geq 0$
		\STATE \textbf{Initialize}: $\state_0, w_0 \in \W$ \sut
		$$
		-\hat \kappa_\up(\nrm{\state_0}) \leq \hat \Value^{w_0}(\state_0) \leq -\hat \kappa_\low(\nrm{\state_0})
		$$
		\STATE $w^\dagger \gets w_{0}, \state^\dagger \gets \state_0, \relprob \gets \relfact$ \alglinelabel{algline_calfinit}  
		\FOR {$t := 1, \dots \infty$}
			\STATE Take action sampled from $\policy_{t-1}(\bullet \mid \state_{t-1})$, get state $\state_t$
			\STATE Try critic update \alglinelabel{algline_calfcrit}
			\[
				\begin{array}{lrl}
					w^* \gets & \argmin \limits_{w \in \mathbb W} & \loss_\crit(w) \\
					          & \sut                              & \substack{\hat \Value^w(\state_t) - \hat \Value^{w^\dagger}(\state^\dagger) \ge \bar \nu \\ -\hat \kappa_\up(\nrm{\state_t}) \le \hat \Value^{w}(\state_t) \le -\hat \kappa_\low(\nrm{\state_t})} \\
				\end{array}
			\]
			\IF{ solution $w^*$ found} \alglinelabel{algline_calfcheck}
				\STATE $\state^\dagger \gets \state_t, w^\dagger \gets w^*$
			\ENDIF	
			
			\STATE $q \gets $ sampled uniformly from $[0, 1]$.
			\IF{ solution $w^*$ found or $q < \relprob$} 
				\STATE	Update policy: \alglinelabel{algline_calfact}
				\[
					\policy_{t}(\bullet \mid \state_{t}) \la \argmin\limits_{\policy \in \policies} \loss_\act(\policy)
				\]		
				If desired, $\policy_{t}$ may be taken to produce a random action with probability $\eps>0$		 
				\ELSE
				\STATE $\policy_{t}(\bullet \mid \state_{t}) \gets \policy_0(\bullet \mid \state_t)$	
			\ENDIF	
			\STATE $\relprob \gets \relfact \relprob$
		\ENDFOR
	\end{algorithmic}
	\caption{Suggested goal reaching agent (state-valued critic).}
	\label{alg_calfstate}
\end{algorithm*}

\begin{rem}
	\label{rem_initrelprob}
	One may set $\relprob$ to some initial value in $[0, 1)$ instead of the assignment $\relprob \gets \relfact$ in line \ref{algline_calfinit} of \Cref{alg_calfstate}.
	Such an initial $\relprob$ and/or $\relfact$ may be considered learnable parameters.
\end{rem}

Our aim here is to prove Theorem 1.

First, let us make some recalls from the main text.

We consider the following Markov decision process (MDP):
\begin{equation}
	\label{eqn_mdp}
	\left(\states, \actions, \transit, \robj \right),
\end{equation}
where:
\begin{enumerate}
\item $\states$ is the \textit{state space}, assumed as a finite-dimensional Banach space of all states of the given environment;
\item $\actions$ is the \textit{action space}, that is a set of all actions available to the agent, assumed to be a compact topological space;
\item $\transit : \states \times \actions \times \states \ \ra \ \R$ is the \textit{transition probability density function} of the environment, that is such a function that $\transit(\bullet \mid \state_{t}, \action_{t})$ is the probability density of the state $s_{t + 1}$ at step $t+1$ conditioned on the current state $\state_{t}$ and current action $\action_{t}$;
\item $\robj : \states \times \actions \ra \mathbb{R}$ is the reward function of the problem, that is a function that takes a state $\state_{t}$ and an action $\action_{t}$ and returns the immediate reward $\robj_{t}$ incurred upon the agent if it were to perform action $\action_{t}$ while in state $\state_{t}$.
\end{enumerate}

Let $(\Omega, \Sigma, \mathbb P)$ be a probability space underlying \eqref{eqn_mdp}, and $\mathbb E$ be the respective expected value operator.
The problem of reinforcement learning is to find a policy $\policy$ from some space of admissible policies $\policies$ that maximizes 
\begin{equation}
	\label{eqn_value}
	\Value^{\policy}(\state) := \E[\Action \sim \policy]{\sum_{t = 0}^{\infty}\gamma^{t}\robj(\State_{t}, \Action_{t}) \mid \State_0=\state}, \state \in \states	
\end{equation}
for some $\gamma \in (0, 1]$ called a discount factor.

Let $\State^{\policy}_t(\state_0)$ denote the state trajectory emanating from $\state_0$ under a policy $\policy \in \policies$ \ie $\State^{\policy}_0(\state_0) = \state_0$ and for $t > 0$ it holds that $\State^{\policy}_{t}(\state_0) \sim \transit(\bullet \ | \ \State^{\policy}_{t - 1}(\state_0), \Action_{t - 1}(\state_0))$, with $\Action_{t - 1}(\state_0) \sim \policy(\bullet \mid \State^{\policy}_{t - 1}(\state_0))$.

\subsection{Proof of main theorem}
\label{sub_mainthmapdx}

\begin{thm}
	\label{thm_calfstabmean}
	Consider the problem \eqref{eqn_value} under the MDP \eqref{eqn_mdp}.
	Let $\policy_0 \in \policies_0$ have the following goal reaching property for $\G \subset \states$ \ie
	\begin{equation}
	\label{eqn_stabilization}
		\forall \state_0 \in \states \spc \PP{\goaldist(\State^{\policy_0}_t(\state_0)) \xrightarrow{t \ra \infty} 0} \ge 1 - \eta, \eta \in [0,1).
	\end{equation}
	Let $\policy_t$ be produced by Algorithm 1 for all $t \ge 0$.
	Then, a similar goal reaching property is preserved under $\policy_t$ \ie
	\begin{equation}
	\label{eqn_stabilizationcalf}
		\forall \state_0 \in \states \spc \PP{\goaldist(\State^{\policy_t}_t(\state_0)) \xrightarrow{t \ra \infty} 0} \ge 1 - \eta.
	\end{equation}	
\end{thm}

\begin{prf}

	Let $\xi_t$ be the indicator of $q < \relprob$ at step $t$ and $\xi_0 := 1$. 
	Notice that $\E{\sum_{i = 0}^{\infty}\xi_{i}} \leq \frac{1}{1 - \relfact}$, thus by Markov's inequality 
	\begin{equation}
		\forall C > 0 \spc \PP{\sum_{i = 0}^{\infty}\xi_{i} \geq C} \leq \frac{1}{C(1 - \relfact)},
	\end{equation}		
	which trivially implies
	\begin{equation}
		\PP{\sum_{i = 0}^{\infty}\xi_{i} = \infty} = 0.
	\end{equation}
    Now, let $\Omega^{\dagger} := \{ \omega \in \Omega \ | \ \sum_{i = 0}^{\infty}\xi_{i}[\omega] \neq \infty \}$, meaning the event of at most finitely many relaxed actions \ie the actions let through by the actor despite violation of the critic constraints.
	From this point on throughout this proof it is assumed that we are working in the probability space induced on $\Omega^{\dagger}$.    
    
    
	Recalling Algorithm \ref{alg_calfstate}, let us denote:
	\begin{equation}
	    \label{eqn_qdagger}
	    \begin{aligned}
	        & \hat \Lambda^\dagger := \hat \Lambda^{w^\dagger}(\state^\dagger).
	    \end{aligned}
	\end{equation}
	
	Next, we introduce:
	\begin{equation}
	    \label{calftimes}
	    \begin{aligned}
	        & \hat{\mathbb T}(\omega) := \{t \in \Z_{\ge 0} : \text{successful critic update} \}, \\
	        & \tilde{\mathbb T}(\omega) := \{t \in \Z_{\ge 0} : \xi_{t}[\omega] = 1 \}, \\
	    \end{aligned}
	\end{equation}
	The former set represents the time steps at which the critic succeeds and the corresponding induced action will fire at the next step.
	The latter set consists of the time steps after each of which the critic first failed, discarding the subsequent failures if any occurred.
	
	
	Now, let us define:
	\begin{equation*}
	    \begin{aligned}
	        & \hat \Lambda^\dagger_t :=  \begin{cases}
	                            \hat \Lambda_t, t \in \hat{\mathbb T},\\
	                            \hat \Lambda^\dagger_{t-1}, \text{ otherwise}.
	                        \end{cases}
	    \end{aligned}
	\end{equation*}
	
	Next, observe that there are at most
	\begin{equation}
	    \label{eqn_CALF2_critic_reaching}
	        \hat T := \max \left\{ \frac{ \hat \Lambda^\dagger_0 - \bar \nu}{\bar \nu}, 0 \right\}
	\end{equation}
	critic updates until the critic stops succeeding and hence only $\policy_0$ is invoked from that moment on. Hence $\hat{\mathbb T}(\omega)$ is a finite set.
	Furthermore, for $\omega \in \Omega^{\dagger}$ evidently $\tilde{\mathbb T}(\omega)$ is too a finite set.
	Notice $\hat T$ was independent of $\G'$ and in turn dependent on the initial value of the critic.

	
	Let $t^\dagger_{\text{final}}[\omega] = \sup \tilde{\mathbb T}[\omega] \cup \hat{\mathbb T}[\omega]$. 
	Now notice that since $\Action_{t^{\dagger}_{\text{final}} + t + 1} \sim \policy_0(\State_{t^{\dagger}_{\text{final}} + t})$, by assumptions of the theorem it holds that 
	\begin{equation}
	\forall \tau \geq 0 \spc \PP{\goaldist(\State_{t^{\dagger}_{\text{final}} + t}) \xrightarrow{t \ra \infty} 0 \spc | \spc t^{\dagger}_{\text{final}} = \tau} \geq 1 - \eta.	
	\end{equation}
	Now, evidently by the law of total probability
	\begin{multline}
		\PP{\goaldist(\State_{t^{\dagger}_{\text{final}} + t}) \xrightarrow{t \ra \infty} 0} = \\ \sum_{\tau=0}^{\infty}\PP{\goaldist(\State_{t^{\dagger}_{\text{final}} + t}) \xrightarrow{t \ra \infty} 0 \spc | \spc t^{\dagger}_{\text{final}} = \tau}\PP{t^{\dagger}_{\text{final}} = \tau} \geq \\ (1 - \eta)\sum_{\tau=0}^{\infty}\PP{t^{\dagger}_{\text{final}} = \tau} = 1 - \eta.	
	\end{multline}
	
	This concludes the proof.
\end{prf}

\begin{rem}
Since $\forall t\geq 0 \spc \goaldist(\State_t)$ is non-negative, $\goaldist(\State_t) \xrightarrow{t \ra \infty} 0$ is equivalent to $\liminf\limits_{t \ra \infty}\goaldist(\State_t)=0$.
\end{rem}

\begin{crl}
	If $\PP{\lim\limits_{t \ra \infty}\goaldist(\State_t) = 0 \spc \vert \spc \Action_{t} \sim \policy_0(\State_t)} \geq 1 - \eta$ and $\pi_0$ also ensures a uniform bound for $\goaldist(\State_t)$\ie
	\begin{equation}
		\exists \bar \state \spc : \spc \Action_{t} \sim \policy_0(\State_t) \implies \forall t\geq 0 \spc \goaldist(\State_t) \leq \bar s,
	\end{equation}
	then 
	\begin{equation}
		\Action_{t} \sim \policy_t(\State_t) \implies \limsup\limits_{t\ra\infty}\E{\goaldist(\State_t)} \leq \bar s\eta.
	\end{equation}
\end{crl}

\begin{proof}
	Let $\Action_{t} \sim \policy_t(\State_t)$ and let $\Omega_0$ denote the convergence event \ie $\lim\limits_{t \ra \infty}\goaldist(\State_t) = 0$ under $\Action_{t} \sim \policy_0(\State_t), t \ge 0$ if $\Omega_0$ took place.
	By law of total expectation:
	\begin{multline}
		\limsup\limits_{t\ra\infty}\E{\goaldist(\State_t)} = \\ \limsup\limits_{t\ra\infty}\E{\goaldist(\State_t) \spc \vert \spc \Omega_0}\PP{\Omega_0} + \limsup\limits_{t\ra\infty}\E{\goaldist(\State_t) \spc \vert \spc \neg \Omega_0}\PP{\neg \Omega_0} = \\ \eta\limsup\limits_{t\ra\infty}\E{\goaldist(\State_t) \spc \vert \spc \neg \Omega_0} \leq \eta \bar \state.
	\end{multline}
\end{proof}

The above corollary is of interest for environments with a compact state-space or those for which a restriction to a bounded subset can be easily ensured.
For instance, this includes all dissipative systems \eg electrical, mechanical, hydraulic.

\begin{crl}
	If $\policy_0$ is known to reach the goal in mean \ie
	\begin{equation}
		\Action_{t} \sim \policy_0(\State_t) \implies \lim\limits_{t \ra \infty}\E{\goaldist(\State_t)} = 0,
	\end{equation}
	then $\policy_0$ also has the 0-improbable goal reaching property.
\end{crl}

\begin{proof}
The law of total expectation yields a trivial proof by contradiction:
\begin{equation}
	\lim\limits_{t\ra\infty}\E{\goaldist(\State_t)} \geq \liminf\limits_{t\ra\infty}\E{\goaldist(\State_t) \spc \vert \spc  \neg \Omega_0}\PP{\neg \Omega_0},
\end{equation}
where $\neg \Omega_0$ is the event of not reaching the goal.
\end{proof}

For the next formal results, we need some definitions.

\begin{dfn}
	\label{dfn_kappaell}
	A continuous function $\beta : \R_{\ge0} \times \R_{\ge0} \ra \R_{\ge0}$ is said to belong to space $\KL$ if
	\begin{itemize}
		\item $\beta(v, \tau)$ is a $\K$ function with respect to $v$ for each fixed $\tau$. 
		\item $\beta(v, \tau)$ is a strictly decreasing function with respect to $\tau$ for each fixed $v$.
		\item For each fixed $v$ it holds that $\lim\limits_{\tau \ra \infty}\beta(v, \tau) = 0$.
	\end{itemize} 
\end{dfn}

If some of the state variables were dropped from consideration, one might restrict the state space to the one spanned by the state variables of interest.
The results derived herein are independent from it, so we simply assume $\G$ to be a compact neighborhood of the origin.
Let $\diam(\G)$ denote the number $\sup_{\state, \state' \in \G} \nrm{\state - \state'}$.

\begin{thm}
	Consider the problem \eqref{eqn_value} under the MDP \eqref{eqn_mdp}.
	Let $\policy_{0} \in \policies_0$ have the following strengthened goal reaching property for $\G \subset \states$ \ie
	\begin{equation}
		\exists \beta \in \KL \spc \forall \state_0 \in \states \spc \PP{\goaldist(\State^{\policy_0}_{t}(\state_0)) \leq \beta(\goaldist(\state_0), t)} \geq 1 - \eta.
	\end{equation}

Let $\bar \nu$ be such that $\hat\kappa_{\up}(\diam(\G)) < \frac{\bar \nu}{2}$. 	
	
	Let $\policy_t$ be produced by Algorithm 1 with the following modification:
	\begin{equation}
		\label{eqn_relprobdrop}
	 	\relprob \gets 0, \relfact \gets 0 \text{ when } \hat\kappa_\low(\nrm{\state_t}) \leq \bar \nu,
	\end{equation}
	Then,  $\policy_t$  ensures the goal reaching property \eqref{eqn_stabilizationcalf} and also uniform overshoot bound \ie
	\begin{multline}
	\forall \eps > 0  \spc \exists \delta > 0 \spc \forall \state \in \states \ \goaldist(\state) < \delta \implies \\ \spc \PP{\goaldist(\State_{t}) \xrightarrow{t \ra \infty} 0, \spc \sup\limits_{t} \goaldist(\State_{t}) \leq \eps} \geq 1 - \eta.
	\end{multline}
\end{thm}

\begin{prf}

	Let $\xi_t$ be the indicator of $q < \relprob$ at step $t$ and $\xi_0 := 1$. 
	Notice that $\E{\sum_{i = 0}^{\infty}\xi_{i}} \leq \frac{1}{1 - \relfact}$, thus by Markov's inequality 
	\begin{equation}
	\forall C > 0 \spc \PP{\sum_{i = 0}^{\infty}\xi_{i} \geq C} \leq \frac{1}{C(1 - \relfact)},
\end{equation}		
	which trivially implies
	\begin{equation}
	\PP{\sum_{i = 0}^{\infty}\xi_{i} = \infty} = 0.
	\end{equation}
    Now, let $\Omega^{\dagger} := \{ \omega \in \Omega \ | \ \sum_{i = 0}^{\infty}\xi_{i}[\omega] \neq \infty \}$, meaning the event of at most finitely many relaxed actions \ie the actions let through by the actor despite violation of the critic constraints.

	From this point on throughout this proof it is assumed that we are working in the probability space induced on $\Omega^{\dagger}$.    
	
	Recalling Algorithm \ref{alg_calfstate}, let us denote:
	\begin{equation}
	    \label{eqn_qdagger}
	    \begin{aligned}
	        & \hat \Lambda^\dagger := \hat \Lambda^{w^\dagger}(\state^\dagger).
	    \end{aligned}
	\end{equation}
	
	Next, we introduce:
	\begin{equation}
	    \label{calftimes}
	    \begin{aligned}
	        & \hat{\mathbb T}(\omega) := \{t \in \Z_{\ge 0} : \text{successful critic update} \}, \\
	        & \tilde{\mathbb T}(\omega) := \{t \in \Z_{\ge 0} : \xi_{t}[\omega] = 1 \}, \\
	    \end{aligned}
	\end{equation}
	The former set represents the time steps at which the critic succeeds and the corresponding induced action will fire at the next step.
	The latter set consists of the time steps after each of which the critic first failed, discarding the subsequent failures if any occurred.
	
	Now, let us define:
	\begin{equation*}
	    \begin{aligned}
	        & \hat \Lambda^\dagger_t :=  \begin{cases}
	                            \hat \Lambda_t, t \in \hat{\mathbb T},\\
	                            \hat \Lambda^\dagger_{t-1}, \text{ otherwise}.
	                        \end{cases}
	    \end{aligned}
	\end{equation*}
	
	Next, observe that there are at most
	\begin{equation}
	    \label{eqn_CALF2_critic_reaching}
	        \hat T := \max \left\{ \frac{ \hat \Lambda^\dagger_0 - \bar \nu}{\bar \nu}, 0 \right\}
	\end{equation}
	critic updates until the critic stops succeeding and hence only $\policy_0$ is invoked from that moment on. Hence $\hat{\mathbb T}(\omega)$ is a finite set.
	Furthermore, for $\omega \in \Omega^{\dagger}$ evidently $\tilde{\mathbb T}(\omega)$ is too a finite set.
	Notice $\hat T$ was independent of $\G'$ and in turn dependent on the initial value of the critic. 
	
	Let $t^\dagger_{\text{final}}[\omega] = \sup \tilde{\mathbb T}[\omega] \cup \hat{\mathbb T}[\omega]$. Now notice that since $\Action_{t^{\dagger}_{\text{final}} + t + 1} \sim \policy_0(\State_{t^{\dagger}_{\text{final}} + t})$.
	Therefore, by the assumptions of the theorem 
	\begin{equation}
		\forall \tau \geq 0 \spc \PP{\goaldist(\State_{t^{\dagger}_{\text{final}} + t}) \leq \beta(\goaldist(\State_{t^{\dagger}_{\text{final}}}), t) \spc | \spc t^{\dagger}_{\text{final}} = \tau} \geq 1 - \eta.	
	\end{equation}
	Now, evidently by the law of total probability
	\begin{multline}
		\PP{\goaldist(\State_{t^{\dagger}_{\text{final}} + t}) \xrightarrow{t \ra \infty} 0} = \\ \sum_{\tau=0}^{\infty}\PP{\goaldist(\State_{t^{\dagger}_{\text{final}} + t}) \xrightarrow{t \ra \infty} 0 \spc | \spc t^{\dagger}_{\text{final}} = \tau}\PP{t^{\dagger}_{\text{final}} = \tau} \geq \\ (1 - \eta)\sum_{\tau=0}^{\infty}\PP{t^{\dagger}_{\text{final}} = \tau} = 1 - \eta.	
	\end{multline}
	
	We now proceed to address the overshoot implication. 
	Let us assume without loss of generality that $\delta \leq \hat\kappa_{\up}^{-1}(\bar \nu) - \diam(\G) =: \bar \delta$.
	In that case 
	\begin{equation}
	\hat\kappa_{\low}(\nrm{\state_{0}}) \leq \hat\kappa_{\up}(\nrm{\state_{0}}) \leq \hat\kappa_{\up}(\goaldist(\state_{0}) + \diam(\G)) < \hat\kappa_{\up} \left(\delta + \diam(\G) \right) \leq \bar \nu,
	\end{equation}
	and hence $\nrm{\state_{0}} \leq \hat \kappa_{\up}^{-1}(\bar \nu)$. 
	In this case, \eqref{eqn_relprobdrop} implies that for a $\delta$ this small $\xi_{t}$ will always be equal to zero, hence $\tilde{\mathbb{T}} = \emptyset$. The very same expression also implies that the critic will be initialized to 
	\begin{equation}
		\hat \Lambda^\dagger_{0} \leq \hat \kappa_{\up}(\nrm{\state_{0}}) < \bar \nu,
	\end{equation}
	and therefore $\hat{\mathbb{T}} = \emptyset$. Since $\tilde{\mathbb{T}} \cup \hat{\mathbb{T}} = \emptyset$ it is true that $\forall t \geq 0 \spc \Action_{t}(\state_0) \sim \policy_0(\bullet \mid \State_t)$ and thus by the assumption of the theorem
	\begin{equation}
	\PP{\underbrace{\goaldist(\State_{t}(\state_0)) \leq \beta(\goaldist(\state_0), t)}\limits_{=:\Omega_0}} \geq 1 - \eta.
	\end{equation}
	The above already implies goal reaching.
	It remains to demonstrate that the overshoot is bounded.
	For simplicity in the following we assume that $\Omega_0$ is already known to have happened.
	
	It is known that any $\KL$ function can be bounded in the following way \citep[Lemma~8]{Sontag1998Commentsintegr}:
	\begin{equation}
	 \beta(\delta, t) \leq \kappa_{\text{dist}}(\delta)\kappa_{\text{time}}(e^{-t}), \text{where }\kappa_{\text{dist}}, \kappa_{\text{time}} \in \Kinf.
	\end{equation}
	Therefore $\goaldist(\State_{t}(\state_0)) \leq \beta(\delta, t) \leq \kappa_{\text{dist}}(\delta)\kappa_{\text{time}}(1)$, which in turn means that
	\begin{equation}
	\delta < \min\left(\kappa_{\text{dist}}^{-1}\left(\frac{\eps}{\kappa_{\text{time}}(1)}\right), \bar \delta \right) \implies \forall t \geq 0 \spc \goaldist(\State_{t}(\state_0)) < \eps.
	\end{equation}

\end{prf}

\begin{thm}
	\label{thm_calfuniform}
	Consider the problem \eqref{eqn_value} \ under the MDP \eqref{eqn_mdp}.
	
	Let $\policy_{0} \in \policies_0$ have the following strengthened goal reaching property for $\G \subset \states$ \ie
	\begin{equation}
	\label{eqn_strong_goal_reaching}
		\exists \beta \in \KL \spc \forall \state_0 \in \states \spc \PP{\goaldist(\State^{\policy_0}_{t}(\state_0)) \leq \beta(\goaldist(\state_0), t)} \geq 1 - \eta.
	\end{equation}

	Let $\transit (\bullet \ | \ \State_{t - 1}, \Action_{t - 1})$ be such that $\nrm{\State_{t}} \leq \rho(\State_{t - 1}, \Action_{t - 1})$ where $\rho \ : \ \states \times \actions \ra \mathbb{R}$ is some upper-semicontinuous function.
	
	Let $\bar \nu$ be such that $\hat\kappa_{\up}(\diam(\G)) < \bar \nu$. 
	
	Let $\policy_t$ be produced by Algorithm 1 with the following modifications:
	\begin{enumerate}
	\item $\relprob \gets 0, \relfact \gets 0, s^{\dagger} \gets 0$ when $\hat\kappa_\low(\nrm{\state_t}) \leq \bar \nu$,
	\item $\relprob \gets 0, \relfact \gets 0$ when $\nrm{\state_{t}} > \hat\kappa_{\low}^{-1}(\hat\kappa_{\up}(\nrm{\state_{0}}))$.
	\end{enumerate}
	
	Then,  $\Action_{t} \sim \policy_t$  ensures the following properties:
	\begin{enumerate}
	\item  \begin{multline}
	\label{eqn_uniform_reaching_time}
	\forall H > h > 0 \spc \exists \bar T^{\policy_0}(H, h) \spc \forall T_{\text{relax}} > 0\\
	 \PP{T^{\policy_t}(H, h) \ \leq T_{\text{relax}} + (\hat{T} + 1)\bar T^{\policy_0}(H, h)} \geq  (1 - \eta)(\relfact^{T_{\text{relax}}}; \spc \relfact)_{\infty},
	\end{multline} where $T^{\policy_t}(H, h) = \inf\{T_{\text{relax}} \in \N\cup\{0\} \ \vert \ \forall t \geq T_{\text{relax}} \ \goaldist(\State_t) \leq h\}$ and $(c; \spc q)_{n}$ denotes the $q$-Pochhammer symbol.
	Intuitively, $T_{\text{relax}}$ is an upper bound on the last time instant where a relax action was invoked.
	\item There exists $\delta(\bullet) \in \Kinf$, such that \begin{equation}
	\label{eqn_uniform_boundedness}
	\exists \delta_{0} \geq 0 \spc \PP{\sup\limits_{t\geq 0}\goaldist({\State_{t}}) \leq \delta^{-1}(\goaldist(\state_{0}) + \delta_{0})} \geq 1 - \eta.
	\end{equation}
	\end{enumerate}
\end{thm}

\begin{prf}
	Let $\xi_t$ be the indicator of $q < \relprob$ at step $t$ and $\xi_0 := 1$. 
	Notice that $\E{\sum_{i = 0}^{\infty}\xi_{i}} \leq \frac{1}{1 - \relfact}$, thus by Markov's inequality 
	\begin{equation}
		\forall C > 0 \spc \PP{\sum_{i = 0}^{\infty}\xi_{i} \geq C} \leq \frac{1}{C(1 - \relfact)},
	\end{equation}		
	which trivially implies
	\begin{equation}
		\PP{\sum_{i = 0}^{\infty}\xi_{i} = \infty} = 0.
	\end{equation}
    Now, let $\Omega^{\dagger} := \{ \omega \in \Omega \ | \ \sum_{i = 0}^{\infty}\xi_{i}[\omega] \neq \infty \}$, meaning the event of at most finitely many relaxed actions \ie the actions let through by the actor despite violation of the critic constraints.
	From this point on throughout this proof it is assumed that we are working in the probability space induced on $\Omega^{\dagger}$.    
    
	Now, let us consider $\G'$ to be a closed superset of $\G$, where the Hausdorff distance between $\G$ and $\G'$ is non-zero.
	Let $h$ denote the said distance.
	
	Recalling Algorithm \ref{alg_calfstate}, let us denote:
	\begin{equation}
	    \label{eqn_qdagger}
	    \begin{aligned}
	        & \hat \Lambda^\dagger := \hat \Lambda^{w^\dagger}(\state^\dagger).
	    \end{aligned}
	\end{equation}
	
	Next, we introduce:
	\begin{equation}
	    \label{calftimes}
	    \begin{aligned}
	        & \hat{\mathbb T}(\omega) := \{t \in \Z_{\ge 0} : \text{successful critic update} \}, \\
	        & \tilde{\mathbb T}(\omega) := \{t \in \Z_{\ge 0} : \xi_{t}[\omega] = 1 \}.
	    \end{aligned}
	\end{equation}
	The former set represents the time steps at which the critic succeeds and the corresponding induced action will fire at the next step.
	The latter set consists of the time steps after each of which the critic first failed, discarding the subsequent failures if any occurred.
	
	Now, let us define:
	\begin{equation*}
	    \begin{aligned}
	        & \hat \Lambda^\dagger_t :=  \begin{cases}
	                            \hat \Lambda_t, t \in \hat{\mathbb T},\\
	                            \hat \Lambda^\dagger_{t-1}, \text{ otherwise}.
	                        \end{cases}
	    \end{aligned}
	\end{equation*}
	
	Next, observe that there are at most
	\begin{equation}
	    \label{eqn_CALF2_critic_reaching}
	        \hat T := \max \left\{ \frac{ \hat \Lambda^\dagger_0 - \bar \nu}{\bar \nu}, 0 \right\}
	\end{equation}
	critic updates until the critic stops succeeding and hence only $\policy_0$ is invoked from that moment on. Hence $\hat{\mathbb T}(\omega)$ is a finite set.
	Furthermore, for $\omega \in \Omega^{\dagger}$ evidently $\tilde{\mathbb T}(\omega)$ is too a finite set.
	Notice $\hat T$ was independent of $\G'$ and in turn dependent on the initial value of the critic.
	
	Let $t^\dagger_{\text{final}}[\omega] = \sup \tilde{\mathbb T}[\omega] \cup \hat{\mathbb T}[\omega] + 1$.	
	
	Let $\states_0 := \{\state \in \states \spc \vert \spc \nrm{\state} \leq \hat\kappa_{\low}^{-1}(\hat\kappa_{\up}(\nrm{\state_{0}})) \}$. Notice that by the assumptions of the theorem
	\begin{equation}
		t \in \hat{\mathbb T}\cup\tilde{\mathbb{T}} \implies \goaldist({\State_{t + 1}}) \leq \nrm{\State_{t + 1}} \leq \sup_{\state \in \states_0, \spc \action \in \actions}\rho(\state, \action) =: \bar \rho.
	\end{equation}
	
	Existence of $\bar \rho$ can be inferred from the extreme value theorem.
	
	It is evident that
	\begin{equation}
	\PP{\sup\limits_{t \geq 0}\goaldist(\State'_{t}(S_{t^{\dagger}_{\text{final}}}(\state_0))) \leq \beta(\bar \rho, 0)} \geq 1 - \eta,
	\end{equation}
	thus  setting
	\begin{equation}
	\delta^{-1}(\goaldist(\State_{t}) + \delta_{0}) := \beta\left(\sup\limits_{\action \in \actions, \spc \nrm{\state} \leq \hat\kappa_{\low}^{-1}(\hat\kappa_{\up}(\goaldist(\state_{0}) + \diam(\G)))}\rho(\state, \action), 0\right)
	\end{equation} concludes the proof of \eqref{eqn_uniform_boundedness}.
	
		Assumption \eqref{eqn_strong_goal_reaching} implies that for $\policy_0$ there exists a uniform reaching time $T^{\policy_0}(H, h): = \beta^{-1}(H, \bullet)$ that satisfies:
	
	 \begin{equation}
	 \forall H > h \ : \ \goaldist(\state_{0}) \leq H \implies \PP{\forall t \geq T^{\policy_0}(H, h) \spc \goaldist(\State'_{t}(\state_{0})) < h} \geq 1 - \eta.
	\end{equation}

	Notice that 
	\begin{equation}
	\goaldist(\nrm{\state_{t}}) \leq \hat \kappa_{\low}^{-1}(\bar \nu) - \diam(\G) \implies \hat \kappa_{\low}(\nrm{\state_{t}}) \leq \bar \nu.
	\end{equation}
	
	Let $\mathcal T^{\policy_0}(t, T)$ denote the beginning of the longest sequence of consecutive invocations of $\policy_0$ on $[t, T]$.
	
	Now note that if $\forall t \geq T_{\text{relax}} \ \xi_{t} = 0$ then $[T_{\text{relax}}, T_{\text{relax}} + (\hat{T} + 1)T^{\policy_0}(\goaldist(\state_{0}), h)]$ there will be at least one sequence of $T^{\policy_0}(\goaldist(\state_{0}), h)$ consecutive invocations of $\policy_0$. Therefore if $h \leq \kappa_{\low}^{-1}(\bar \nu) - \diam(\G)$, then 
	\begin{multline}
		\PP{\forall t \geq \mathcal T \ \goaldist(\State_{t}) \leq \beta(\goaldist(\State_{\mathcal T}), t)} \geq (1 - \eta)\prod_{i=T_{\text{relax}}}^{\infty}(1 - \relfact^{i}) \\
		= (1 - \eta)(\relfact^{T_{\text{relax}}}; \spc \relfact)_{\infty}, \\ 
		\text{ where } \mathcal T := \mathcal T^{\policy_0}(T_{\text{relax}}, T_{\text{relax}} + (\hat{T} + 1)T^{\policy_0}(\goaldist(\state_{0}), h)). 
	\end{multline}
	Note that $T^{\policy_t}(H, \min(h, \kappa_{\low}^{-1}(\bar \nu) - \diam(\G))) \geq T^{\policy_t}(H, h)$ and also that $\mathcal T \leq T_{\text{relax}} + \hat{T}T^{\policy_0}(\goaldist(\state_{0}), h)$. From this it follows that
	\begin{multline}
	\PP{T^{\policy_t}(H, h) \leq T_{\text{relax}} + (\hat{T} + 1)T^{\policy_0}(H, \min(h, \kappa_{\low}^{-1}(\bar \nu) - \diam(\G)))} \geq \\ (1 - \eta)\prod_{i=T_{\text{relax}}}^{\infty}(1 - \relfact^{i}) = (1 - \eta)(\relfact^{T_{\text{relax}}}; \spc \relfact)_{\infty},
	\end{multline}
	which allows us to set $\bar T^{\policy_0}(H, h):= T^{\policy_0}(H, \min(h, \kappa_{\low}^{-1}(\bar \nu) - \diam(\G)))$ and conclude the proof of the theorem.
\end{prf}

\begin{rem}
Notice that $\delta_{0} = 0$ if $\mathbb G$ is invariant under $\policy_0$.
\end{rem}

\begin{rem}
If $\relfact = 0$, then the reaching time bound becomes deterministic \\ $T^{\policy_t}(H, h) \leq (\hat{T} + 1)T^{\policy_0}(H, \min(h, \kappa_{\low}^{-1}(\bar \nu) - \diam(\G)))$. The latter of course occurs with probability $1 - \eta$.
\end{rem}

\subsection{On $\omega$-uniform convergence moduli}
\label{sub_kappaell}

The following technical result will be needed later on. 
It may be considered as a generalization of the classical results on existence of $\KL$ convergence certificates for uniformly globally stable systems by \cite{Lin1996smoothconverse}.

\begin{prp}
	\label{prp_lim2kappa}
	Consider a function $\Phi: \R^n \times \R_{\ge 0} \ra \R_{\ge 0}$ with the property that
	\begin{equation}
		\label{eqn_genericlim}
		\begin{aligned}
			& \forall \state \in \R^n \spc \lim_{\tau \ra \infty} \Phi(\state, \tau) = 0, \\
			& \forall \eps \ge 0 \spc \exists \delta \geq 0 \spc \nrm{\state} \le \delta \implies \forall \tau \ge 0 \spc \Phi(\state, \tau) \le \eps.
		\end{aligned}		
	\end{equation}
	where $\delta$ is unbounded as a function of $\eps$ and the convergence of the limit is assumed compact with respect to $\state$ \ie
	\begin{equation}
		\label{eqn_comapctconv}
			\forall \text{ compact } \states_0, \eps>0 \spc \exists \T>0 \spc\sut \forall \tau \ge \T, \state \in \states_0 \spc \Phi(\state, \tau) \le \eps.
	\end{equation}
	
	Then, there exists a $\KL$ function $\beta$ with the following property:
	\begin{equation}
		\label{eqn_exitskl}
		\forall \state \in \R^n, \tau \ge 0 \spc \Phi(\state, \tau) \le \beta(\nrm{\state}+C_0, \tau),
	\end{equation}
	where $C_0 \ge 0$ is a constant that equals zero if $\delta$ is positive outside zero as a function of $\eps$.
\end{prp}

\begin{prf}
	First, we may assume that $\delta$ is a non-decreasing function of $\eps$.
	To see this, fix $0 < \eps_1 \le \eps_2$.
	Then, 
	\[
		\nrm{\state} \le \delta(\eps_1) \implies \forall t \ge 0 \spc \Phi(\state, \tau) \le \eps_1,
	\]
	from which it trivially follows that $\Phi(\state, \tau) \le \eps_2$ also.
	Hence, redefining $\delta(\eps) := \sup\limits_{\eps' \in [0, \eps]}\delta_\text{old}(\eps')$ is valid without loss of generality.	
	Second, $\T_{\states_0}$, for a fixed compact set $\states_0$ is a non-increasing function of $\eps$.
	Again, fix some $0 < \eps_1 \le \eps_2$ and observe that, for $\state$ in $\states_0$,
	\[
		\tau \ge \T_{\states_0}(\eps_1) \implies \Phi(\state, \tau) \le \eps_1
	\]
	implies trivially $\Phi(\state, \tau) \le \eps_2$ whenever $\tau \ge \T_{\states_0}(\eps_2)$.
	Now, let us consider $\T$ as a function of two arguments -- radius $v$ of the minimal ball containing the given compact set $\states_0$ and $\eps$ -- for convenience.	
	Then, it is a non-increasing function in the first argument.
	Yet again, fix any $0 < \eps_1 \le \eps_2$ and $v >0$.
	Take two compact sets $\states_0, \states_0'$ with $\states_0 \subseteq \states_0'$ and $\states_0' \subseteq \ball_v, v>0$.
	Then, for any $\eps>0$, we have
	\[
		\state \in \states_0', \tau \ge \T(v, \eps) \implies \Phi(\state, \tau) \le \eps
	\]
	and, simultaneously,
	\[
		\state \in \states_0, \tau \ge \T(v, \eps) \implies \Phi(\state, \tau) \le \eps.
	\]	
	Furthermore, we redefine $\delta$ to be continuous, strictly increasing and positive on $[v_0, \infty)$, where $v_0 := \sup \cl{\mathcal E^\delta_0}$ and $\mathcal E^\delta_0:= \{ \eps : \delta(\eps) = 0 \}$.

	To see this, observe that $\delta$ is positive everywhere outside zero, non-decreasing and bounded by one.
	Hence, $\delta$ is Riemann integrable on the unit interval precisely because the sets of discontinuities of non-decreasing functions have measure zero on compact domains.
	Consider a function $\hat \delta$ defined by
	\[
		\hat \delta(\eps) := \frac{1}{\eps + 1}\int_{0}^{\eps} \delta(\eps') \diff \eps'.
	\]
	Evidently, $\hat\delta$ is continuous. 
	Notice, that $\hat\delta(\eps)$ is zero whenever $\delta(\eps)$ is zero and for $\eps - v_{0} > h > 0$ it holds due to $\delta$ being non-decreasing that
	\begin{equation}
	\hat\delta(\eps) = \frac{1}{\eps + 1}\int_{0}^{\eps} \delta(\eps') \diff \eps' \geq \frac{1}{\eps + 1}\int_{\eps - h}^{\eps} \delta(\eps') \geq \frac{1}{\eps + 1}h\delta(\eps - h) > 0.
	\end{equation}
	Therefore $\hat\delta(\eps)$ is positive whenever $\delta(\eps)$ is positive except possibly at $v_0$. 
	Now, note that assuming $\eps > 0$ by mean value theorem we have
	\begin{equation}
	 \hat\delta(\eps) = \frac{1}{\eps + 1}\int_{0}^{\eps} \delta(\eps') \diff \eps' \leq \frac{1}{\eps}\int_{0}^{\eps} \delta(\eps') \diff \eps'= \delta(\hat\eps) \leq \delta(\eps), \text{ where }\hat\eps \in [0, \eps].
	\end{equation}
	Since $\delta(0) = \hat\delta(0)$, it is true that $\hat\delta$ bounds $\delta$ from below.
	Also, for $\eps \geq 1$ by mean value theorem
	\begin{equation}
	\hat\delta(\eps) = \frac{1}{\eps + 1}\int_{0}^{\eps} \delta(\eps') \geq \frac{1}{2\eps}\int_{0}^{\eps} \delta(\eps') \geq \frac{1}{4} \frac{1}{\frac{1}{2}\eps}\int_{\frac{\eps}{2}}^{\eps} \delta(\eps') \geq \frac{1}{4}\delta\left(\frac{\eps}{2}\right),
	\end{equation}
	therefore $\hat\delta$ is unbounded.
	Finally, for $h > 0$, $\eps \notin \mathcal E^\delta_0$
	\begin{align}
		& h\int_{0}^{\eps}\delta(\eps')\diff \eps'\leq h\eps\delta(\eps) < h(\eps + 1)\delta(\eps) \leq (\eps + 1)\int_{\eps}^{\eps + h}\delta(\eps') \diff \eps'\\
		& \implies (\eps + 1)\int_{0}^{\eps}\delta(\eps')d\eps' + h\int_{0}^{\eps}\delta(\eps')\diff \eps' < \\
		& \mquad{2} (\eps + 1)\int_{0}^{\eps}\delta(\eps')\diff \eps' + (\eps + 1)\int_{\eps}^{\eps + h}\delta(\eps') \diff \eps'\\
		& \implies (\eps + 1 + h)\int_{0}^{\eps}\delta(\eps')\diff \eps' < (\eps + 1)\int_{0}^{\eps + 1 + h}\delta(\eps')\diff \eps'\\
		& \implies (\eps + 1)(\eps + 1 + h)\hat\delta(\eps) < (\eps + 1)(\eps + 1 + h)\hat\delta(\eps + h)\\
		& \implies \hat\delta(\eps) < \hat\delta(\eps + h),
	\end{align}
	therefore $\hat\delta$ is strictly increasing outside of $\cl{\mathcal E^\delta_0}$.	
	
	Let us now fix an $\state \in \R^n, \state \ne 0$. 
	Then, by the compact convergence \eqref{eqn_comapctconv}, it holds that
	\[
		\forall \state' \in \ball_{\nrm{\state}}, \eps > 0, \tau \ge \T(\nrm{\state}, \eps) \spc \Phi(\state', \tau) \le \eps.
	\]
	Let $\eps_x := \delta\inv(\nrm{\state})$ which is possible since $\delta$ is now assumed continuous and strictly increasing outside $\cl{\mathcal E^\delta_0}$.
	Notice $\eps_x = v_0$ if $\state=0$ by the above redefinition of $\delta$.
	In particular, it is zero if the original $\delta$ was positive outside zero.
	From the overshoot bound condition, second line of \eqref{eqn_genericlim}, it holds that:
	\[
		\forall \state' \in \ball_{\nrm{\state}}, \tau \ge 0 \spc \Phi(\state', \tau) \le \eps_x.
	\]
	Therefore, it is safe to assume that $\T(\nrm{\state}, \eps)$, as a function of $\eps$, has a compact support since $\eps_x$ is never exceeded by $\Phi$ as shown above.
	
	Now, define inductively the numbers $\eps_j, j \in \N$ starting with $\eps_1 := \eps_x$:
	\[
		\eps_j := 
		\begin{cases}
			\sup_{\tau \in [j-1, j]} \cl{\{ \mathcal E^\T_{\state,j} \}}, \text{ if } \mathcal E^\T_{\state,j} := \{ \eps : \T(\nrm{\state}, \eps) \in [j-1,j] \} \ne \emptyset, \\
			\eps_{j-1}, \text{ otherwise}.
		\end{cases}
	\]	 
	By the convergence property \eqref{eqn_genericlim}, it holds that $\limsup\limits_{\eps \ra 0} \T(\nrm{\state}, \eps) = 0$ and $\T$ is defined for any $\eps$.
	Hence, $\T$ may not ``blow up'' into infinity at any $\eps>0$ and so the above construction of the numbers $\eps_j$ is valid.

	Now, let $\beta'_v, v = \nrm{\state}$ be a function, such that its graph is a polygonal chain specified by \\
	$(0, \eps_1), (1, \eps_1), (2, \eps_2), (3, \eps_3) \dots$ \ie
	\begin{equation}
	\beta'_v(\tau) = \eps_{1} + \sum_{i = 1}^{\infty}(\eps_{i + 1} - \eps_{i})\chi(\tau - i), \text{ where } \chi(\tau) = \frac{1}{2}(1 + \lvert \tau \rvert - \lvert \tau - 1 \rvert).
	\end{equation}
	Evidently, $\beta'_v$ is continuous.
	A setting $\beta'_v \mapsto \beta'_v + C_1 v e^{- C_2 \tau}$, with $C_1, C_2 >0$ arbitrary, makes it also strictly decreasing.
	Notice that any other function $\T'(\nrm{\state}, \eps)$, defined for all $\eps > 0$, with $\T' \ge \T$ and $\limsup\limits_{\eps \ra 0} \T'(\nrm{\state}, \eps) = 0$ is a valid certificate for the convergence property as per \eqref{eqn_comapctconv}.	
	We thus see that the inverse of $\beta'_v$, which exists, is such a valid certificate. Now, let

	\begin{align}
		& \beta(v, \tau) := (v - \floor{v})\beta_{\floor{v} + 2}(v, \tau) + (\floor{v} + 1 - v)\beta_{\floor{v} + 1}(v, \tau), \\
		& \text{where } \beta_i(v, \tau) := 2^i\max\limits_{j = 1}^i \beta^*_j(v, \tau), \\
		& \text{where } \beta^*_i(v, \tau) := \min(\beta'(i, \tau), \xi(v)), \\
		& \xi(v) := \left\{\begin{array}{ll}
		                  v, \text{ if }v < \delta^{-1}(0),\\
		                  \delta^{-1}(v - \delta^{-1}(0)), \text{ otherwise.}
		                \end{array}\right.
	\end{align}
	It is true that $\beta'(i, \bullet)$ provides a valid upper bound on $\Phi(\state, \tau)$ not only for $\nrm{\state} = i$, but also for $\nrm{\state} < i$. 
	However the largest overshoot is no greater than $\xi(\nrm{\state} + \delta^{-1}(0))$, thus $\min(\beta'(i, \tau), \xi(\nrm{\state} + \delta^{-1}(0)))$ sharpens the former bound. 
	Evidently $\beta^*_{i}$ are bounded above by $\beta_{i}$ and therefore $\nrm{\state} \leq i$ implies $\beta_{i}(\nrm{\state} + \delta^{-1}(0), \tau) \geq \Phi(\state, \tau)$. 
	Finally, $\beta(\nrm{\state} + \delta^{-1}(0), \bullet)$ is a convex combination of $\beta_{\lfloor \nrm{\state} \rfloor + 1}(\nrm{\state} + \delta^{-1}(0), \bullet)$ and $\beta_{\lfloor \nrm{\state} \rfloor + 2}(\nrm{\state} + \delta^{-1}(0), \bullet)$, which are both upper bounds for $\Phi(\state, \tau)$. 
	The latter thus also applies to $\beta(\nrm{\state} + \delta^{-1}(0), \bullet)$.
	
	Note that $\beta_i$ are composed of continuous functions and therefore are continuous themselves. 
	At the same time, when $\beta$ is restricted to an arbitrary segment $[l - 1, l + 1], \ l \in \mathbb{N},$ it can be represented as $\chi_3(v - l + 1, \beta_{l}(v,\tau), \beta_{l + 1}(v,\tau), \beta_{l + 2}(v, \tau))$ where $\chi_3(t, k_{1}, k_{2}, k_{3}) = k_{1} + (k_{2} - k_{1})\chi(t - 1) + (k_{3} - k_{2})\chi(t - 2)$.
	Since $\chi$ and $\beta_{i}$ are continuous, then $\beta$ is continuous on an arbitrary segment $[l - 1, l + 1], \ l \in \mathbb{N},$ and thus $\beta$ is continuous on the entirety of its domain.
	Notice, that $\beta^*_{i}$ are non-decreasing w.r.t $v$ and thus so are $\beta_{i}$.
	Now, let $0 < v < v' < \lfloor v\rfloor + 1$ and observe that
	\begin{multline}
		\beta(v', \tau) - \beta(v, \tau) \geq (v'- v)\beta_{\lfloor v \rfloor + 2}(v, \tau) + (v - v' )\beta_{\lfloor v\rfloor + 1}(v, \tau) = \\ 
		(v' - v)(\beta_{\lfloor v \rfloor + 2}(v, \tau) - \beta_{\lfloor v\rfloor + 1}(v, \tau)) \geq (v' - v)(2\beta_{\lfloor v \rfloor + 1}(v, \tau) - \beta_{\lfloor v\rfloor + 1}(v, \tau)) = \\ 
		(v' - v)\beta_{\lfloor v \rfloor + 1}(v, \tau) \geq (v' - v)\min(\beta'(\lfloor v \rfloor + 1, \tau), \delta^{-1}(v)) > 0.
	\end{multline}
	By continuity, this also extends to $v = 0$, therefore $\beta$ is increasing \wrt $v$.
	Also note that for $l \in \mathbb{N}$
	\begin{equation}
		\beta(l, \tau) = \beta_{l + 1}(l, \tau) \geq 2^{l + 1} \min(\beta'(2, \tau), \delta^{-1}(1)),
	\end{equation}
	thus $\beta$ is unbounded \wrt $v$ for every $\tau$.
	
	It is true that $v = 0$ if and only if $\beta(v, \tau) = 0$.
	Indeed,
	\begin{align}
		& \beta(0, \tau) = 2 \min(\beta'(1, \tau), \delta^{-1}(0)) = 0; \\ 
		& v > 0 \implies \beta(v, \tau) \geq \beta'(\lfloor v\rfloor + 1, \tau) > 0  \ \lor \ \beta(v, \tau) \geq \delta^{-1}(v) > 0.
	\end{align}
	Trivially for $\alpha > 0, v > 0$, function $\alpha\beta_{i}(v, \bullet)$ is decreasing, positive and tending to $0$. 
	For whole values of $v$, it is true that $\beta(v, \bullet) = \beta_{v + 1}(v, \bullet)$, otherwise $\beta(v, \bullet)$ can be represented as $\alpha_{1}\beta_i(v, \bullet) + \alpha_{2}\beta_j(v, \bullet)$, thus $\beta(v, \tau)$ is decreasing, positive and tending to $0$ \wrt $\tau$ for every fixed $v > 0$ since the same applies to the terms.
	
	Thus it has been established that $\beta$ is a $\Kinf$ function \wrt $v$ for each $\tau$, while also being decreasing, positive and tending to $0$ w.r.t. $\tau$, which proves that $\beta$ is indeed a $\KL$-function. 
	At the same time it has been demonstrated that $\Phi(\state, \tau) \leq \beta(\nrm{\state} + C_0, \tau)$, where $C_0 = \delta^{-1}(0)$.

\end{prf}

\begin{rem}
	\label{rem_brokenlinedelta}
	The continuous and strictly increasing redefinition of $\delta$ by means of Riemann integrals admits an alternative construction as follows.
	Define for each interval $[k, k+1], k \in \Z_{\ge 0}$, the numbers:
	\[
		d_{k+1} := \inf_{\eps \in [k, k+1]} \cl{\{ \delta(\eps) \}}.
	\]
	This setting is possible due to the fact that $\delta$ is bounded for all $\eps$ and $\lim\limits_{\eps \ra \infty} \delta(\eps) = \infty$.
	That $\delta$ is bounded for any $\eps$ and tends to infinity as $\eps$ does follows from the trivial observation that $\delta \le \eps$ always and $\delta$ was assumed unbounded.
	Notice that any other unbounded function $\delta'$ with $\delta' \le \delta$ is a valid certificate for the overshoot bound property as per the second line of \eqref{eqn_genericlim}.
	Hence, we safely redefine $\delta$ to be such that its graph is a polygonal chain specified by $(0, d_1), (1, d_1), (2, d_2), (3, d_3) \dots$\ie
	\begin{equation}
	\delta(\eps) := d_{1} + \sum_{i = 1}^{\infty}(d_{i + 1} - d_{i})\chi(t - i), \text{ where } \chi(\eps) := \frac{1}{2}(1 +  \abs{\eps}  -  \abs{\eps-1} ).
	\end{equation}
	Notice neither the original nor the thus redefined $\delta$ needs to be zero at zero and/or strictly increasing.
	However, we can make it strictly increasing on $[v_0, \infty)$ by defining $\delta$ to be
	\[
		\hat \delta(\eps) := \int_{0}^{\eps} \delta(\eps') \diff \eps'.
	\]
	similar to the construction in the proof, except for the factor, and defined only on the unit interval.
	To this end, let us assume, without loss of generality, that $\cl{\mathcal E^\delta_0}$ is well contained in the unit interval.
	Hence, unlike the polygonal chain construction above, we do not redefine $\delta$ to be the horizontal line segment on the unit interval.
	We take the line segment $\mathcal E^\delta_0$ instead and carry out the above Riemann integral construction.
	Evidently, if $v_0=0$, the so redefined $\delta$ is a $\Kinf$ function.
\end{rem}

\begin{rem}
	\label{rem_ifeps0}
	If the uniform overshoot condition in \eqref{eqn_genericlim} read
	\begin{equation*}
			\begin{aligned}
				& \forall \state \in \R^n \spc \lim_{\tau \ra \infty} \Phi(\state, \tau) = 0, \\
				& \exists \eps_0 > 0 \spc \forall \eps \ge \eps_0 \spc \exists \delta \geq 0 \spc \nrm{\state} \le \delta \implies \forall \tau \ge 0 \spc \Phi(\state, \tau) \le \eps,
			\end{aligned}		
	\end{equation*}
	then the $C_0$ would have to be set so as to account for the $\eps_0$ offset.
	In practice, if $\Phi$ were to refer to the distance to some goal under a some given policy, existence of an offset $\eps_0$ would mean that the said policy could not render the goal forward-invariant \ie some overshoot would exist even when starting inside the goal set. 
\end{rem}

For uniform goal reaching, it is crucial to have a uniform $\KL$ bound to estimate overshoot.
By \Cref{prp_lim2kappa}, existence of a uniform goal reaching policy $\policy_0$ (cf. strengthened goal reaching property as per \eqref{eqn_strong_goal_reaching}) would imply:

\begin{equation}
	\label{eqn_klreachprob}
	\forall \state_0 \in \states \spc \PP{ \exists \beta \in \KL \spc \goaldist(\State_t) \le \beta(\goaldist(\state_0)+C_0, t) \mid \Action_t \sim \policy_0(\bullet \mid \state_t)} \geq 1 - \eta,
\end{equation}
where a constant $C_0$ accounts for the $\eps_0$ offset above.

Notice each such $\beta$ depends on the sample variable $\omega$ in general.
This fact would prevent us from uniformly bounding an overshoot.
The reason is that general $\KL$ functions may be rather bizarre despite the common intuition of them being akin to, say, exponentials or something similarly behaved.
In fact, a $\KL$ may entail extremely slow convergence which cannot be described by any polynomial, elementary or even computable functions.
It may formally entail \eg a construction based on the busy beaver function.
We have to avoid such exotic cases.
Hence, we consider special classes of $\KL$ functions which we will refer to as tractable.

First, we need some definitions.

Let us denote a function $\beta$ that depends continuously on some $M \in \N$ parameters besides its main arguments, say, $v, \tau \in \R$ as $\beta\left[\vartheta^{0:M-1}\right](v, \tau)$.
Notice the upper index of the parameter variable $\vartheta$ refers to the respective parameter component, not an exponent. 

\begin{dfn}
	\label{dfn_kappaellfinitely}
	Let $\kappa, \xi$ be of space $\Kinf$ and $M \in \N$.
	A subspace $\KL[\kappa, \xi, M]$ of $\KL$ functions is called \textit{fixed-asymptotic tractable} if it is continuously finitely parametrizable \wrt $\kappa, \xi$ as follows: for any $\beta$ from $\KL[\kappa, \xi, M]$ there exist $M$ real numbers $\vartheta^{0:M-1}$ \sut
	\[
		\forall v, t \spc \beta(v, t) = \kappa\left[\vartheta^{0:M-2}\right](v) \xi\left( e^{-\vartheta^{M-1} t} \right),
	\]
	where $\kappa$ depends continuously on $\vartheta^{0:M-2}$.
\end{dfn}

\begin{exm}
	\label{exm_kappaellfinitely}
	Exponential functions $C v e^{ - \relfact t }$ with parameters $C \ge 0, \relfact \ge 0$ form a fixed-asymptotic tractable $\KL$ subspace.
\end{exm}

The next proposition allows to extract $\omega$-uniform $\KL$ bounds from policies that yield fixed-asymptotic tractable $\KL$ bounds.

\begin{prp}
	\label{prp_fixasymtractpolicy}
	Let $\KL[\kappa, \xi, M]$ be fixed-asymptotic tractable.
	Let $\policy_0$ satisfy
	\begin{equation}
		\label{eqn_fixasymtractstab}
		\begin{aligned}
			\forall \state_0 \in \states \spc \PP{ \exists \beta \in \KL[\kappa, \xi, M] \spc \goaldist(\State_t) \le \beta(\goaldist(\state_0)+C_0, t) \mid \Action_t \sim \policy_0(\bullet \mid \state_t)} \\
			\pushright{\geq 1 - \eta, \eta \in [0,1), C_0>0}.	
		\end{aligned}
	\end{equation}
	
	Then, for each $\eta'>0$, there exists a $\KL$ function $\beta'$ \sut
	\begin{equation}
			\forall \state_0 \in \states \spc \PP{\goaldist(\State_t) \le \beta'(\goaldist(\state_0)+C'_0, t) \mid \Action_t \sim \policy_0(\bullet \mid \state_t)} \geq 1 - \eta - \eta', C'_0>0. 
	\end{equation}
	Notice that the existence quantifier for $\beta'$ in the latter statement stays before the probability operator, whence uniformity over compacts of the sample space $\Omega$.
\end{prp}

\begin{prf}
	Recall $(\Omega, \Sigma, \mathbb P)$, the probability space underlying \eqref{eqn_mdp} and denote the state trajectory induced by the policy $\policy_0$ emanating from $\state_0$ as $\Traj_{0:\infty}^{\policy_0}(\state_0)$ with single elements thereof denoted $\Traj_t^{\policy_0}(\state_0)$.
	Let us explicitly specify the sample variable in the trajectory $\Traj_{0:\infty}^{\policy_0}(\state_0)$ emanating from $\state_0$ as follows: $\Traj_{0:\infty}^{\policy_0}(\state_0)[\omega], \omega \in \Omega$.
	Define, the subset $\Omega_0$ of elements $\omega \in \Omega$ to satisfy:
	\begin{equation}
		\label{eqn_omegastar}
		\Omega_0(\state_0) := \left\{ \exists \beta \in \KL[\kappa, \xi, M] \spc \goaldist\left( \Traj_t^{\policy_0}(\state_0)[\omega] \right) \le \beta(\state_0, t) \right\}.
	\end{equation}	
	We argue here assuming a fixed $\state_0$, so we omit the $\state_0$ argument in $\Omega_0$ for brevity.
	In other words, for every $\omega \in \Omega_0$, there exists a $\KL$ function $\beta^\omega$ from $\KL[\kappa, \xi, M]$ \sut 
	\begin{equation}
		\label{eqn_betaomega}
		\forall t \in \Z_{\ge 0} \spc \goaldist\left( \Traj_t^{\policy_0}(\state_0)[\omega] \right) \le \beta^\omega(\goaldist(\state_0), t).
	\end{equation}
	By the condition of the lemma, $\PP{\Omega_0} \ge 1 - \eta$.

	Now, let $\vartheta_{k=0:\infty}$ be a dense sequence in $\R^{M}$.
	
	Then, let $\{\beta_k\}_{k=0:\infty}$ be the sequence defined by \\ $\beta_k(v, t) := \kappa\left[\vartheta^{0:M-2}_k\right](v) \xi\left( e^{-\vartheta^{M-1}_k t}, \right)$ for all $v, t$.
	First, observe that $\{\beta_k\}_{k=0:\infty}$ is dense in $\image(\beta[ \bullet ])$ since the dependence of $\kappa\left[\vartheta^{0:M-2}_k\right](v) \xi\left( e^{-\vartheta^{M-1}_k t} \right)$ on the parameters $\vartheta^{0:M-1}$ is continuous and continuous functions map dense sets into dense sets.
	
	Define a sequence of functions $\{\hat \beta_k\}_{k=0:\infty}$ by \\
	$\hat \beta_k(v, t) := \left( \kappa\left[\vartheta^{0:M-2}_k\right](v) + 1 \right) \xi\left( e^{-\vartheta^{M-1}_k t} \right)$ for all $v, t$.
	These are $\KL$ functions with an offset along the $v$-dimension.

	The reason for such an offset will be revealed in \eqref{eqn_supbound}.
	
	
	Let $\hat \Omega_0^{(k)}, j \in \Z_{\ge 0}$ be defined via
	\[
	\forall \omega \in \hat \Omega_0^{(k)} \spc \forall \state_0 \in \states, t \ge 0 \spc \goaldist\left( \Traj_t^{\policy_0}(\state_0)[\omega] \right) \le \hat \beta_k(\goaldist(\state_0), t).
	\]
	Then, 
	\[
		\hat \Omega_0^{(\infty)} := \bigcap_{k=0:\infty} \hat \Omega_0^{(k)}
	\]
	is an event, since it is a countable intersection.
	Now, to any $\beta^\omega$ there corresponds some $\hat \beta_k$ \sut $\hat \beta_k \ge \beta^\omega$.
	To see this, consider
	\[
		\beta^\omega(v, t) = \kappa\left[\vartheta^{0:M-2}\right](v) \xi\left( e^{-\vartheta^{M-1} t}, \right)
	\] 
	as a function of $v, t$ with fixed $\vartheta^{0:M-1}$.
	By the density of $\{\beta_k\}_{k=0:\infty}$ in $\image(\beta[ \bullet ])$, there exists some $\beta_l, l \in \Z_{\ge 0}$ \sut 
	\begin{equation}
		\label{eqn_supbound}
		\sup_{v, t} \nrm{ \kappa\left[\vartheta^{0:M-2}\right](v) \xi\left( e^{-\vartheta^{M-1} t} \right) - \kappa\left[\vartheta^{0:M-2}_l\right](v) \xi\left( e^{-\vartheta^{M-1}_l t} \right) } \le \frac{\xi(1)} {2}.
	\end{equation}
	Now, take $\vartheta_k$ \sut $\vartheta^{0:M-2}_k := \vartheta^{0:M-2}_l$ and $\vartheta^{M-1}_k \le \vartheta^{M-1}$.
	From \eqref{eqn_supbound}, it is now clear how an offset by one in $\hat \beta$s along with the approximation bound $\frac{\xi(1)} {2}$ ensure $\hat \beta_k \ge \beta^\omega$.
	Namely, \eqref{eqn_supbound} implies
	\[
		\sup_{v} \nrm{ \kappa\left[\vartheta^{0:M-2}\right](v) - \kappa\left[\vartheta^{0:M-2}_l\right](v) } \le \frac{1}{2}.
	\]
	This in turn implies that
	\[
		\forall v \spc \kappa\left[\vartheta^{0:M-2}\right](v) \le \kappa\left[\vartheta^{0:M-2}_l\right](v) + 1.
	\]
	This together with the fact that
	\[
		\forall t \spc \xi\left( e^{-\vartheta^{M-1} t} \right) \le \xi\left( e^{-\vartheta_k^{M-1} t} \right)
	\]
	ensures $\hat \beta_k \ge \beta^\omega$.
	
	
	

	Then, it holds that $\Omega_0 \subseteq \hat \Omega_0^{(\infty)}$.
	
	Define a new sequence $\{\bar \beta_k\}_{k=0:\infty}$ as follows:
	\begin{equation}
		\label{eqn_barbetasqn}
		\begin{aligned}
			\forall v, t \spc \bar \beta_k(v, t) := \sum_{j \le k} \hat \beta_j(v, t),
		\end{aligned}
	\end{equation}	
	which is still a $\KL$ function with an offset along the $v$-dimension since it is a sum of finitely many $\KL$ functions with offsets along the $v$-dimension.
	
		
	Then it holds that $\bar \beta_{k+1} \ge \bar \beta_k$ meaning $\forall v \in \R_{\ge 0}, \forall t \ge 0 \spc \bar \beta_{k+1}(v, t) \ge \bar \beta_{k}(v, t)$.
	Then, evidently
	\begin{equation}
		\label{eqn_omegastarin}
		\Omega_0 \subseteq \bar \Omega_0^{(\infty)} := \bigcap_{k=0:\infty} \bar \Omega_0^{(k)},
	\end{equation}
	where $\bar \Omega_0^{(k)}$ is defined similarly to $\hat \Omega_0^{(k)}$ with $\bar \beta_k$ in place of $\hat \beta_k$.
	
	Too see this, fix an arbitrary $k \in \Z_{\ge 0}$ and observe that for any $\omega$
	\[
		\forall t \in \Z_{\ge 0}, \state_0 \in \states \spc \goaldist\left( \Traj_t^{\policy_0}(\state_0)[\omega] \right) \le \bar \beta_k(\goaldist(\state_0))
	\]
	implies
	\[
		\forall t \in \Z_{\ge 0}, \state_0 \in \states \spc \goaldist\left( \Traj_t^{\policy_0}(\state_0)[\omega] \right) \le \hat \beta_k(\goaldist(\state_0))
	\]
	hence an event $\bar \Omega_0^{(k)}$ implies an event $\hat \Omega_0^{(k)}$.
	
	By the construction \eqref{eqn_barbetasqn}, 
	\[
		\forall k \in \Z_{\ge 0} \spc \bar \Omega_0^{(k)}  \subseteq \bar \Omega_{k+1}.
	\]
	Furthermore, by \eqref{eqn_omegastarin}
	\begin{equation}
		\label{eqn_limbaromegaprob}
	   \PP{\bar \Omega_\infty} \ge \PP{\Omega_0}. 
	\end{equation}
	
	Notice $\bar \beta_k$ was constructed uniform on the respective $\bar \Omega_0^{(k)}$ in terms of $\omega$.
	Intuitively, $\bar \Omega_0^{(k)}$ are the sets of outcomes for which there is a uniform convergence modulus.
	
	For any $k \in \Z_{\ge 0}$ we have $\PP{\bar \Omega_0^{(k)}} \ge 1 - \eta + \eta_k$ where $\eta_k$ is a real number.

	The condition $\PP{\bar \Omega_0^{(k)}} \le \PP{\bar \Omega_{k+1}}$ implies that the sequence $\{\eta_k\}_{k=0:\infty}$ is a monotone increasing sequence which is bounded above.
	Hence, $\lim_{k \ra \infty} \eta_k$ exists.
	It also holds that $\lim_{k \ra \infty} \eta_k \le 0$ by \eqref{eqn_limbaromegaprob}.

	We thus argue that with probability not less than $1 - \eta + \eta_k$, the policy $\policy_0$ is a uniform stabilizer with an $\omega$-uniform $\KL$ bound.
	
	Fix any $k \in \Z_{\ge 0}$ and take the corresponding $\bar \beta_k$.
	It has the form
	\[
		\bar \beta_k(v, t) = \sum_{j \le k} \left( \kappa\left[\vartheta^{0:M-2}_j\right](v) + 1 \right) \xi\left( e^{-\vartheta^{M-1}_j t} \right).
	\]
	Let $\tilde \beta_k$ be defined as
	\[
		\tilde \beta_k(v, t) = \sum_{j \le k} \left( \kappa\left[\vartheta^{0:M-2}_j\right](v+C'_0) \right) \xi\left( e^{-\vartheta^{M-1}_j t} \right),
	\]
	where $C'_0$ is chosen \sut $\tilde \beta_k \equiv \bar \beta_k$.
	
	Finally, $\tilde \beta$ is a $\KL$ function sought to satisfy the claim of the proposition.
	
\end{prf}

\begin{rem}
	\label{rem_zerooffset}
	Let us assume $C_0=0$.
	The construction \\ $\hat \beta_k(v, t) := \left( \kappa\left[\vartheta^{0:M-2}_k\right](v) + 1 \right) \xi\left( e^{-\vartheta^{M-1}_k t} \right)$ for all $v, t$, necessarily leads to a non-zero offset $C_0'$ in the final $\KL$ function $\tilde \beta$, namely, it satisfies
	\[
		\sum_{j \le k} \kappa\left[\vartheta^{0:M-2}_j\right](C'_0) \xi\left( 1 \right) = k \xi(1).
	\]
	We cannot in general eliminate such an offset -- the density argument does not suffice here since for any fixed sequence of $\KL$ functions, there can always be one that grows faster than any selected one from the sequence in a vicinity of zero.
	However, the offset can be made arbitrarily small leading to an arbitrary small relaxation of the goal set to which stabilization is of interest.
	
	This can be done \eg by setting
	\[
		\hat \beta_k(v, t) := \left( \kappa\left[\vartheta^{0:M-2}_k\right](v) + \lambda_\kappa^{k+1} \right) \xi\left( e^{-\vartheta^{M-1}_k t} \right),
	\]
	where $0<\lambda_\kappa<1$ can be chosen arbitrarily small.
	Then, the offset will not exceed
	\[
		\frac{\lambda_\kappa}{1-\lambda_\kappa} \xi(1)
	\]
	leading to the respective arbitrarily small constant $C'_0$.
\end{rem}

Now, we address subspaces of $\KL$ functions called \textit{weakly tractable}.
First, let introduce the following relation on the space of two scalar argument functions:

\begin{dfn}
	\label{dfn_orderonkappas}
	A function of two arguments $\beta_2$ is said to precede another function of two arguments $\beta_1$, with the respective relation denoted as $\beta_1 \succ \beta_2$, if $\forall v, t \spc \beta_1(v,t) > \beta_2(v, t)$.
\end{dfn}

Next, we consider a special notion of continuity on the space of parametric maps into $\KL$ functions made handy afterwards.
Let $\mathcal P = \R^M, M \in \N$.

\begin{dfn}
	\label{dfn_klcontparammaps}
	A parametric map $\beta[\bullet]: \mathcal P \ra \KL$ is called $\KL$-continuous if for any $\KL$ function $\xi$ there exists $\delta>0$ \sut $\forall \vartheta_1, \vartheta_2 \in \mathcal P$ with $\nrm{\vartheta_1 - \vartheta_2} \le \delta$ it holds that $\abs{ \beta[\vartheta'] - \beta[\vartheta] } \prec \xi$.
\end{dfn}

Finally, we introduce proper $\KL$-continuous maps which avoid pathological cases where $\beta[\bullet]$ could get ``saturated'' with the parameter, in other words, have a finite limit superior in the $\KL$-continuity sense.

\begin{dfn}
	\label{dfn_klproperparammaps}
	A parametric map $\beta[\bullet]: \mathcal P \ra \KL$ is called proper if it is continuous and for any $\vartheta \in \mathcal P$ there exists a $\vartheta' \in \mathcal P$ \sut $\beta[\vartheta'] \succ \beta[\vartheta]$.
\end{dfn}

Notice that the condition $\beta[\vartheta'] \succ \beta[\vartheta]$ implies existence of a $\KL$ function $\xi$ \sut $\beta[\vartheta'] \succeq \beta[\vartheta]+\xi$.
Take \eg $\xi := \frac{\beta[\vartheta'] + \beta[\vartheta]}{2}$.

\begin{dfn}
	\label{dfn_weaktractkl}
	A subspace $\mathcal F$ of $\KL$ is called weakly tractable if there exists a proper parametric map $\beta[\bullet]: \mathcal P \ra \KL$ \sut $\image(\beta[\bullet]) = \mathcal F$.
\end{dfn}

\begin{exm}
	A fixed-asymptotic tractable $\KL$ subspace consisting of functions of the form
	\[
		\forall v, t \spc \beta(v, t) = \kappa\left[\vartheta^{0:M-2}\right](v) \xi\left( e^{-\vartheta^{M-1} t} \right),
	\]	
	where $\xi$ is unbounded \wrt $\vartheta^{M-1}$, and for any $\vartheta_1$ there exists $\vartheta_2$ \sut $\forall v \spc \kappa\left[\vartheta_1^{0:M-2}\right](v) < \kappa\left[\vartheta_2^{0:M-2}\right](v)$, is weakly tractable.
	However, a general fixed-asymptotic tractable $\KL$ subspace is not necessarily weakly tractable.
\end{exm}

Now, the following may be claimed.

\begin{prp}
	\label{prp_weaktractpolicy}
	Let $\mathcal F$ be weakly tractable.
	Let $\policy_0$ satisfy
	\begin{equation}
		\label{eqn_weaktractstab}
		\begin{aligned}
			& \forall \state_0 \in \states \spc \PP{ \exists \beta \in \mathcal F \spc \goaldist(\State_t) \le \beta(\goaldist(\state_0)+C_0, t) \mid \Action_t \sim \policy_0(\bullet \mid \state_t)} \geq 1 - \eta, \\
			& \pushright{\eta \in (0,1), C_0>0.	}
		\end{aligned}
	\end{equation}
	
	Then, for each $\eta'>0$, there exists a $\KL$ function $\beta'$ \sut
	\begin{equation}
			\forall \state_0 \in \states \spc \PP{\goaldist(\State_t) \le \beta'(\goaldist(\state_0)+C_0, t) \mid \Action_t \sim \policy_0(\bullet \mid \state_t)} \geq 1 - \eta - \eta'. 
	\end{equation}
	Notice that the existence quantifier for $\beta'$ in the latter statement stays before the probability operator, whence uniformity over compacts of the sample space $\Omega$.
\end{prp}

\begin{prf}
	The proof follows the same ideas as the one for \Cref{prp_fixasymtractpolicy}, yet many constructions become simpler.
	We again choose a dense sequence $\{\vartheta_k\}_k$ in $\mathcal P$.
	Then, for a given $\beta^\omega$ from $\mathcal F$, we find a $\hat \beta \in \mathcal F$ \sut $\hat \beta \succeq \beta^\omega + \xi$ for some $\KL$ function $\xi$ using the fact that $\beta[\bullet]$ is proper. 
	Using $\KL$-continuity of $\beta[\bullet]$, take a $\xi/2$-close neighbor to $\hat \beta$ from $\{\beta[\vartheta_k]\}_k$.
	The rest of the proof is essentially the same as for \Cref{prp_fixasymtractpolicy}.	
\end{prf}

\begin{rem}
	\label{rem_zerooffsetweaktract}
	Notice that the final $\KL$ function possesses the same offset as in the condition \eqref{eqn_weaktractstab}, in contrast to the fixed-asymptotic case.
\end{rem}

Let us refer to goal reaching policies which yield fixed-asymptotic or weakly tractable $\KL$ bounds as per \eqref{eqn_fixasymtractstab}, \eqref{eqn_weaktractstab} fixed-asymptotic and, respectively, weakly tractable.

\begin{exm}
	\label{exm_expstab}
	For a controlled Markov chain 
	\begin{equation}
		\label{eqn_linmdp}
		\State_{t + 1} = F \State_t + G \Action_t + \Noise_t, F, G \text{ being matrices of proper dimension},
	\end{equation}
	where $\Noise_t$ is uniformly distributed on $[-\bar \noise, \bar \noise]^n, n \in \N, \bar \noise < \diam(\G)$, a condition
	\begin{equation}
		\label{eqn_expstab}
		\PP{\lim_{t \ra \infty} \goaldist(\State_t) = 0 \mid \Action_t \sim \policy_0(\bullet \mid \state_t)} \ge 1-\eta, \eta \in (0,1)
	\end{equation}	
	implies $\policy_0$ is a stabilizer with the respective exponential fixed-asymptotic tractable $\KL$ subspace as per \Cref{exm_kappaellfinitely}.
\end{exm}

Hence, if $\policy_0$ is fixed-asymptotic or weakly tractable goal reaching, then \Cref{thm_calfuniform} applies.

\begin{rem}
	To verify statistical goal reaching properties in practice, one may use confidence bounds by \eg Hoeffding’s Inequality in terms of the environment episodes \citep{Hertneck2018Learningapprox}.
\end{rem}

\begin{rem}
	It follows that \Crefalgcalfstate \ qualitatively retains the goal reaching property of $\policy_0$, although the exact reaching times may be different.
\end{rem}


\section{Miscellaneous variants of the approach}
\label{sec_miscalgos}

\begin{algorithm}
	\begin{algorithmic}[1]
		\STATE {\bfseries Setup:} MDP, nominal agent details \eg networks, actor loss function $\loss_\act$, critic loss function $\loss_\crit$, and ${\bar \nu} > 0, \hat \kappa_\low, \hat \kappa_\up, \policy_0 \in \policies_0, \text{relaxation factor }1 > \relfact \geq 0$
		\STATE \textbf{Initialize}: $\state_0, \action_0 \sim \policy_0(\bullet \mid \state_0), w_0 \in \W$ \sut
		$$
		-\hat \kappa_\up(\nrm{\state_0}) \leq \hat \Value^{w_0}(\state_0) \leq -\hat \kappa_\low(\nrm{\state_0})
		$$
		\STATE $w^\dagger \gets w_{0}, \state^\dagger \gets \state_0, \relprob \gets \relfact$
		\FOR {$t := 1, \dots \infty$}
			\STATE Take action $\action_{t-1}$
			\STATE	Update policy:
			\[
				\policy_{t}(\bullet \mid \state_{t}) \la \argmin\limits_{\policy \in \policies} \loss_\act(\policy)
			\]		
			If desired, $\policy_{t}$ may be taken to produce a random action with probability $\eps>0$ 
			\STATE Get new action $\action_t$ sampled from $\policy_{t}(\bullet \mid \state_{t})$		
			\STATE Try critic update
			\[
			\begin{array}{lll}
		 		w^*  \gets  &  \argmin\limits_{w \in \W} \loss^\crit(w) \\
					& \sut \; \hat Q^w(\state_t, \action_t) - \hat Q^{w^\dagger}(\state^\dagger, \action^\dagger) \le - {\bar \nu}, \\
					& \phantom{\sut \;} \hat \kappa_\low (\nrm{\state_t}) \le \hat Q^{w}(\state_t, \action_t) \le \hat \kappa_\up(\nrm{\state_t})
			\end{array}
			\]
			\IF{ solution $w^*$ found}	
			\STATE $\state^\dagger \gets \state_t, \action^\dagger \gets \action_t, w^\dagger \gets w^*$
			\ENDIF
			\STATE $q \gets $ sampled unformly from $[0, 1]$.
			\IF{ solution $w^*$ not found and $q \geq \relprob$}
			\STATE $\policy_{t}(\bullet \mid \state_{t}) \gets \policy_0(\bullet \mid \state_t)$	
			\ENDIF	
			\STATE $\relprob \gets \relfact \relprob$
		\ENDFOR		 
	\end{algorithmic}
	\caption{Suggested goal reaching agent (state-action-valued critic by model $\hat Q^w(\state, \action)$).}
	\label{alg_calfstateaction}
\end{algorithm}

\begin{rem}
	\label{rem_calfqstab}
	\Crefthmcalfstabmean \ with its corollaries apply for goal reaching preservation by \Cref{alg_calfstateaction} as well.
	The proofs are exactly the same upon setting $\hat \Lambda^\dagger_{t} := - \hat Q^{w^\dagger_t}(\state^\dagger_t, \action^\dagger_t)$, where $\action^\dagger_t$ is set analogously to the setting of $\state^\dagger_t$ (see \Cref{sec_thms}).
\end{rem}

\section{Environments}
\label{sec_systems}


\subsection{Inverted pendulum}
\label{sub_cartpole}

A pole is attached by an unactuated joint to a cart, which moves along a frictionless track.
The goal is to balance the pole by applying a positive or a negative force $F$ to the left side of the cart.
The duration of one episode is set to $15$ seconds, while the sampling rate of measurements is set to 100 Hz.


\begin{equation}
	\label{eqn_acrtpole}
	\begin{aligned}
		& \state_{t} = \begin{pmatrix} \vartheta_{t} \\ x_{t} \\ \omega_{t} \\ v_{t} \end{pmatrix}, \begin{array}{l} \vartheta \text{ -- pendulum angle,} \\ x \text{ -- cart horizontal coordinate,} \\ \omega \text{ -- pendulum angular velocity,} \\ v \text{ -- cart horizontal velocity,} \end{array} a_{t} = \begin{pmatrix}
		F_{t}
\end{pmatrix},	\begin{array}{l}
F \text{ -- horizontal force,}
\end{array}	\\
		& \reward(\state, \action) =  -20 (1 - \cos\vartheta) - 2\omega^2.
	\end{aligned}
\end{equation}

The environment dynamics are described by discretizing the following equations of motion:

\begin{equation}
    \label{eqn_state_dynamics}
    \diff \left(
        \begin{array}{c}
            \vartheta \\ x \\ \omega \\ v
        \end{array}
        \right) = \left(
        \begin{array}{c}
            \omega \\
            v \\
            \frac{g \sin{\vartheta}(m_c + m_p) - \cos{\vartheta}(F + m_p l \omega^2 \sin{\vartheta})}{\frac{4l}{3}(m_c + m_p) - lm_p \cos^2{\vartheta}} \\
            \frac{F + m_p l \omega ^2 \sin{\vartheta} - \frac{3}{8}m_p g\sin(2\vartheta)}{m_c + m_p - \frac{3}{4} m_p \cos ^ 2 \vartheta}
        \end{array}
        \right) \diff t,
\end{equation}
where
\begin{itemize}
    \item $m_c = 0.1$ is the mass of the cart [kg]
    \item $m_p = 2.0$ is the mass of the pole [kg]
    \item $l = 0.5$ is the length of the pole [m]
    \item $g = 9.81$ is the acceleration of gravity [m / $s^2$]
\end{itemize}
We define the goal set $\G$  as follows
\begin{equation}
    \G = \{|\vartheta| \leq  0.1\}.
\end{equation}
Initial condition $\state_0$ for the inverted pendulum environment is 
\begin{equation}
	\state_0 =  \left(
	\begin{array}{c}
		\pi / 7 \\ 2 \\ 0 \\ 0
	\end{array}
	\right).
\end{equation}
The range of allowable actions in the environment is:
\begin{equation}
    F \in [-50, 50].
\end{equation}


\begin{figure*}[ht]
\vskip 0.2in
\begin{center}
\centerline{\includegraphics[scale=0.5]{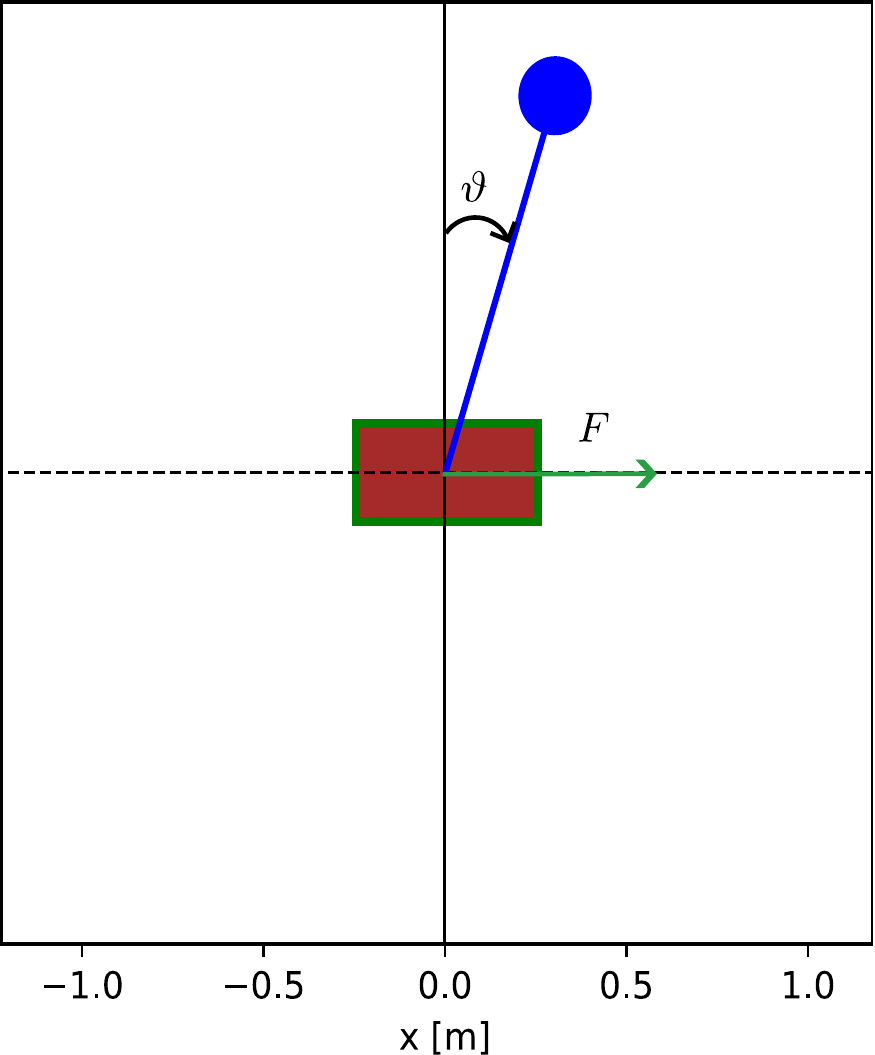}}
\caption{A diagram of the inverted pendulum environment.}
\end{center}
\vskip -0.2in
\end{figure*}

\newpage
\subsection{Pendulum}
\label{sub_invpendel}

The system consists of a pendulum attached at one end to a fixed point, and the other end being free.
The pendulum is initially pointed downwards and the goal is to apply torque $\tau$ to the joint to swing the pendulum into an upright position, with its center of gravity right above the fixed point.
The duration of one episode is set to $10$ seconds, while the sampling rate of measurements is set to 100 Hz.

\begin{equation}
	\label{eqn_invpendel}
	\begin{aligned}
		& \state_{t} = \begin{pmatrix}
		\vartheta_{t} \\
		\omega_{t}
		\end{pmatrix}, \begin{array}{l} \vartheta \text{ -- angle,} \\  \omega \text{ -- angular velocity,}\end{array}
		a_{t} = \begin{pmatrix}
		\tau_{t}
	\end{pmatrix},	\begin{array}{l}
	\tau \text{ -- torque,}
	\end{array}	\\
		& \reward(\state, \action) =  -\arccos ^ 2(\cos(\vartheta)) - 0.1 \omega^2 - 0.001 \tau ^ 2.
	\end{aligned}	
\end{equation}
The dynamics are described by the following differential equations (omitting the subscript \(t\) for simplicity):
\begin{equation}
\diff \left(
    \begin{array}{c}
        \vartheta \\
        \omega 
    \end{array}
    \right) = \left(
    \begin{array}{c}
        \omega \\
        \frac{3g}{2l}\sin(\vartheta) + \frac{3\tau}{m l^2}
    \end{array}
    \right) \diff t,
\end{equation}
where
\begin{itemize}
    \item $m = 0.127$ is the mass of the pendulum [kg]
    \item $l = 0.337$ is the length of the pendulum [m]
    \item $g = 9.81$ is the acceleration of gravity [m / $s ^ 2$]
\end{itemize}
We define the goal set $\G$  as follows
\begin{equation}
    \G = \{|\vartheta \text{ mod } 2 \pi - \pi | \leq  0.25, |\omega| \leq  0.25\}.
\end{equation}
Initial condition $\state_0$ for the pendulum environment is 
\begin{equation*}
\state_0 = \left(
\begin{array}{c}
	\pi \\
	1 
\end{array}.
\right) 
\end{equation*}
The range of allowable actions in the environment is:
\begin{equation}
    \tau \in [-0.1, 0.1].
\end{equation}
\begin{figure*}[ht]
\vskip 0.2in
\begin{center}
\centerline{\includegraphics[scale=0.5]{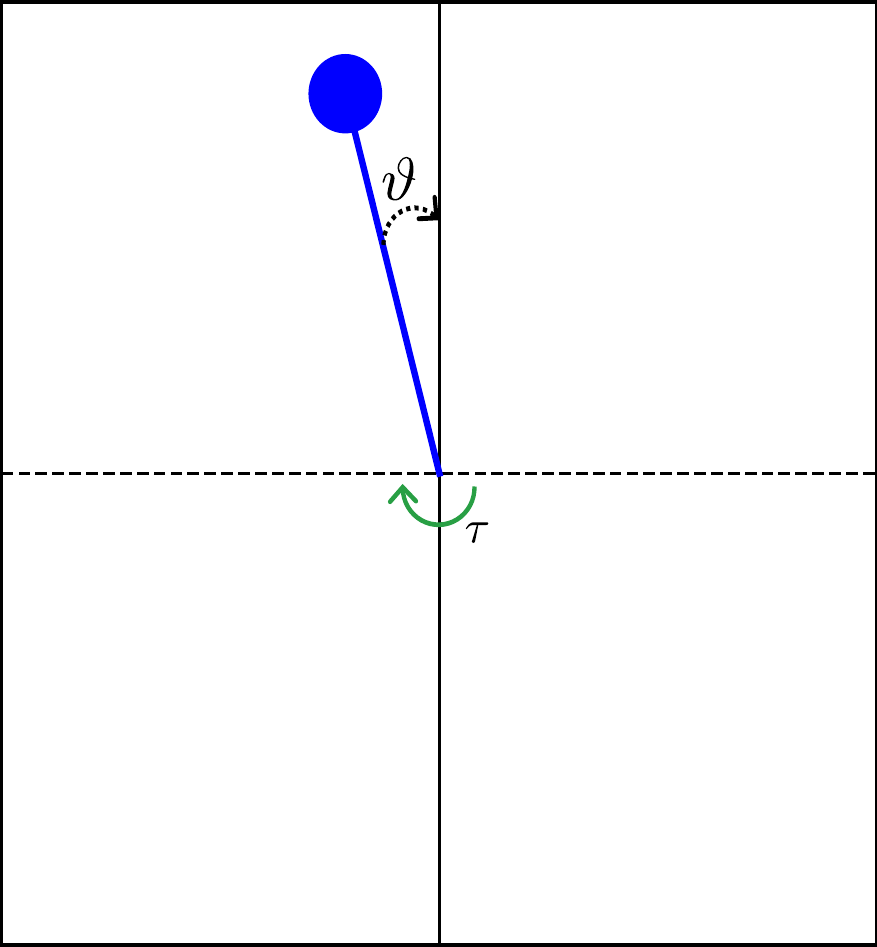}}
\caption{A diagram of the pendulum environment.}
\end{center}
\vskip -0.2in
\end{figure*}

\newpage
\subsection{Three-wheel robot}
\label{sub_3wrobot}

A pair of actuated wheels and a ball-wheel are attached to a platform that moves on a flat surface.
The wheels roll without slipping.
The pair of actions for the respective actuators is decomposed into forward velocity $v$ that is aligned with the direction in which the robot is facing and angular velocity $\omega$ applied to the center of mass of the robot and directed perpendicular to the platform.
The goal is to park the robot at the origin and facing the negative $x$ axis.
The duration of one episode is set to $5$ seconds, while the sampling rate of measurements is set to 100 Hz.

\begin{equation}
	\label{eqn_3wrobot}
	\begin{aligned}
		& \state_{t} = \begin{pmatrix}
			x_{t} \\ y_{t} \\ \vartheta_{t}
			\end{pmatrix}, \begin{array}{l}  x \text{ -- coordinate along x-axis,} \\ y \text{ -- coordinate along y-axis,} \\ \vartheta \text{ -- angle,} \end{array}
			a_{t} = \begin{pmatrix}
		v_{t} \\ \omega_{t}
\end{pmatrix},	\begin{array}{l}
v \text{ -- forward velocity,} \\ \omega \text{ -- angular velocity,}
\end{array}	 \\
		& \reward(\state, \action) =  -x^2 - 10 y^2 - \vartheta^2.
	\end{aligned}
\end{equation}
The dynamics are described by the following differential equations (omitting the subscript \(t\) for simplicity):
\begin{equation}
    \diff \left(
        \begin{array}{c}
            x \\
            y \\
            \vartheta 
        \end{array}
        \right) = \left(
        \begin{array}{c}
            x \cos \vartheta \\
            y \sin \vartheta \\  
            \omega
        \end{array}
        \right) \diff t.
    \end{equation}
We define the goal set $\G$  as follows
\begin{equation}
    \G = \{|x| \leq  1, |y| \leq  1, |\vartheta| \leq  0.7\}.
\end{equation}
Initial condition $\state_{0}$ is
$$
	\state_0 =  \left(
	\begin{array}{c}
		5 \\
		5 \\
		\frac{2 \pi }{3}
	\end{array}
	\right).
$$
The range of allowable actions in the environment is:
\begin{equation}
    (v, \omega) \in [-25, 25] \times [-5, 5].
\end{equation}
\begin{figure*}[ht]
\vskip 0.2in
\begin{center}
\centerline{\includegraphics[scale=0.5]{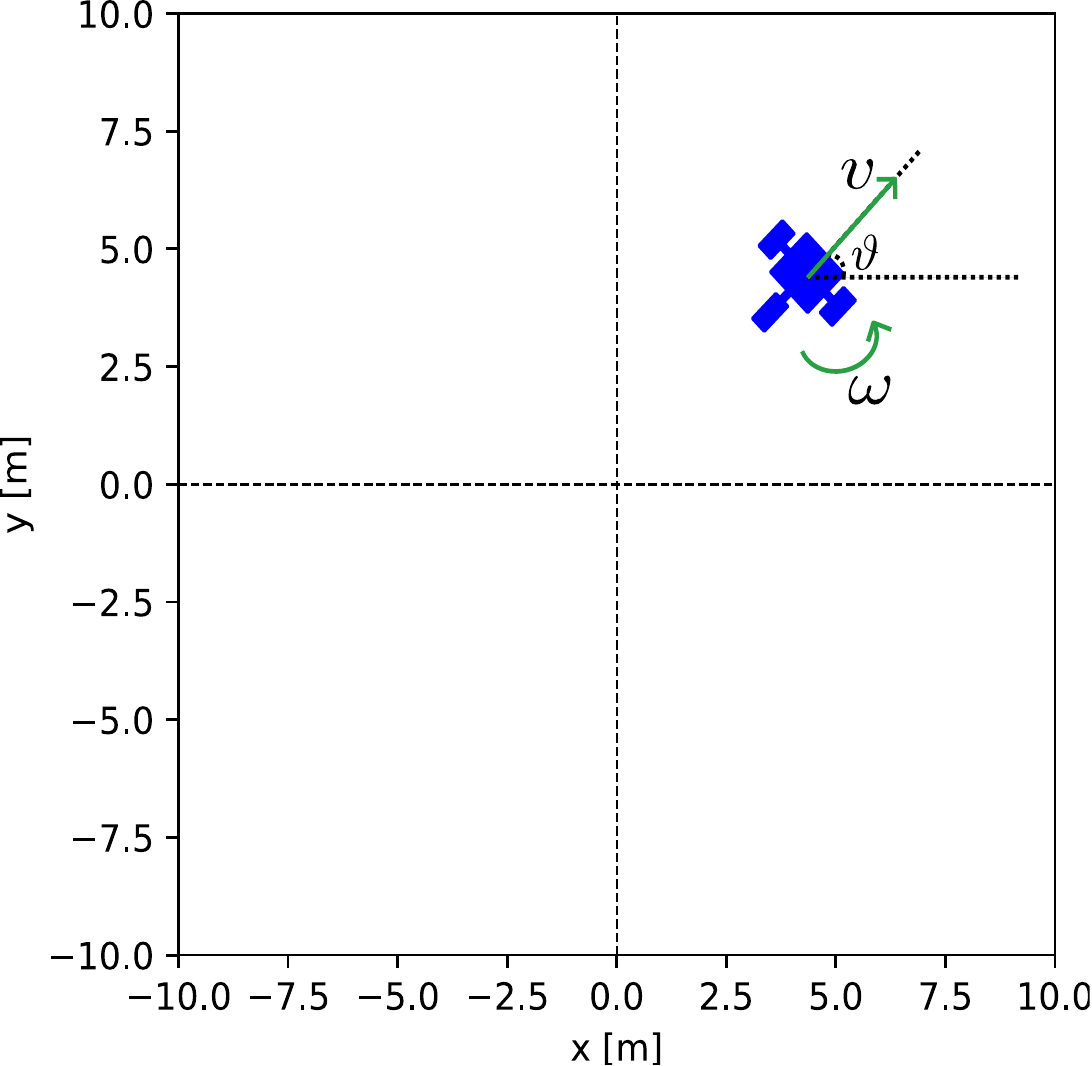}}
\caption{A diagram of the three-wheel robot environment.}
\end{center}
\vskip -0.2in
\end{figure*}

\newpage
\subsection{Two-tank system}
\label{sub_2tank}

The two-tank system consists of two tanks which are connected by an open valve.
Liquid inflow $i$ to the first tank is governed by a pump and there is an interconnection between the tanks.
In addition, there is a permanent leak from tank 2.
The goal is to keep both tanks exactly 40\% full.
The duration of one episode is set to $80$ seconds, while the sampling rate of measurements is set to 10 Hz.

\begin{equation}
	\label{eqn_2tank}
	\begin{aligned}
		& \state_{t} = \begin{pmatrix}
		h_{1, t} \\ h_{2, t}
		\end{pmatrix}, \begin{array}{l} h_{1} \text{ -- first tank fullness,} \\ h_{2} \text{ -- second tank fullness,} \end{array}
		a_{t} = \begin{pmatrix}
		i_{t}
\end{pmatrix},	\begin{array}{l}
i \text{ -- liquid inflow,}
\end{array}	 \\
		& \reward(\state, \action) =  -10 (h_1 - 0.4)^2 - 10 (h_2 - 0.4)^2.
	\end{aligned}
\end{equation}
The dynamics are described by the following differential equations (omitting the subscript \(t\) for simplicity):
\begin{equation}
    \diff \left(
    \begin{array}{c}
        h_1 \\
        h_2 
    \end{array}
    \right) = \left(
    \begin{array}{c}
        \frac{1}{\tau_1}(K_1 i - h_1) \\
        \frac{1}{\tau_2} (- h_2 + K_2 h_1 + K_3 h_2 ^ 2)
    \end{array}
    \right) \diff t,
\end{equation}
where
\begin{itemize}
\item $K_1 = 1.3$  represents the coefficient that scales the liquid inflow \( i \) to the first tank
\item $K_2 = 1.0$ represents the coefficient that scales the flow from the first tank \( h_1 \) to the second tank \( h_2 \)
\item $K_3 = 0.2$  represents the coefficient that scales the quadratic term of the second tank’s fullness
\item $\tau_1 = 18.4$ is the base area of the first tank
\item $\tau_2 = 24.4$ is the base area of the second tank
\end{itemize}
We define the goal set $\G$  as follows
\begin{equation}
    \G = \{|h_1 - 0.4| \leq  0.05, |h_2-0.4| \leq  0.05\}.
\end{equation}
Initial condition $\state_0$ is

\begin{equation}
	\state_0 = \left(
	\begin{array}{c}
		2 \\
		-2 
	\end{array}
	\right).
\end{equation}

The range of allowable actions in the environment is:
\begin{equation}
    i \in [0, 1].
\end{equation}

\begin{figure*}[ht]
\vskip 0.2in
\begin{center}
\centerline{\includegraphics[scale=0.5]{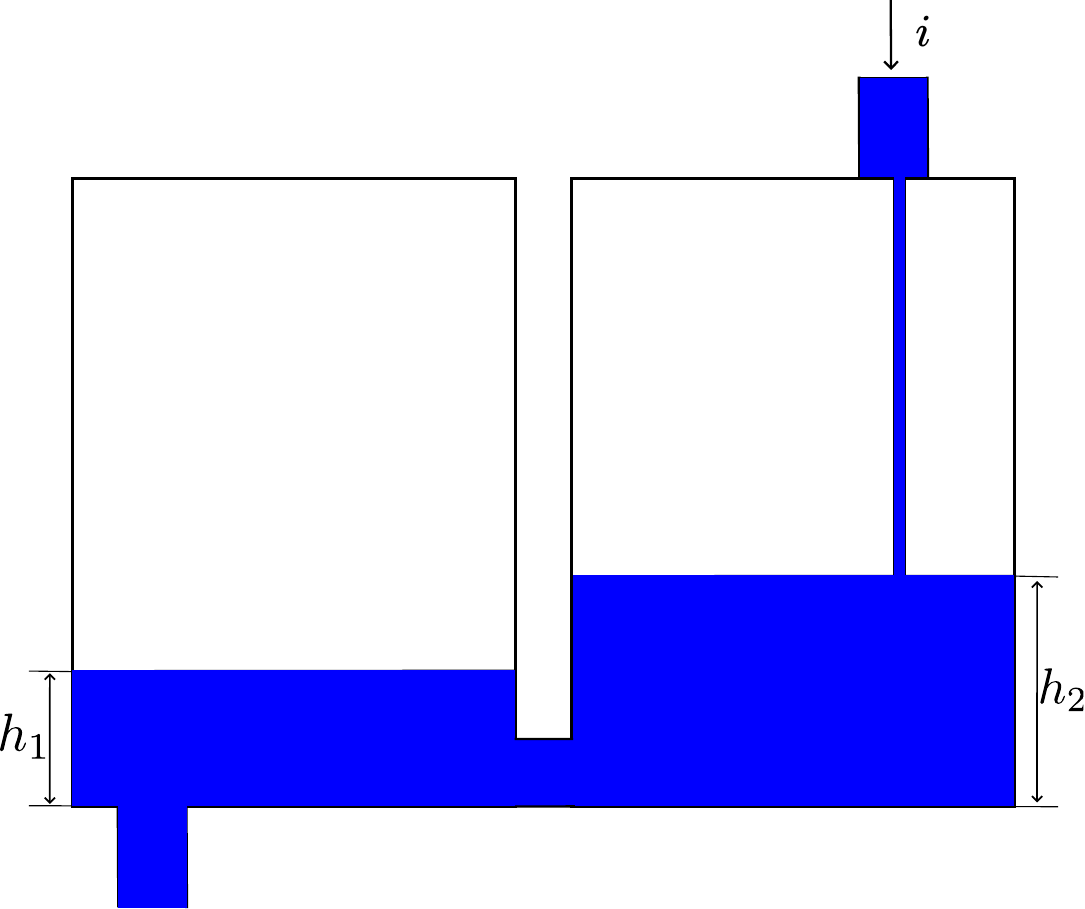}}
\caption{A diagram of the two-tank environment.}
\end{center}
\vskip -0.2in
\end{figure*}

\newpage
\subsection{Omnibot (kinematic point)}
\label{sub_kinpt}

A massless point moves on a plane with velocity $\left({v_{x}} \atop {v_{y}}\right)$.
The goal is to drive the point to the origin.
The duration of one episode is set to $10$ seconds, while the sampling rate of measurements is set to 100 Hz.

\begin{equation}
	\label{eqn_omnibot}
	\begin{aligned}
		& \state_{t} = \begin{pmatrix}
		x_{t} \\ y_{t}
		\end{pmatrix}, \begin{array}{l} x \text{ -- coordinate along x-axis,} \\ y \text{ -- coordinate along y-axis,} \end{array} a_{t} = \begin{pmatrix}
		v_{x, t} \\ v_{y, t}
\end{pmatrix},	\begin{array}{l}
v_{x} \text{ -- velocity along x-axis,} \\ v_{y} \text{ -- velocity along y-axis,}
\end{array}	 \\
		& \reward(\state, \action) =  -10 x^2 - 10 y^2.
	\end{aligned}
\end{equation}
The dynamics are described by the following differential equations (omitting the subscript \(t\) for simplicity):
\begin{equation}
    \diff \left(
    \begin{array}{c}
        x \\
        y
    \end{array}
    \right) = \left(
        \begin{array}{c}
            v_x \\ 
            v_y
        \end{array}
    \right) \diff t.        
\end{equation}
We define the goal set $\G$ as follows:
\begin{equation}
    \G = \{|x| \leq  0.5, |y| \leq  0.5\}.
\end{equation}
Initial condition $\state_0$ is 
\begin{equation}
	\state_0 = \left(
	\begin{array}{c}
		-10 \\
		-10 
	\end{array}
	\right).
\end{equation}
The range of allowable actions in the environment is:
\begin{equation}
(v_x, v_y) \in [-10, 10] \times [-10, 10].
\end{equation}


\begin{figure*}[ht]
\vskip 0.2in
\begin{center}
\centerline{\includegraphics[scale=0.5]{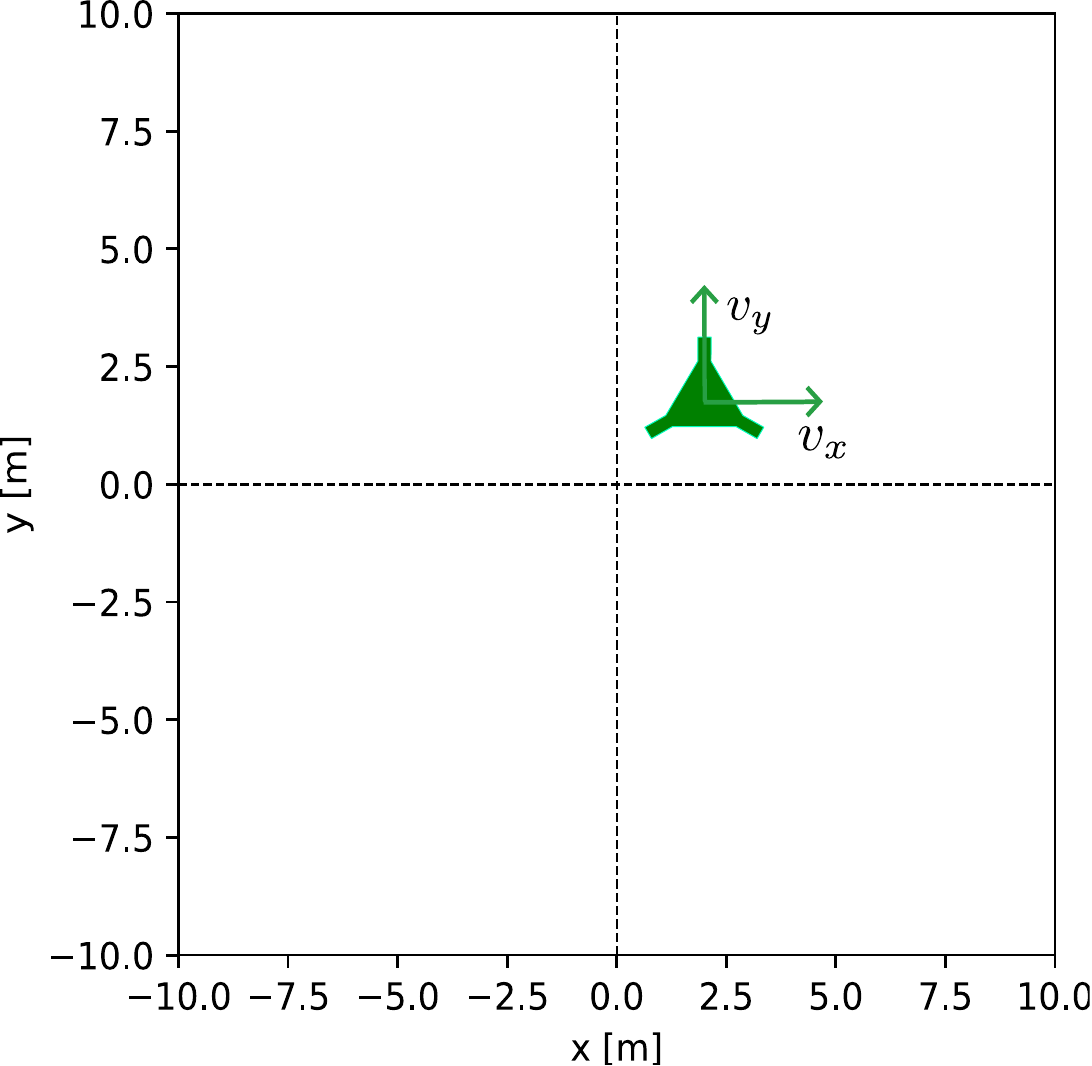}}
\caption{A diagram of the omnibot environment.}
\end{center}
\vskip -0.2in
\end{figure*}

\newpage
\subsection{Lunar lander}
\label{sub_lunlander}

A jet-powered spacecraft is approaching the surface of the moon.
It can activate its two engines and thus accelerate itself in the direction opposite to which the activated engine is facing.
The goal is to land at the desired location at the appropriate speed and angle.
The duration of one episode is set to $20$ seconds, while the sampling rate of measurements is set to 100 Hz.

\begin{equation}
	\label{eqn_lunarlander}
	\begin{aligned}
		& \state_{t} = \begin{pmatrix}
		x_{t} \\ y_{t} \\ \vartheta_{t} \\ v_{x, t} \\ v_{y, t} \\ \omega_{t} \end{pmatrix}, \begin{array}{l}  x \text{ -- horizontal coordinate,} \\ y \text{ -- vertical coordinate,} \\ \vartheta \text{ -- angle,} \\ v_{x} \text{ -- horizontal velocity,} \\ v_{y} \text{ -- vertical velocity,} \\ \omega \text{ -- angular velocity,} \end{array} a_{t} = \begin{pmatrix}
            F_{\text{side}, t} \\ F_{\text{vert},t}
\end{pmatrix},	\begin{array}{l}
F_{\text{side}} \text{ -- lateral thrust,} \\ F_{\text{vert}} \text{ -- upward thrust,}
\end{array}	 \\
		& \reward(\state, \action) =  -x^2 - 0.1 (y - 1)^2 - 10 \vartheta^2 - 0.1 v_{x}^{2} - 0.1 v_{y}^{2} - 0.1 \omega^{2}.
	\end{aligned}
\end{equation}
The dynamics are described by the following differential equations (omitting the subscript \(t\) for simplicity):
\begin{equation}
    \diff \left(
        \begin{array}{c}
           x \\
           y \\
           \vartheta \\
           v_x \\
           v_y \\
           \omega
        \end{array}
        \right) = \left(
        \begin{array}{c}
            v_x \\ v_y \\ \omega \\ \frac{1}{m}\left(F_{\text{side}} \cos \vartheta-F_{\text{vert}} \sin \vartheta\right) \\ \frac{1}{m}\left(F_{\text{side}} \sin \vartheta+F_{\text{vert}} \cos \vartheta\right)-g \\ \frac{F_{\text{side}}}{J}\
        \end{array}
        \right) \diff t,
\end{equation}
where
\begin{itemize}
    \item $m=10$ is the mass of the spaceship
    \item $J=3$ is the moment of inertia of the spaceship with respect to its axis of rotation
    \item $g=1.625$ is the acceleration of gravity
\end{itemize}
We define the goal set $\G$ as follows:
\begin{equation}
    \G = \{|y- 1| \leq  0.05, |\vartheta| \leq  0.05\}.
\end{equation}
Initial condition
\begin{equation}
	\state_0 =  \left(
	\begin{array}{c}
		3 \\
		5 \\
		\frac{2 \pi }{3} \\
		0 \\
		0 \\
		0
	\end{array}
	\right).
\end{equation}
The range of allowable actions in the environment is:
\begin{equation}
(F_{\text{side}}, F_{\text{vert}}) \in [-100, 100] \times [-50, 50].
\end{equation}
\begin{figure*}[ht]
\vskip 0.2in
\begin{center}
\centerline{\includegraphics[scale=0.5]{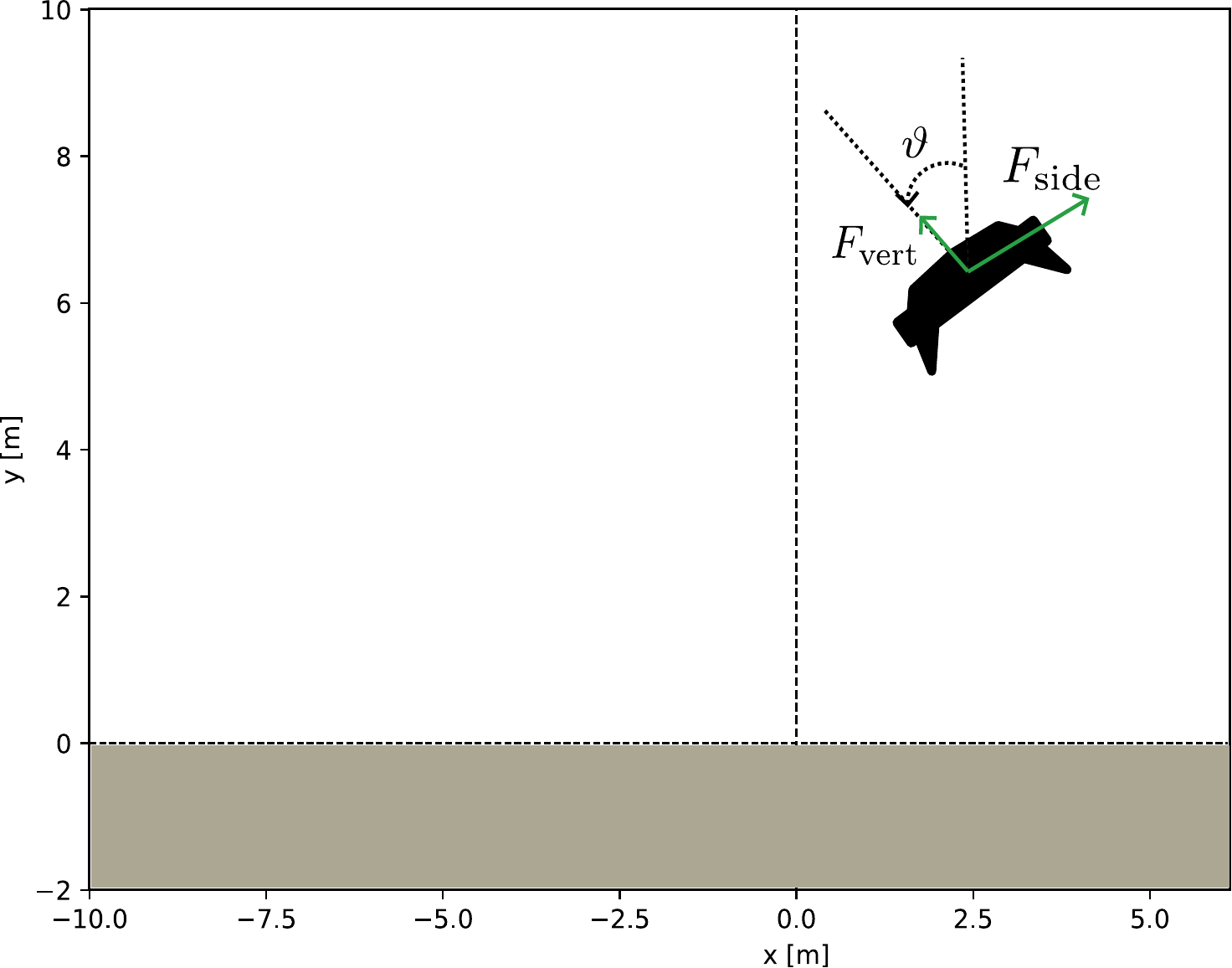}}
\caption{A diagram of the lunar lander environment.}
\end{center}
\vskip -0.2in
\end{figure*}

\section{Technical details of experiments}

\label{sec_experimdetails}

For evaluation, we used total rewards starting from initial conditions specified in \Cref{sec_systems}. 
This approach ensures that the performance of the policy is evaluated consistently from the same initial conditions, providing a clear baseline for comparison. 
Each experiment was conducted over 10 runs, with seeds set from 1 to 10.
All experiments are pre-configured in the respective \texttt{.yaml} files in \texttt{presets} folder of the associated code bundle found at \href{https://github.com/osinenkop/regelum-calf}{https://github.com/osinenkop/regelum-calf}, kindly see the code structure in \Cref{part_codeappendix}, \Cref{sec_codestruct}. 
The structure of this code, expected results and instructions on how to launch the code are contained in Code Appendix, \Cref{part_codeappendix}.
In our experiments, hyperparameters for SAC, PPO, and VPG algorithms were carefully tuned through a comprehensive grid search to optimize performance for each specific environment.
For SAC, although we began with the CleanRL's well-tuned defaults, we chose ranges for some hyperparameters based on preliminary tests indicating effective learning rates, thus exploring policy learning rates in the range of 0.0001 to 0.001. 
This led us to values like 0.00079 for Pendulum and 0.001 for Lunar Lander. 
For PPO, we tested policy learning rates from 0.0001 to 0.01, determining 0.0005 for the 3-wheeled robot and 0.003 for inverted pendulum, while also fine-tuning the critic TD order and hidden sizes based on environment complexity. 
Similarly, VPG's hyperparameters were optimized with policy learning rates set within 0.0001 to 0.01, choosing 0.0005 for the 3-wheeled robot. 
Each algorithm was finely tuned to leverage the properties of its environment, thereby enhancing overall performance.

\pagebreak
\subsection{Soft Actor Critic} 
Input for SAC algorithm:
\begin{itemize}
	\item $\theta$: Initial policy parameters
	\item $\phi_1, \phi_2$: Initial Q-function parameters
	\item $\mathcal{D}$: Empty replay buffer
\end{itemize}
\begin{algorithm}[H]
	\caption{Soft Actor-Critic}
	\begin{algorithmic}[1]
		\STATE $\phi_{\text{targ}, 1}, \phi_{\text{targ}, 2} \leftarrow \phi_1, \phi_2$ // Initialize target parameters
		\REPEAT
		\STATE Observe state $s$ and select action $a \sim \pi_\theta(\cdot \mid s)$
		\STATE Execute $a$ in the environment
		\STATE Observe next state $s'$, reward $r$, and done signal $d$
		\STATE Store $(s, a, r, s', d)$ in replay buffer $\mathcal{D}$
		\IF {$s'$ is terminal}
		\STATE Reset environment state
		\ENDIF
		\IF {it's time to update}
		\FOR {update step $j$ in range(however many updates)}
		\STATE Randomly sample a batch of transitions $B = \{(s, a, r, s', d)\}$ from $\mathcal{D}$
		\STATE Compute targets for the Q-functions:
		$$
		y(r, s', d) = r + \gamma(1 - d)\left(\min_{i=1,2} Q_{\phi_{\text{targ}, i}}(s', \tilde{a}') - \alpha \log \pi_\theta(\tilde{a}' \mid s')\right)
		$$
		$$
		\quad \tilde{a}' \sim \pi_\theta(\cdot \mid s')
		$$
		\STATE Update Q-functions by one step of gradient descent:
		$$
		\nabla_{\phi_i} \frac{1}{|B|} \sum_{(s, a, r, s', d) \in B}\left(Q_{\phi_i}(s, a) - y(r, s', d)\right)^2
		\quad \forall i \in \{1, 2\}
		$$
		\STATE Update policy by one step of gradient ascent:
		$$
		\nabla_\theta \frac{1}{|B|} \sum_{s \in B}\left(\min_{i=1,2} Q_{\phi_i}(s, \tilde{a}_\theta(s)) - \alpha
		\log \pi_\theta(\tilde{a}_\theta(s) \mid s)\right)
		$$
		where $\tilde{a}_\theta(s) \sim \pi_\theta(\cdot \mid s)$ is differentiable w.r.t. $\theta$ via the reparametrization trick.
		\STATE Update target networks:
		$$
		\phi_{\text{targ}, i} \leftarrow \rho \phi_{\text{targ}, i} + (1 - \rho) \phi_i \quad \forall i \in \{1,
		2\}
		$$
		\ENDFOR
		\ENDIF
		\UNTIL {convergence}
	\end{algorithmic}
\end{algorithm}

\begin{table}[H]
	\caption{Hyperparameters for SAC algorithm}
	\begin{adjustbox}{max width=\textwidth}
		\begin{tabular}{lllllll}
			\toprule
			& Pendulum & 2-tank & 3-wheeled & Cartpole & Lunar & Omnibot \\
			&  & system & robot &  & lander &  \\
			\midrule
			Total timesteps & 1,000,000 & 1,000,000 & 1,000,000 & 1,000,000 & 1,000,000 & 1,000,000 \\ 
			Buffer size & 1,000,000 & 1,000,000 & 1,000,000 & 1,000,000 & 1,000,000 & 1,000,000 \\ 
			Gamma ($\gamma$) & 0.99 & 0.99 & 0.99 & 0.99 & 0.993 & 0.99 \\ 
			Tau ($\tau$) & 0.005 & 0.005 & 0.005 & 0.005 & 0.005 & 0.005 \\ 
			Batch size & 256 & 256 & 256 & 256 & 256 & 256 \\ 
			Learning starts & 5,000 & 5,000 & 5,000 & 5,000 & 5,000 & 5,000 \\ 
			Policy learning rate ($\text{lr}_{\text{policy}}$) & 0.00079 & 3.0E-4 & 3.0E-4 & 3.0E-4 & 0.001 & 3.0E-4 \\ 
			Q-function learning rate ($\text{lr}_{\text{q}}$) & 0.00025 & 1.0E-3 & 1.0E-3 & 1.0E-3 & 1.0E-3 & 1.0E-3 \\ 
			Policy frequency & 2 & 2 & 2 & 2 & 2 & 2 \\ 
			Target network frequency & 1 & 1 & 1 & 1 & 1 & 1 \\ 
			Alpha ($\alpha$) & 0.0085 & 0.2 & 0.2 & 0.2 & 1.0E-3 & 0.2 \\ 
			Alpha learnable & False & True & True & True & False & True \\ 
			\bottomrule 
		\end{tabular}
	\end{adjustbox}
\end{table}

\subsection{Twin Delayed DDPG (TD3)}
Input for TD3 algorithm:
\begin{itemize}
	\item $\theta$: Initial policy parameters
	\item $\phi_1, \phi_2$: Initial Q-function parameters
	\item $\mathcal{D}$: Empty replay buffer
\end{itemize}

\begin{algorithm}[H]
	\caption{Twin Delayed DDPG (TD3)}
	\begin{algorithmic}[1]
		\STATE $\theta_{\text{targ}} \leftarrow \theta$, $\phi_{\text{targ}, 1} \leftarrow \phi_1$, $\phi_{\text{targ}, 2} \leftarrow \phi_2$ // Initialize target parameters
		\REPEAT
		\STATE Observe state $s$ and select action $a = \operatorname{clip}(\mu_\theta(s) + \epsilon, a_{\text{Low}}, a_{\text{High}})$, where $\epsilon \sim \mathcal{N}$
		\STATE Execute $a$ in the environment
		\STATE Observe next state $s'$, reward $r$, and done signal $d$
		\STATE Store $(s, a, r, s', d)$ in replay buffer $\mathcal{D}$
		\IF {$s'$ is terminal}
		\STATE Reset environment state
		\ENDIF
		\IF {it's time to update}
		\FOR {update step $j$ in range(however many updates)}
		\STATE Randomly sample a batch of transitions $B = \{(s, a, r, s', d)\}$ from $\mathcal{D}$
		\STATE Compute target actions:
		$$
		a'(s') = \operatorname{clip}(\mu_{\theta_{\text{targ}}}(s') + \operatorname{clip}(\epsilon, -c, c), a_{\text{Low}}, a_{\text{High}}), \quad \epsilon \sim \mathcal{N}(0, \sigma)
		$$
		\STATE Compute targets:
		$$
		y(r, s', d) = r + \gamma(1 - d)\min_{i=1,2} Q_{\phi_{\text{targ}, i}}(s', a'(s'))
		$$
		\STATE Update Q-functions by one step of gradient descent:
		$$
		\nabla_{\phi_i} \frac{1}{|B|} \sum_{(s, a, r, s', d) \in B} \left(Q_{\phi_i}(s, a) - y(r, s', d)\right)^2 \quad \forall i \in \{1, 2\}
		$$
		\IF {$j \bmod \texttt{policy\_delay} = 0$}
		\STATE Update policy by one step of gradient ascent:
		$$
		\nabla_\theta \frac{1}{|B|} \sum_{s \in B} Q_{\phi_1}(s, \mu_\theta(s))
		$$
		\STATE Update target networks:
		$$
		\begin{aligned}
			\phi_{\text{targ}, i} & \leftarrow \rho \phi_{\text{targ}, i} + (1 - \rho) \phi_i \quad \forall i \in \{1, 2\} \\
			\theta_{\text{targ}} & \leftarrow \rho \theta_{\text{targ}} + (1 - \rho) \theta
		\end{aligned}
		$$
		\ENDIF
		\ENDFOR
		\ENDIF
		\UNTIL {convergence}
	\end{algorithmic}
\end{algorithm}

\begin{table}[H]
	\caption{Hyperparameters for TD3 algorithm}
	\begin{adjustbox}{max width=\textwidth}
		\begin{tabular}{lllllll}
			\toprule
			& Pendulum & 2-tank & 3-wheeled & Cartpole & Lunar & Omnibot \\
			&  & system & robot &  & lander &  \\
			\midrule
			Total timesteps & 1,000,000 & 1,000,000 & 1,000,000 & 1,000,000 & 1,000,000 & 1,000,000 \\ 
			Learning rate ($\text{lr}$) & 3.0E-4 & 3.0E-4 & 3.0E-4 & 3.0E-4 & 3.0E-4 & 3.0E-4 \\ 
			Buffer size & 1,000,000 & 1,000,000 & 1,000,000 & 1,000,000 & 1,000,000 & 1,000,000 \\ 
			Gamma ($\gamma$) & 0.99 & 0.99 & 0.99 & 0.99 & 0.99 & 0.99 \\ 
			Tau ($\tau$) & 0.005 & 0.005 & 0.005 & 0.005 & 0.005 & 0.005 \\ 
			Batch size & 256 & 256 & 256 & 256 & 256 & 256 \\ 
			Policy noise & 0.2 & 0.2 & 0.2 & 0.2 & 0.2 & 0.2 \\ 
			Exploration noise & 0.1 & 0.1 & 0.1 & 0.1 & 0.1 & 0.1 \\ 
			Learning starts & 25,000 & 25,000 & 25,000 & 25,000 & 25,000 & 25,000 \\ 
			Policy frequency & 2 & 2 & 2 & 2 & 2 & 2 \\ 
			Noise clip & 0.5 & 0.5 & 0.5 & 0.5 & 0.5 & 0.5 \\ 
			\bottomrule 
		\end{tabular}
	\end{adjustbox}
\end{table}

\subsection{Proximal Policy Optimization (PPO)}
Input for PPO algorithm:
\begin{itemize}
	\item $\theta_1$ is the initial policy weights
	\item The policy $\policy^{\theta}(\bullet \mid \State)$ is represented by a truncated normal distribution\\ $\operatorname{Truncated}[\normpdf(\mu^{\theta}(\State), \sigma^2I)$], where $\sigma$ is a hyperparameter, $\mu^{\theta}$ is a neural network with weights $\theta$, and $I$ is the identity matrix.
	\item $\hat\Value^w(\State)$ is the model of the value function, represented by a neural network with weights $w$
	\item $M$ is the number of episodes
	\item $\mathcal{I}$ is the number of iterations
	\item $\eplen$ is the number of steps per episode
	\item $\gamma$ is the discount factor
	\item $\lract$ is the policy learning rate for Adam optimizer
	\item $\lrcrit$ is the critic learning rate for Adam optimizer
	\item $N^{\crit}_{\operatorname{epochs}}$ is the number of epochs for the critic
	\item $N^{\act}_{\operatorname{epochs}}$ is the number of epochs for the policy
	\item $\lambda$ is the GAE parameter
\end{itemize}

\begin{algorithm}[H]
	\caption{PPO}
	\begin{algorithmic}[1]
		\FOR {learning iteration $i := 1 \dots \mathcal I$}
		\FOR {episode $j := 1 \dots M$}
		\STATE obtain initial state $\State_0^{j}$
		\FOR {step $t := 0 \dots \eplen - 1$}
		\STATE sample action $\Action_t^j \sim \policy^{\theta}(\bullet \mid \State_t^{j})$
		\STATE transition to next state $t + 1$ from transition function $\State_{t+1}^j \sim \transit(\bullet \mid \State_t^j, \Action_t^j)$
		\ENDFOR
		\ENDFOR
		
\STATE  Optimize critic $\hat\Value^w$ with Adam optimizer:
\FOR {$1 \dots N^{\crit}_{\operatorname{epochs}}$}
\STATE Perform a gradient descent step:
$$w^{new} \la w^{old} - \lrcrit\nabla_w \loss_{\crit}(w) \big|_{w = w^{old}},$$
where $\loss_{\crit}(w)$ is a temporal difference loss:
\begin{equation}
	\label{eqn_td_loss}
	\sum_{j=1}^M\sum_{t=0}^{T-1 - N_{\text{TD}}}\left(\hat\Value^w(\State_t^j) - \sum_{t'=t}^{t + N_{\text{TD}} - 1} \gamma^{t'-t}\robj(\State_{t'}^j, \Action_{t'}^j) - \gamma^{N_{\text{TD}}}\hat\Value^w(\State_{t + N_{\text{TD}}}^j)\right)^2.
\end{equation}
\ENDFOR
\STATE Estimate GAE-Advantages: 
$$
\hat\Advan^{\policy ^{\theta _i}}(\State_t^j, \Action_t^j) := \sum_{t' = t}^{T-1} (\gamma \lambda) ^ {t'} \left(\robj(\State_{t'}^j, \Action_{t'}^j) + \gamma \hat\Value^w(\State_{t' + 1}^j) - \hat\Value^w(\State_{t'}^j)\right)
$$
\STATE Optimize policy $\policy^{\theta}$ with Adam optimizer:
\FOR {$1 \dots N^{\act}_{\operatorname{epochs}}$}
\STATE Perform a gradient ascent step:
$$\theta^{new} \la \theta^{old} + \lract\nabla_{\theta} \loss_{\act}(\theta) \big|_{\theta = \theta^{old}},$$
where $\loss_{\act}(\theta)$ is defined as follows:
\begin{equation}
	\hspace{-1cm} \begin{aligned}
		\frac {1}{M} \sum _{j = 1}^M \sum _{t=0}^{T-1} \gamma ^ t \min \Bigg( \hat\Advan ^{\policy ^{\theta _i}}(\State _t^j, \Action _t^j) \frac {\policy ^{\theta }(\Action _t^j \mid \State _t^j)}{\policy ^{\theta _i}(\Action _t^j \mid \State _t^j)}, 
		\hat\Advan ^{\policy ^{\theta _i}}(\State _t^j, \Action _t^j) \operatorname {clip}_{1 - \varepsilon }^{1 + \varepsilon }\left (\frac {\policy ^{\theta }(\Action _t^j \mid \State _t^j)}{\policy ^{\theta _i}(\Action _t^j \mid \State _t^j)}\right ) \Bigg).
		\end{aligned}
	\end{equation}
\ENDFOR
\STATE Denote $\theta^{i+1}$ as the latest value of $\theta^{new}$ in the \textbf{for} loop above.
\ENDFOR
\end{algorithmic}
\end{algorithm}

\begin{table}[H]
    \caption{Hyperparameters for PPO algorithm}
\begin{adjustbox}{max width=\textwidth}
\begin{tabular}{lllllll}
    \toprule
    & Inverted & 2-tank & 3-wheeled & Cartpole & Lunar & Omnibot \\
    & pendulum & system & robot &  & lander &  \\
    \midrule
    Sampling rate in Hz & 100 & 10 & 100 & 100 & 100 & 100\\ 
    Steps per episode ($\eplen$) & 1000 & 800 & 500 & 1500 & 1000 & 1000\\ 
    Number of episodes & 2 & 2 & 10 & 5 & 10 & 2\\ 
    Discount factor & 0.9964 & 0.9895 & 0.9964 & 0.9989 & 0.9964 & 0.9964\\ 
    Critic TD order ($N_{\text{TD}}$) & 1 & 70 & 1 & 1 & 1 & 1\\ 
    Critic hidden sizes & [100,100,100,100] & [100,50,10] & [15,15] & [100,50] & [15,15] & [100,50,10]\\ 
    Critic learning rate ($\lrcrit$) & 0.001 & 0.001 & 0.1 & 0.01 & 0.1 & 0.1\\ 
    Critic epochs ($N^{\crit}_{\text{epochs}}$) & 50 & 50 & 30 & 50 & 30 & 50\\ 
    Policy hidden sizes & [4] & [4] & [15, 15] & [32, 32] & [15, 15] & [4]\\ 
    Policy learning rate ($\lract$) & 0.01 & 0.01 & 0.0005 & 0.003 & 0.0005 & 0.005\\ 
    Policy epochs ($N^{\act}_{\text{epochs}}$) & 50 & 50 & 30 & 50 & 30 & 50 \\ 
    \bottomrule 
\end{tabular}
\end{adjustbox}
\end{table}

\newpage
\subsection{Vanilla Policy Gradient (VPG)}
Input for VPG algorithm:
\begin{itemize}
	\item $\theta_1$ is the initial policy weights
	\item The policy $\policy^{\theta}(\bullet \mid \State)$ is represented by a truncated normal distribution\\ $\operatorname{Truncated}[\normpdf(\mu^{\theta}(\State), \sigma^2I)$], where $\sigma$ is a hyperparameter, $\mu^{\theta}$ is a neural network with weights $\theta$, and $I$ is the identity matrix.
	\item $\hat\Value^w(\State)$ is the model of the value function, represented by a neural network with weights $w$
	\item $M$ is the number of episodes
	\item $\mathcal{I}$ is the number of iterations
	\item $\eplen$ is the number of steps per episode
	\item $\gamma$ is the discount factor
	\item $\lract$ is the policy learning rate for Adam optimizer
	\item $\lrcrit$ is the critic learning rate for Adam optimizer
	\item $N^{\crit}_{\operatorname{epochs}}$ is the number of epochs for the critic
	\item $\lambda$ is the GAE parameter
\end{itemize}

\begin{algorithm}[H]
    \caption{VPG}
    \begin{algorithmic}[1]
    \FOR {learning iteration $i := 1 \dots \mathcal I$}
        \FOR {episode $j := 1 \dots M$}
            \STATE obtain initial state $\State_0^{j}$
            \FOR {step $t := 0 \dots \eplen - 1$}
                \STATE sample action $\Action_t^j \sim \policy^{\theta}(\bullet \mid \State_t^{j})$
                \STATE transition to next state $t + 1$ from transition function $\State_{t+1}^j \sim \transit(\bullet \mid \State_t^j, \Action_t^j)$
            \ENDFOR
        \ENDFOR
        \STATE  Optimize critic $\hat\Value^w$ with Adam optimizer:
        \FOR {$1 \dots N^{\crit}_{\operatorname{epochs}}$}
            \STATE Perform a gradient decsent step:
            $$w^{new} \la w^{old} - \lrcrit\nabla_w \loss_{\crit}(w) \big|_{w = w^{old}},$$
            where $\loss_{\crit}(w)$ is a temporal difference loss:
            \begin{equation}
                \label{eqn_td_loss}
                \sum_{j=1}^M\sum_{t=0}^{T-1 - N_{\text{TD}}}\left(\hat\Value^w(\State_t^j) - \sum_{t'=t}^{t + N_{\text{TD}} - 1} \gamma^{t'-t}\robj(\State_{t'}^j, \Action_{t'}^j) - \gamma^{N_{\text{TD}}}\hat\Value^w(\State_{t + N_{\text{TD}}}^j)\right)^2.
            \end{equation}
        \ENDFOR
        \STATE Estimate GAE-Advantages: \\
        $
        \hat\Advan^{\policy ^{\theta _i}}(\State_t^j, \Action_t^j) := \sum_{t' = t}^{T-1} (\gamma \lambda) ^ {t'} \left(\robj(\State_{t'}^j, \Action_{t'}^j) + \gamma \hat\Value^w(\State_{t' + 1}^j) - \hat\Value^w(\State_{t'}^j)\right)
        $
        \STATE Perform a gradient ascent step:
        \begin{equation} 
            \theta_{i + 1} \la \theta_i + \lract \frac {1}{M} \sum _{j = 1}^M \sum _{t=0}^{T-1} \gamma ^ t \hat \Advan ^{\policy ^{\theta }}(\State _t^j, \Action _t^j) \nabla _{\theta } \log \policy ^{\theta }(\Action _t^j \mid \State _t^j)\big|_{\theta = \theta_i}.
        \end{equation}
    \ENDFOR
    \end{algorithmic}
\end{algorithm}

\begin{table}[H]
    \caption{Hyperparameters for VPG algorithm}
\begin{adjustbox}{max width=\textwidth}
\begin{tabular}{lllllll}
    \toprule
    & Inverted & 2-tank & 3-wheeled & Cartpole & Lunar & Omnibot \\
    & pendulum & system & robot &  & lander &  \\
    \midrule
    Sampling rate in Hz & 100 & 10 & 100 & 100 & 100 & 100\\ 
    Steps per episode ($\eplen$) & 1000 & 800 & 500 & 1500 & 1000 & 1000\\ 
    Number of episodes & 2 & 2 & 10 & 5 & 10 & 2\\ 
    Discount factor & 0.9964 & 0.9895 & 0.9964 & 0.9989 & 0.9964 & 0.9964\\ 
    Critic TD order ($N_{\text{TD}}$) & 1 & 70 & 1 & 1 & 1 & 1\\ 
    Critic hidden sizes & [100,100,100,100] & [100,50,10] & [15,15] & [100,50] & [15,15] & [100,50,10]\\ 
    Critic learning rate ($\lrcrit$) & 0.001 & 0.001 & 0.1 & 0.01 & 0.1 & 0.1\\ 
    Critic epochs ($N^{\crit}_{\text{epochs}}$) & 50 & 50 & 30 & 50 & 30 & 50\\ 
    Policy hidden sizes & [4] & [4] & [15, 15] & [32, 32] & [15, 15] & [4]\\ 
    Policy learning rate ($\lract$) & 0.01 & 0.01 & 0.0005 & 0.003 & 0.0005 & 0.005\\ 
    Policy epochs ($N^{\act}_{\text{epochs}}$) & 50 & 50 & 30 & 50 & 30 & 50 \\ 
    \bottomrule 
\end{tabular}
\end{adjustbox}
\end{table}

\newpage
\subsection{Deep Deterministic Policy Gradient (DDPG)}
\begin{algorithm}
    \caption{DDPG}
    \begin{algorithmic}[1]
    \STATE {\bfseries Input:} 
    \begin{itemize}
        \item $\theta_1$ is the initial policy weights
        \item The policy $\policy^{\theta}(\State)$ is represented by a neural network with weights $\theta$
        \item $M$ is the number of episodes
        \item $\mathcal{I}$ is the number of iterations
        \item $\eplen$ is the number of steps per episode
        \item $\gamma$ is the discount factor
        \item $\lract$ is the policy learning rate for Adam optimizer
    \end{itemize}
    \FOR {learning iteration $i := 1 \dots \mathcal I$}
        \FOR {episode $j := 1 \dots M$}
            \STATE obtain initial state $\State_0^{j}$
            \FOR {step $t := 0 \dots \eplen - 1$}
                \STATE sample action $\Action_t^j = \policy^{\theta}(\State_t^j)$
                \STATE (optional) add exploration noise $\Action_t^j := \Action_t^j + \varepsilon$, where $\varepsilon \sim \normpdf(0, \sigma_{\operatorname{expl}}^2)$
                \STATE transition to next state $t + 1$ by transition function $\State_{t+1}^j \sim \transit(\bullet \mid \State_t^j, \Action_t^j)$
            \ENDFOR
        \ENDFOR
        \STATE  Perform a gradient ascent step with Adam optimizer:
        $$
        \theta _{i+1} \la \theta _i + \frac {1}{M} \sum _{j = 1}^M \sum _{t=0}^{T-1} \gamma ^ t \nabla _{\theta }\policy ^{\theta }(\State _t^j)\nabla _{\action } \hat Q^{w}(\State _t, \action )\big \rvert _{\action = \Action _t^j},
        $$
    \ENDFOR
    \end{algorithmic}
\end{algorithm}

\pagebreak

\begin{table}[H]
 \caption{Hyperparameters for DDPG algorithm}   
\begin{adjustbox}{max width=\textwidth}
\begin{tabular}{lllllll}
    \toprule
    & Inverted & 2-tank & 3-wheeled & Cartpole & Lunar & Omnibot \\
    & pendulum & system & robot &  & lander &  \\
    \midrule
    Sampling rate in Hz & 100 & 10 & 100 & 100 & 100 & 100\\ 
    Steps per episode ($\eplen$) & 1000 & 800 & 500 & 1500 & 1000 & 1000\\ 
    Number of episodes & 2 & 2 & 10 & 5 & 10 & 2\\ 
    Discount factor & 0.9964 & 0.9895 & 0.9964 & 0.9989 & 0.9964 & 0.9964\\ 
    Critic TD order ($N_{\text{TD}}$) & 1 & 70 & 1 & 1 & 1 & 1\\ 
    Critic hidden sizes & [100, 100, 100, 100] & [100, 50, 10] & [15, 15] & [100, 50] & [15, 15] & [100, 50, 10]\\ 
    Critic learning rate ($\lrcrit$) & 0.001 & 0.001 & 0.1 & 0.01 & 0.1 & 0.1\\ 
    Critic epochs ($N^{\crit}_{\text{epochs}}$) & 50 & 50 & 30 & 50 & 30 & 50\\ 
    Policy hidden sizes & [4] & [4] & [15, 15] & [32, 32] & [15, 15] & [4]\\ 
    Policy learning rate ($\lract$) & 0.01 & 0.01 & 0.0005 & 0.003 & 0.0005 & 0.005\\ 
    Policy epochs ($N^{\act}_{\text{epochs}}$) & 50 & 50 & 30 & 50 & 30 & 50 \\ 
    \bottomrule 
\end{tabular}
\end{adjustbox}
\end{table}

\newpage 

\subsection{REINFORCE}

\begin{algorithm}
    \caption{REINFORCE}
    \begin{algorithmic}[1]
    \STATE {\bfseries Input:} 
    \begin{itemize}
        \item $\theta_1$ is the initial policy weights
        \item $B_t^1 := 0$ is the initial baseline for the 1st iteration
        \item The policy $\policy^{\theta}(\bullet \mid \State)$ is represented by a truncated normal distribution $\operatorname{Truncated}[\normpdf(\mu_{\theta}(\State), \sigma^2I)$], where $\sigma$ is a hyperparameter, $\mu^{\theta}$ is a neural network, and $I$ is the identity matrix.
        \item $M$ is the number of episodes
        \item $\mathcal{I}$ is the number of iterations
        \item $\eplen$ is the number of steps per episode
        \item $\gamma$ is the discount factor
        \item $\lract$ is the policy learning rate for Adam optimizer
    \end{itemize}
    \FOR {learning iteration $i := 1 \dots \mathcal I$}
        \FOR {episode $j := 1 \dots M$}
            \STATE obtain initial state $\State_0^{j}$
            \FOR {step $t := 0 \dots \eplen - 1$}
                \STATE sample action $\Action_t^j \sim \policy^{\theta}(\bullet \mid \State_t^{j})$
                \STATE transition to next state $t + 1$ by transition function $\State_{t+1}^j \sim \transit(\bullet \mid \State_t^j, \Action_t^j)$
            \ENDFOR
        \ENDFOR
        \STATE  Perform a gradient ascent step with Adam optimizer:
        $$
            \begin{aligned}\theta_{i+1} \la & \theta_i + \\ & \lract \frac{1}{M}\sum_{j = 1}^M \sum_{t = 0}^{\eplen-1}\sum_{t'=t}^{\eplen-1}\left( \gamma^{t'} \robj(\State_{t'}^j, \Action_{t'}^j) - B_{t}^i\right) \nabla_{\theta}\log \policy^{\theta}(A_t^j \mid \State_t^j)\rvert_{\theta = \theta_i}, \end{aligned}
        $$
        \STATE Update baselines for the next iteration: $B_t^{i + 1} = \frac{1}{M}\sum_{j = 1}^M \sum_{t' = t}^{\eplen} \gamma^{t'} \robj(\State_{t'}^j, \Action_{t'}^j)$
    \ENDFOR
    \end{algorithmic}
\end{algorithm}

\begin{table}[H]
    \caption{Hyperparameters for REINFORCE algorithm}
	\begin{adjustbox}{max width=\textwidth}
    \begin{tabular}{lllllll}
        \toprule
        & Inverted & 2-tank & 3-wheeled & Cartpole & Lunar & Omnibot \\
        & pendulum & system & robot &  & lander &  \\
        \midrule
        Sampling rate in Hz & 100 & 10 & 100 & 100 & 100 & 100\\ 
        Steps per episode ($\eplen$) & 1000 & 800 & 500 & 1500 & 1000 & 1000\\ 
        Number of episodes & 4 & 4 & 4 & 3 & 6 & 4\\ 
        Discount factor & 1.0 & 1.0 & 1.0 & 0.9989 & 1.0 & 1.0\\ 
        Policy hidden sizes & [4] & [4, 4] & [15, 15] & [32, 32] & [4] & [4, 4]\\ 
        Policy learning rate ($\lract$) & 0.1 & 0.1 & 0.01 & 0.05 & 0.1 & 0.1\\ 
        Policy epochs ($N^{\act}_{\text{epochs}}$) & 1 & 1 & 1 & 1 & 1 & 1 \\ 
        \bottomrule 
    \end{tabular}
    \end{adjustbox}
\end{table}

\subsection{Agent by \Crefalgcalfstate}

The algorithm is detailed in the listing \Cref{alg_calfstate}. 
We employed temporal difference loss for the critic and greedy policy optimizing on each step.
We considered both the initial ``daggered'' critic weights and the initial $\relprob$ as episode-to-episode learnable parameters.
The update of the former was triggered by the hyperparameter \texttt{is\_propagate\_safe\_weights}.
The latter \ie the initial $\relprob$ was updated via fixed increments episode-to-episode.
In other words, it ensures that $w^{\dagger}$ is carried forward from one episode to the next.
The \texttt{is\_nominal\_first} parameter, when set to \texttt{True}, triggered $\policy_0$ in the first episode.
This way one may achieve a learning curve starting at the value of $\policy_0$.
The key impact on performance was primarily driven by two main hyperparameters: the range of the initial $\relprob$ defined by \texttt{relax\_probability\_min} and \texttt{relax\_probability\_max}, as well as \texttt{is\_propagate\_safe\_weights}. 
These hyperparameters were easy to tune to achieve fast learning. 
CasADi was employed to implement the critic model.


\begin{table}[H]
    \caption{Hyperparameters for agent by \Crefalgcalfstate}
\begin{adjustbox}{max width=\textwidth}
\begin{tabular}{lllllll}
    \toprule
    & Pendulum & 2-tank & 3-wheeled & Inverted & Lunar & Omnibot \\
    &  & system & robot & pendulum & lander &  \\
    \midrule
    Sampling rate in Hz & 100 & 10 & 100 & 100 & 100 & 100\\ 
    Steps per episode ($\eplen$) & 1000 & 800 & 500 & 1500 & 1000 & 1000\\  
    Discount factor & 1.0 & 1.0 & 1.0 & 1.0 & 1.0 & 1.0\\ 
    Critic TD order ($N_{\text{TD}}$) & 1 & 1 & 2 & 2 & 1 & 1\\ 
    Critic batch size ($\Treplay$) & 10 & 3 & 32 & 3 & 3 & 2 \\ 
    relax\_probability\_min & 0.01 & 0. & 0. & 0.5 & 0. & 0.999 \\
    relax\_probability\_max & 0.999 & 0.8 & 0.49 & 0. & 0.46 & 0.75 \\
    is\_propagate\_safe\_weights & False & False & False & False & True & False \\
    is\_nominal\_first & False & True & False & True & False & True \\
    \bottomrule 
\end{tabular}
\end{adjustbox}
\end{table}

\section{Miscellaneous results}
\label{sec_rawresults}

\begin{figure}[H]
\begin{center}
\centerline{\includegraphics[width=\textwidth]{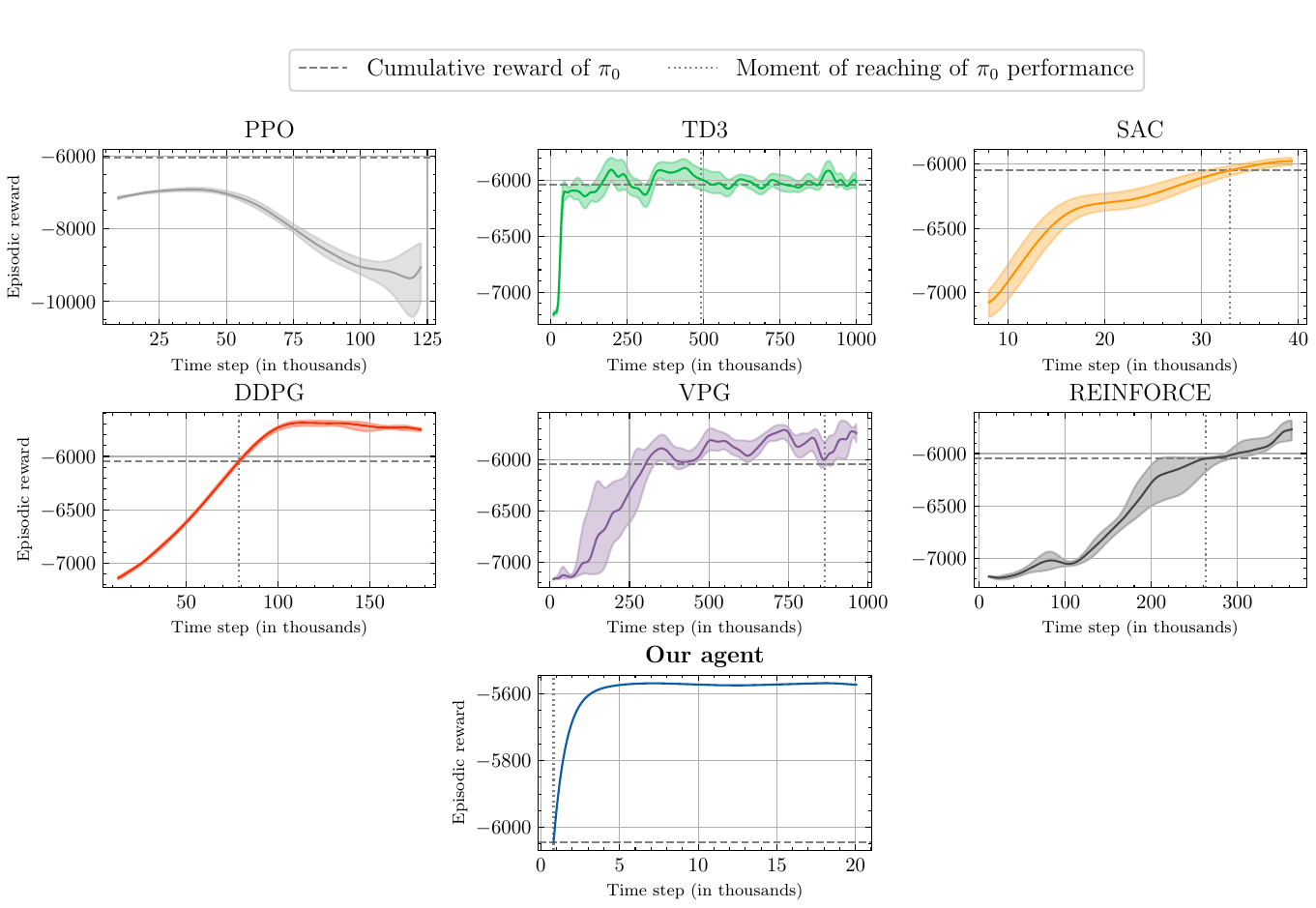}}
\caption{Smoothed learning curves for two-tank system. Each plot is smoothed using a rolling median followed by Bezier interpolation.}
\label{fig_rawlearncurves}
\end{center}
\end{figure}

\begin{figure}[H]
    \begin{center}
    \centerline{\includegraphics[width=\textwidth]{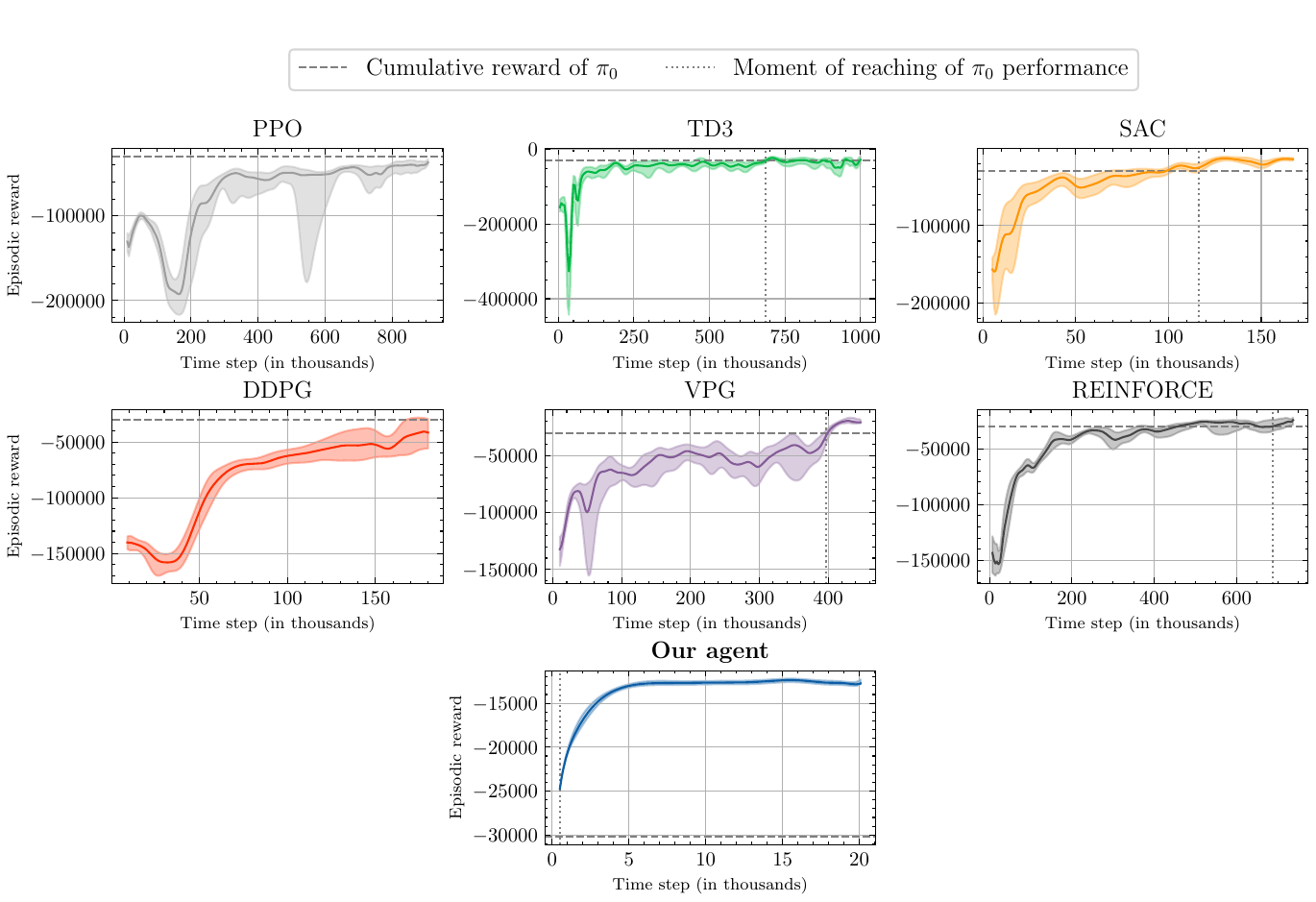}}
    \caption{Smoothed learning curves for three-wheel robot. Each plot is smoothed using a rolling median followed by Bezier interpolation.}
    \label{fig_rawlearncurves}
    \end{center}
\end{figure}

\begin{figure}[H]
    \begin{center}
    \centerline{\includegraphics[width=\textwidth]{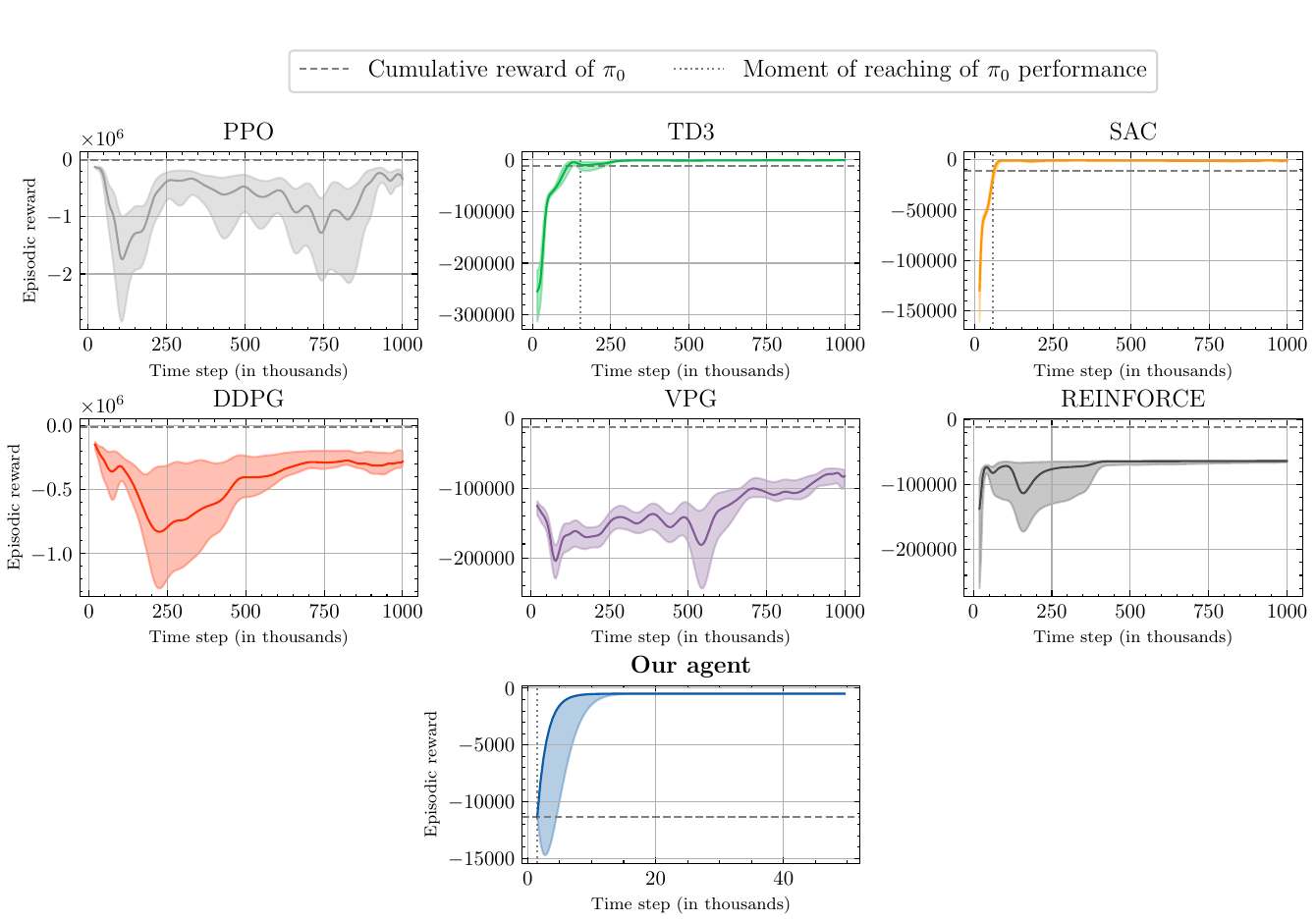}}
    \caption{Smoothed learning curves for inverted pendulum. Each plot is smoothed using a rolling median followed by Bezier interpolation.}
    \label{fig_rawlearncurves}
    \end{center}
\end{figure}

\begin{figure}[H]
    \begin{center}
    \centerline{\includegraphics[width=\textwidth]{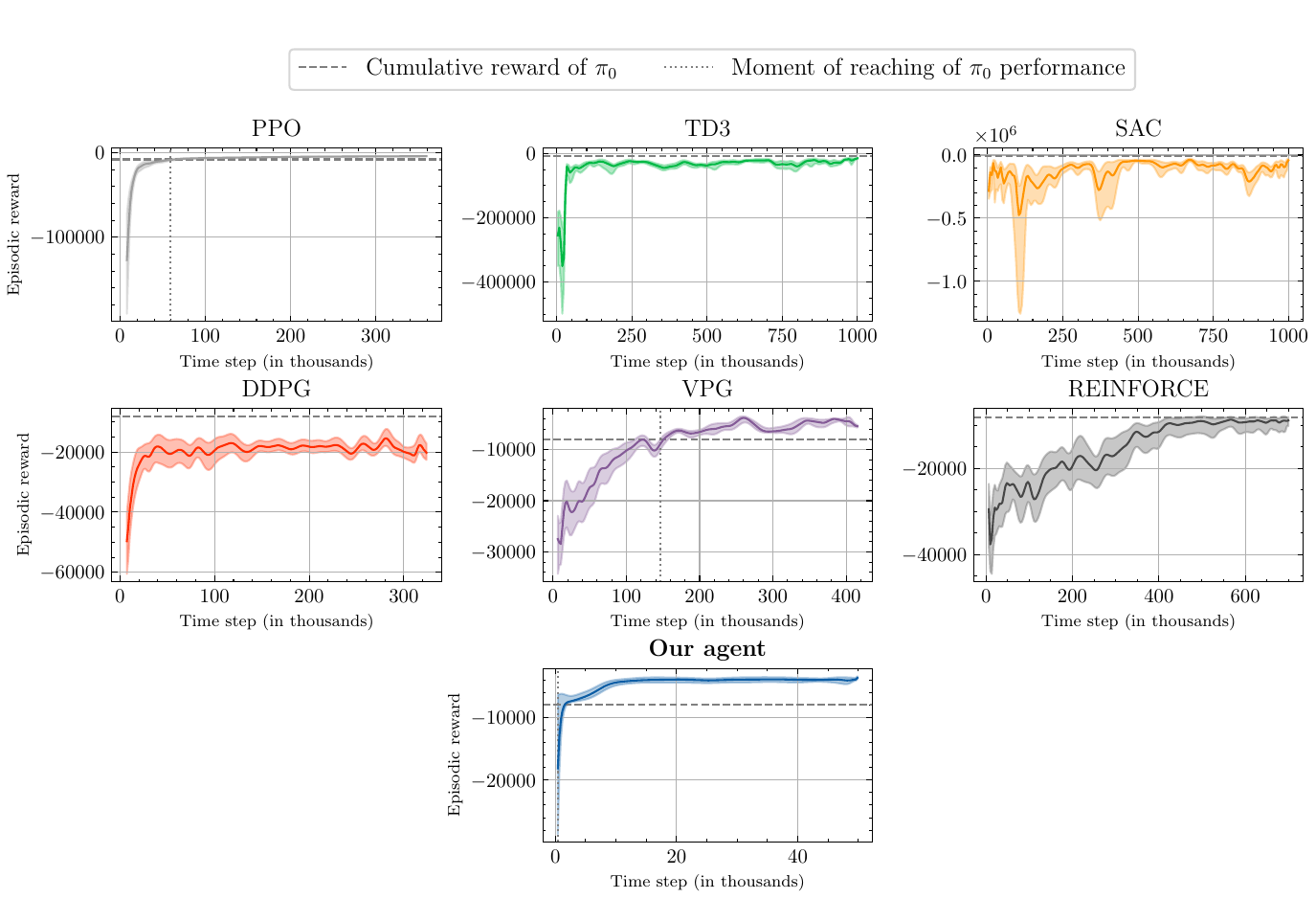}}
    \caption{Smoothed learning curves for lunar lander. Each plot is smoothed using a rolling median followed by Bezier interpolation.}
    \label{fig_rawlearncurves}
    \end{center}
\end{figure}

\begin{figure}[H]
    \begin{center}
    \centerline{\includegraphics[width=\textwidth]{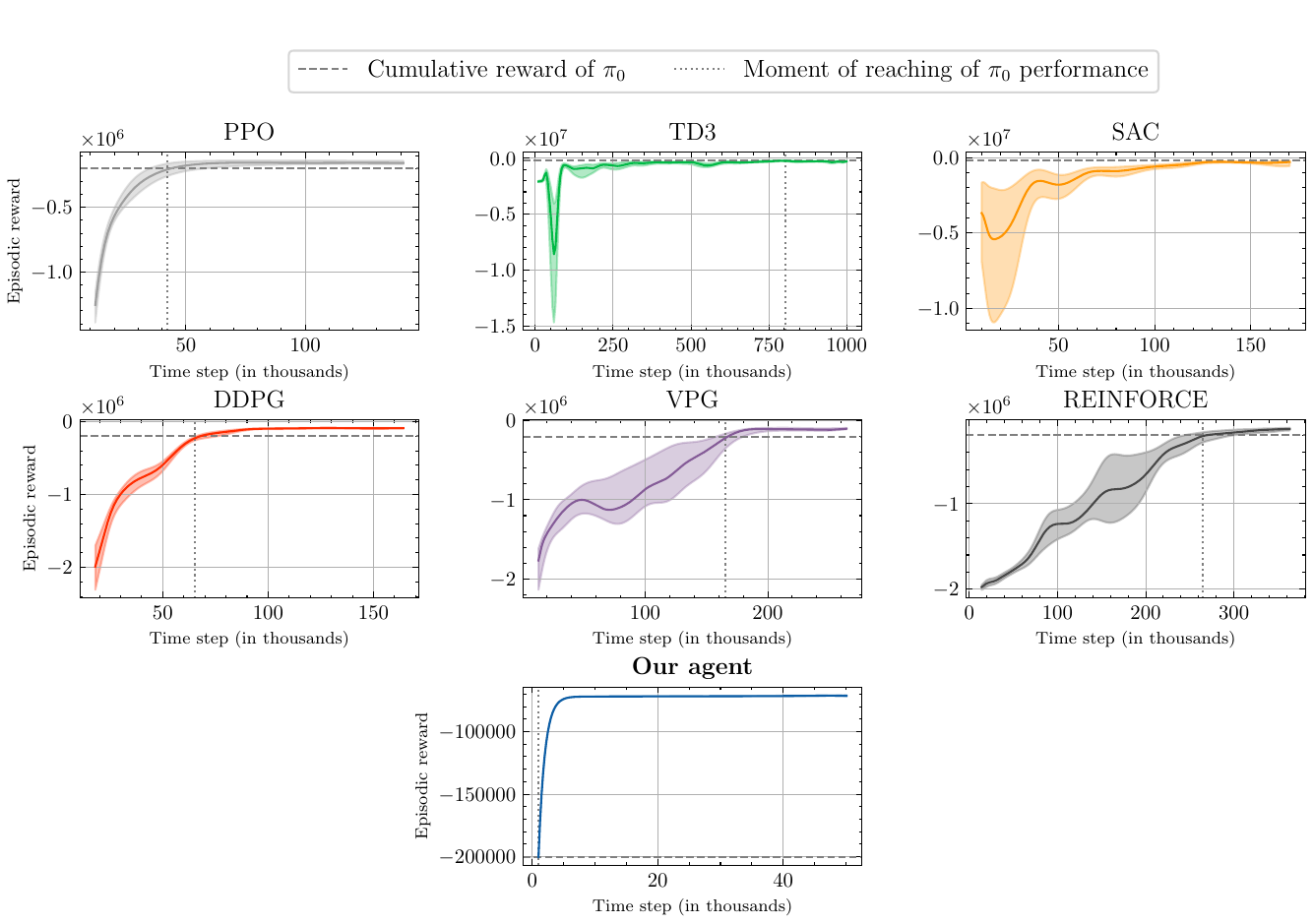}}
    \caption{Smoothed learning curves for omnibot. Each plot is smoothed using a rolling median followed by Bezier interpolation.}
    \label{fig_rawlearncurves}
    \end{center}
\end{figure}

\begin{figure}[H]
    \begin{center}
    \centerline{\includegraphics[width=\textwidth]{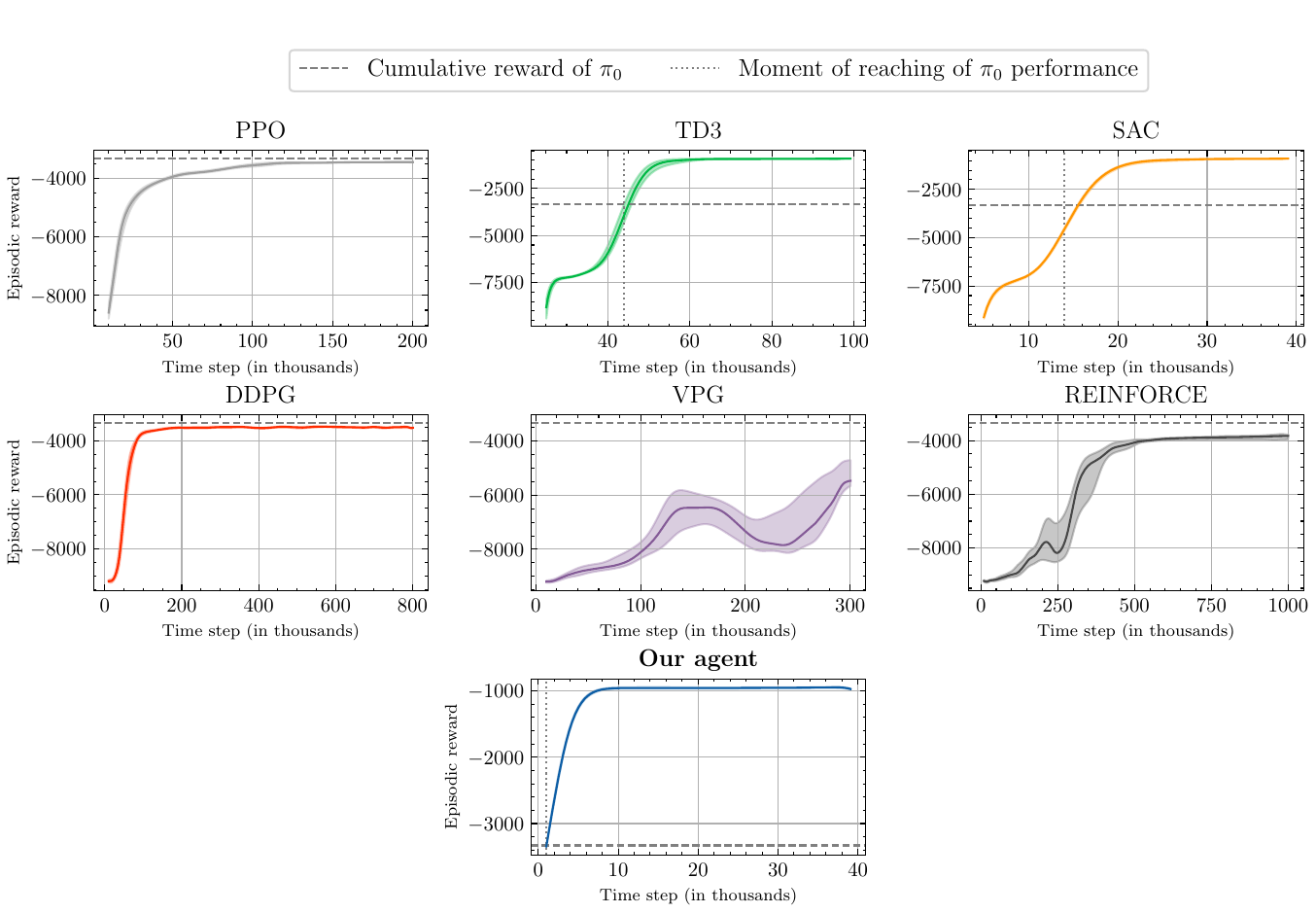}}
    \caption{Smoothed learning curves for pendulum. Each plot is smoothed using a rolling median followed by Bezier interpolation.}
    \label{fig_rawlearncurves}
    \end{center}
\end{figure}
\includegraphics[width=0.5\textwidth]{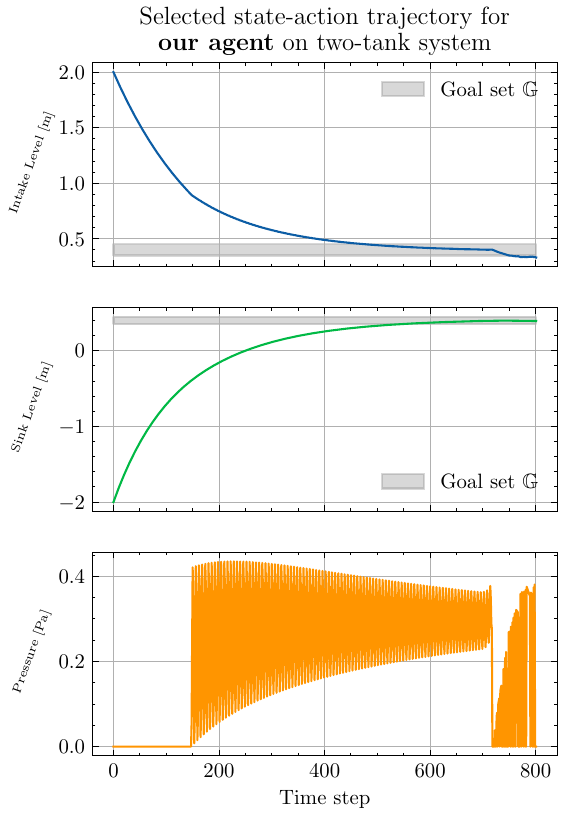}
\includegraphics[width=0.5\textwidth]{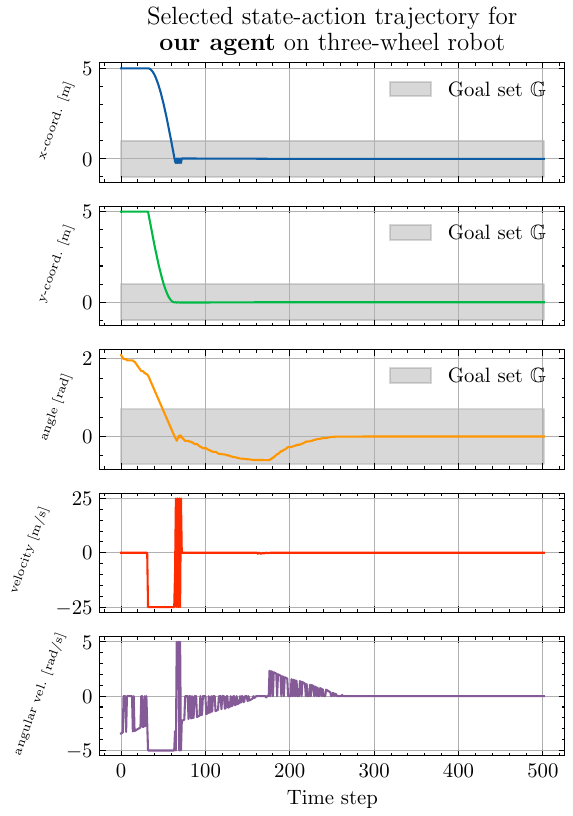}
\includegraphics[width=0.5\textwidth]{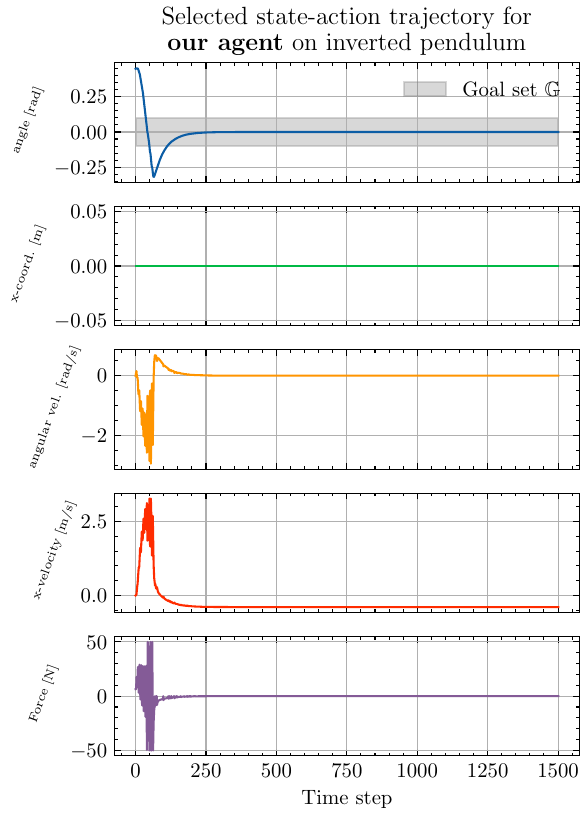}
\includegraphics[width=0.5\textwidth]{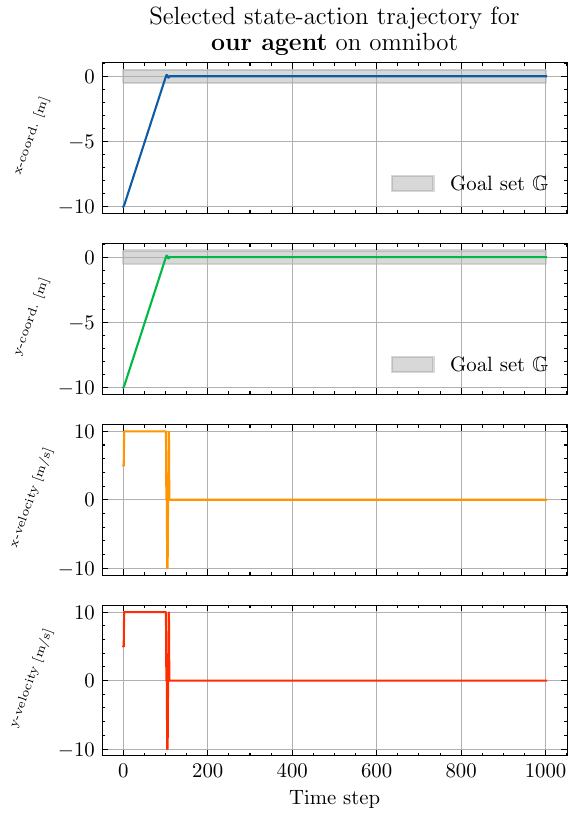}
\includegraphics[width=0.5\textwidth]{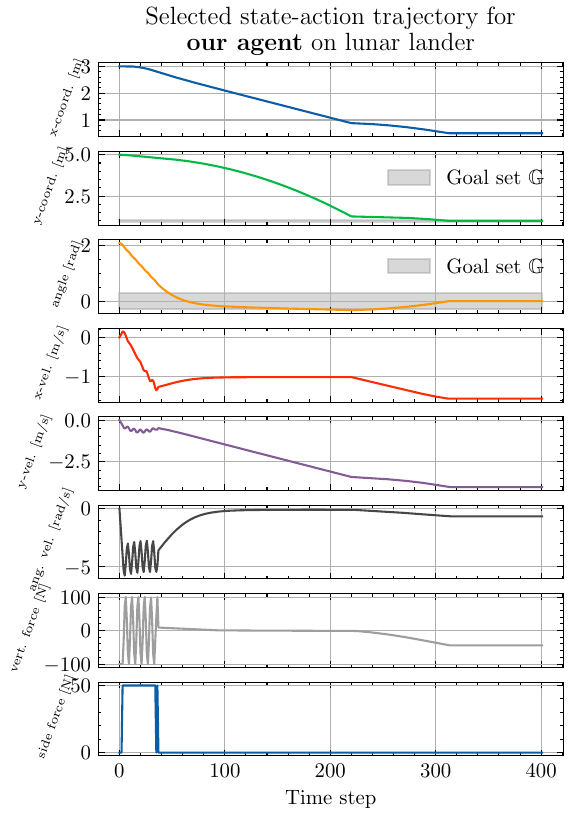}
\includegraphics[width=0.5\textwidth]{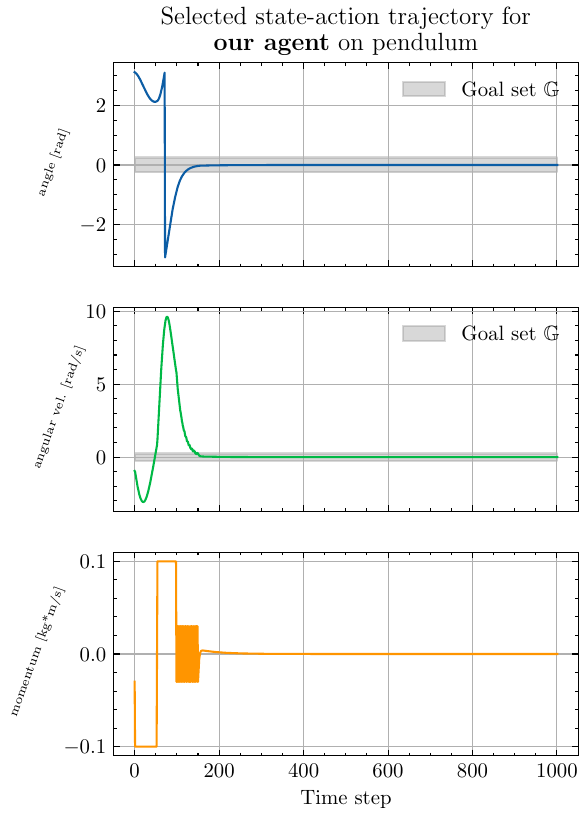}
\includegraphics[width=0.5\textwidth]{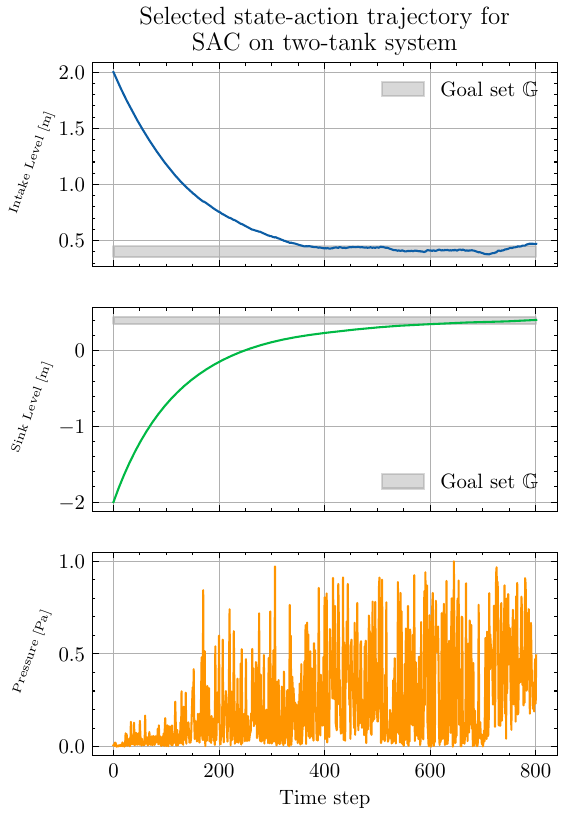}
\includegraphics[width=0.5\textwidth]{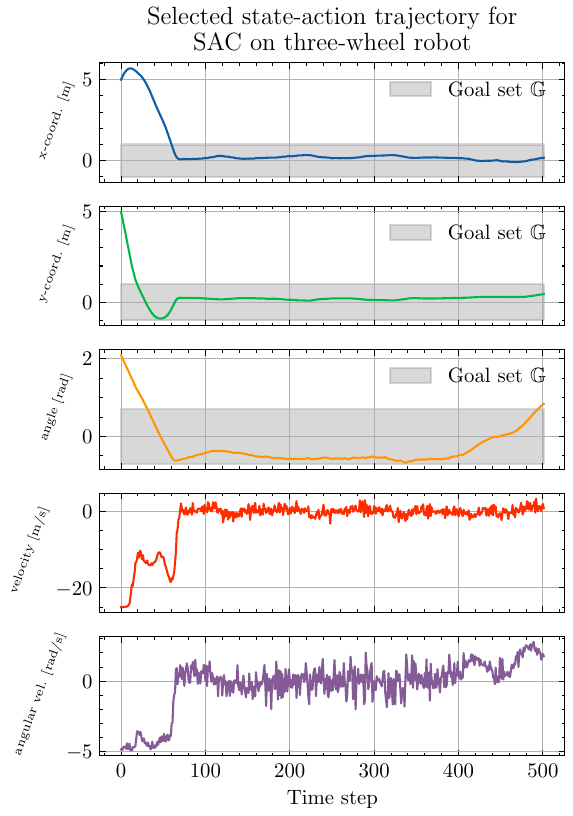}
\includegraphics[width=0.5\textwidth]{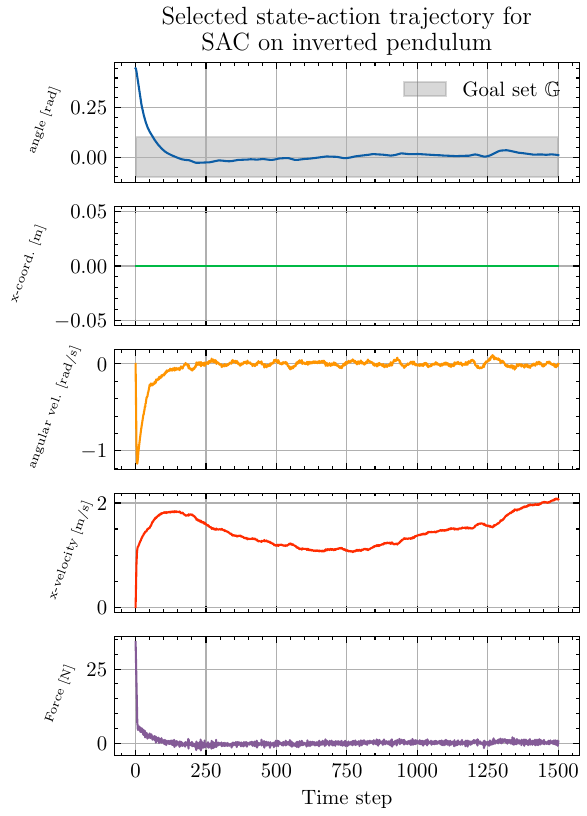}
\includegraphics[width=0.5\textwidth]{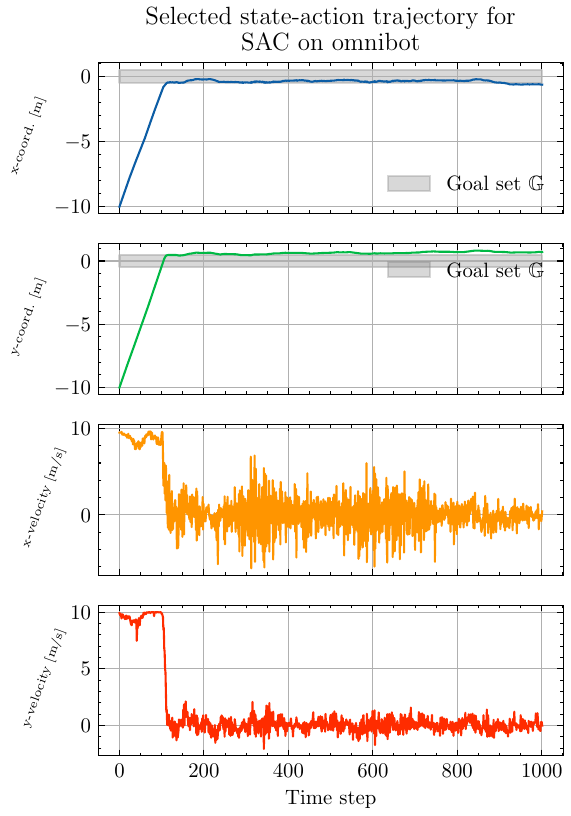}
\includegraphics[width=0.5\textwidth]{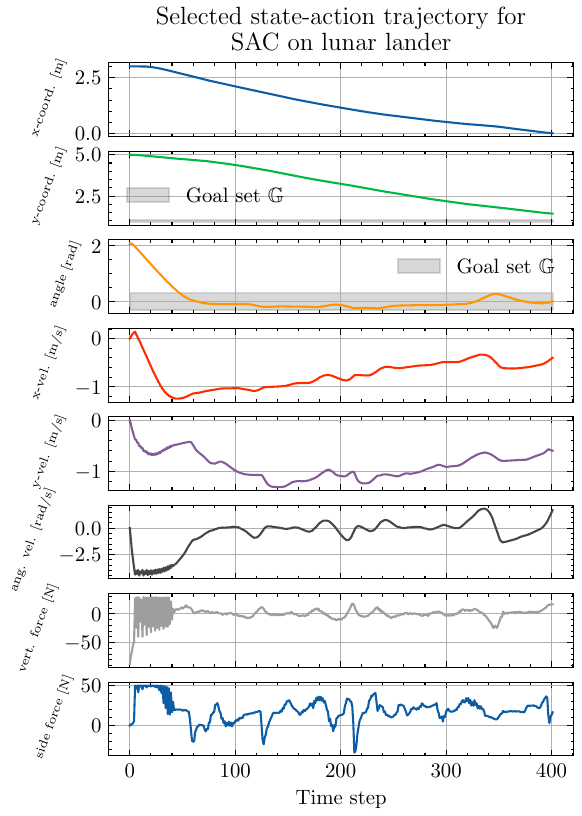}
\includegraphics[width=0.5\textwidth]{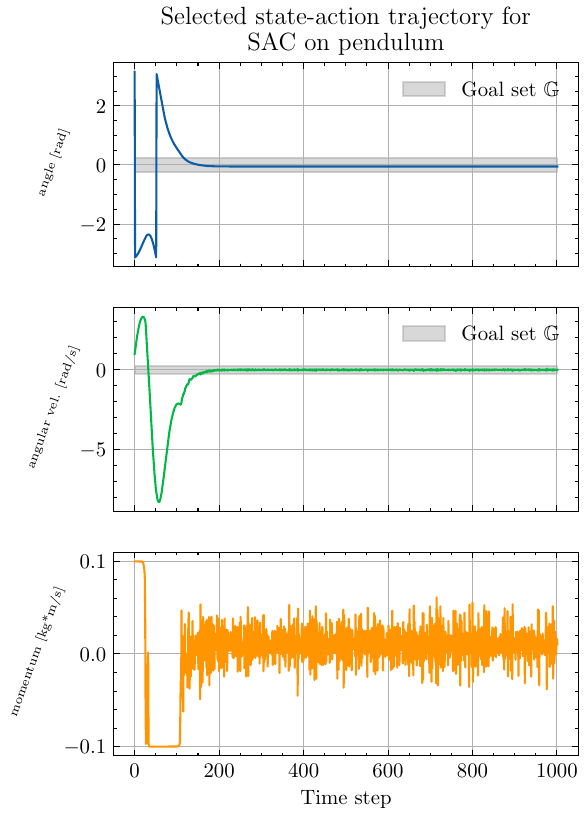}
\includegraphics[width=0.5\textwidth]{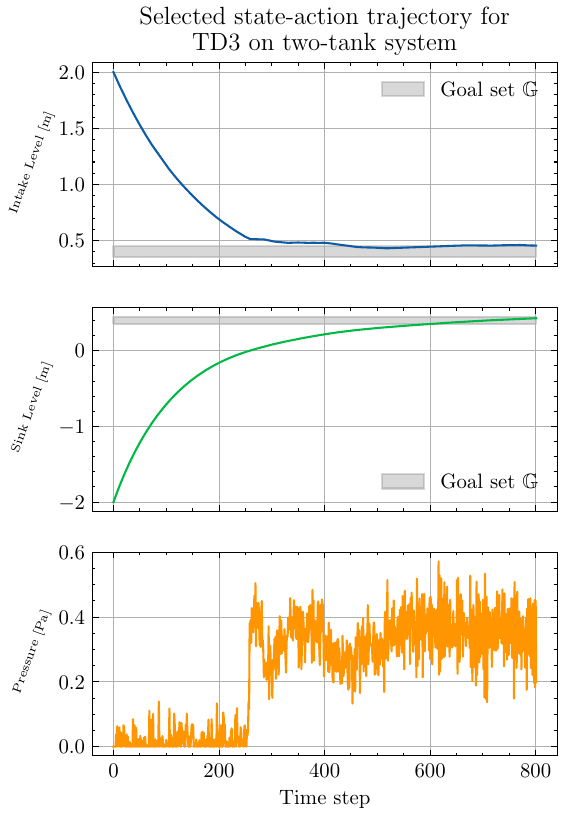}
\includegraphics[width=0.5\textwidth]{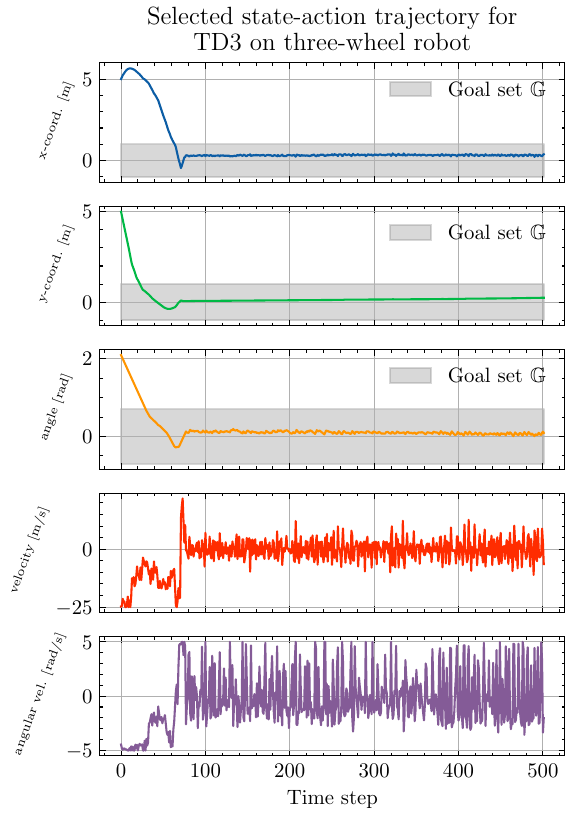}
\includegraphics[width=0.5\textwidth]{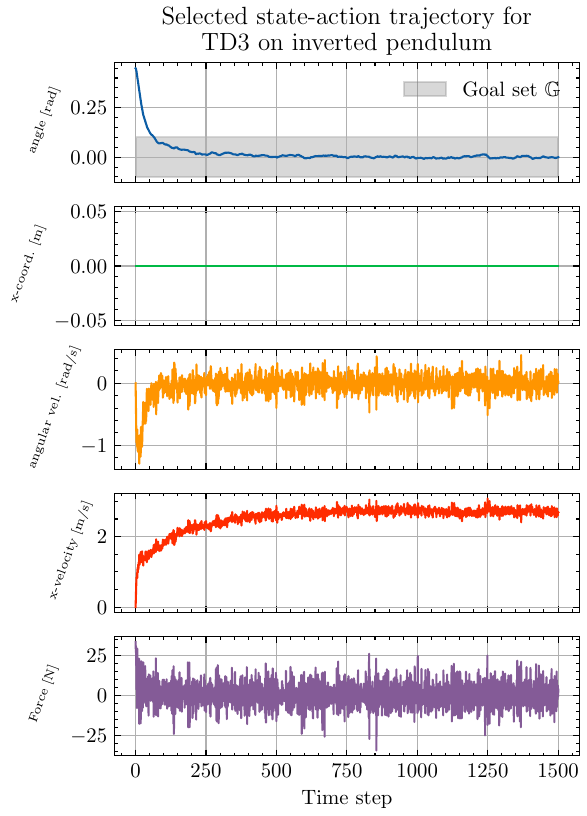}
\includegraphics[width=0.5\textwidth]{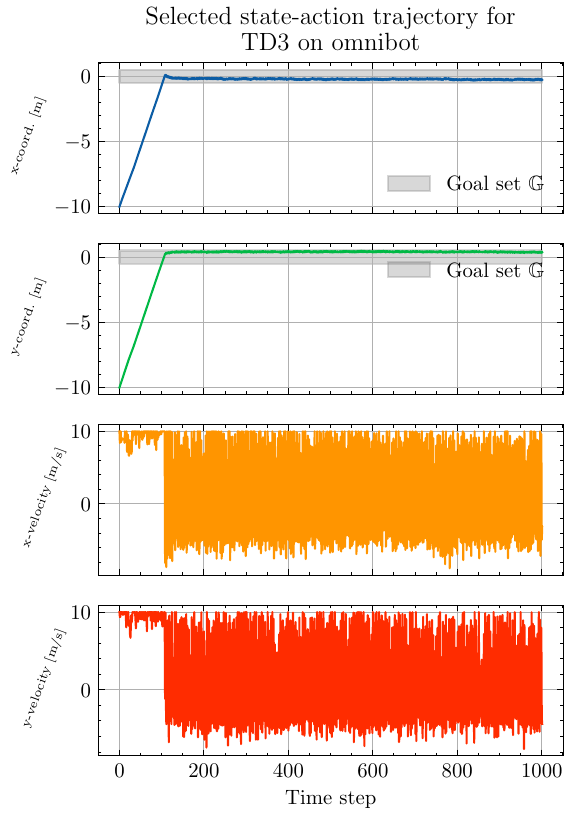}
\includegraphics[width=0.5\textwidth]{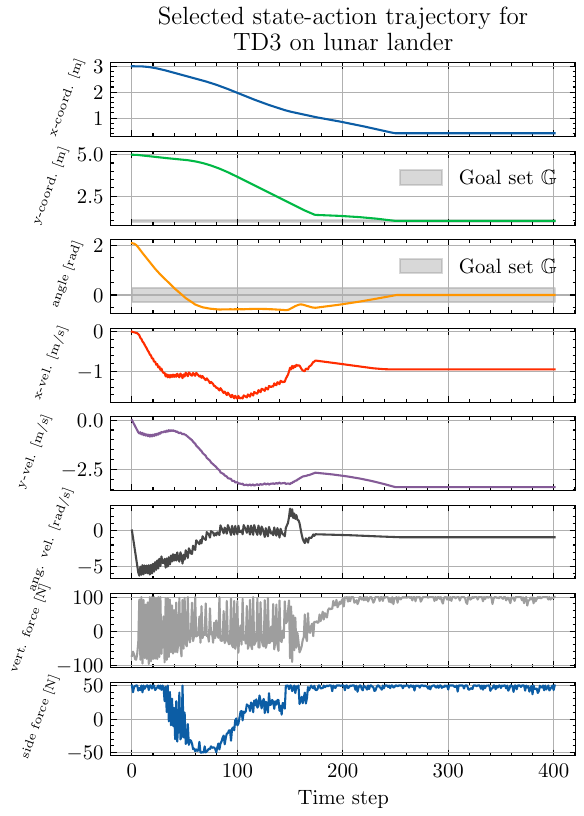}
\includegraphics[width=0.5\textwidth]{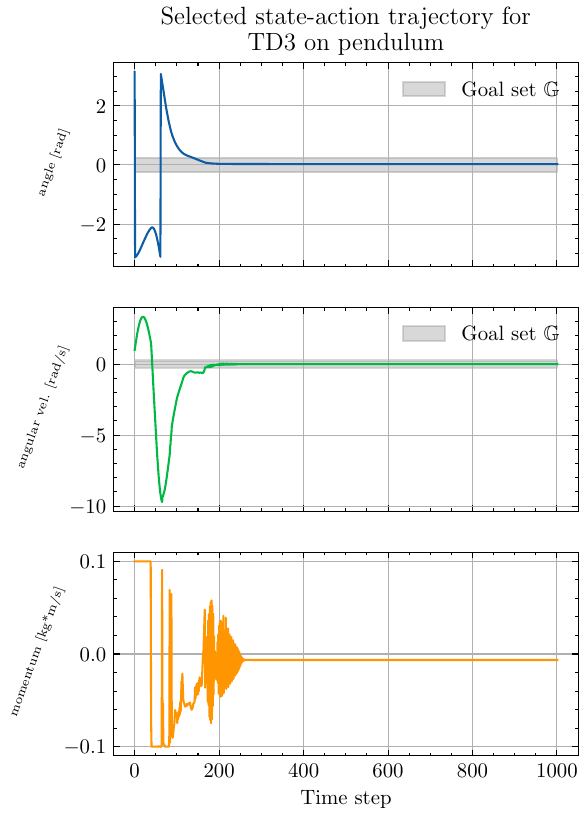}
\includegraphics[width=0.5\textwidth]{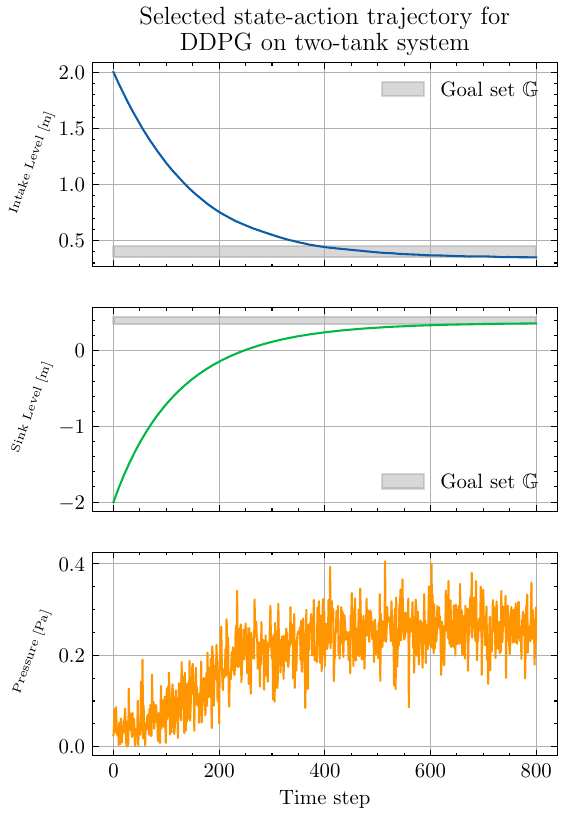}
\includegraphics[width=0.5\textwidth]{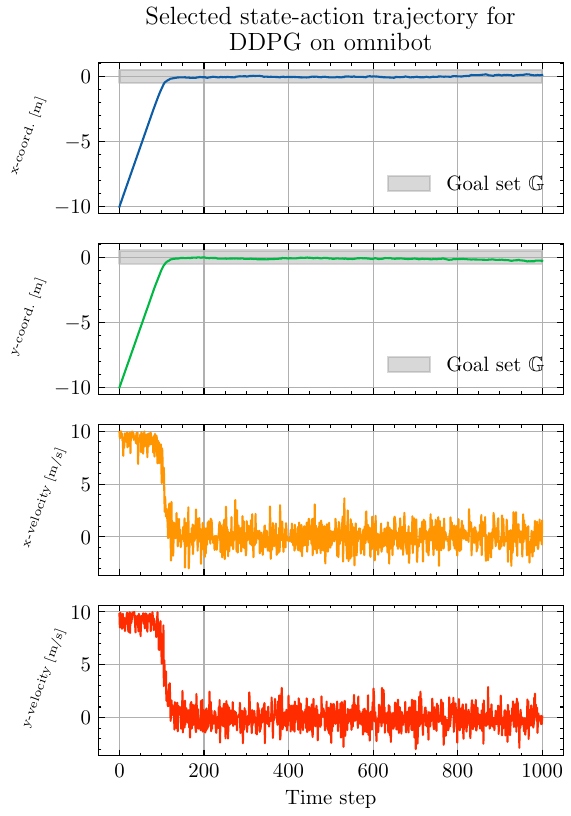}
\includegraphics[width=0.5\textwidth]{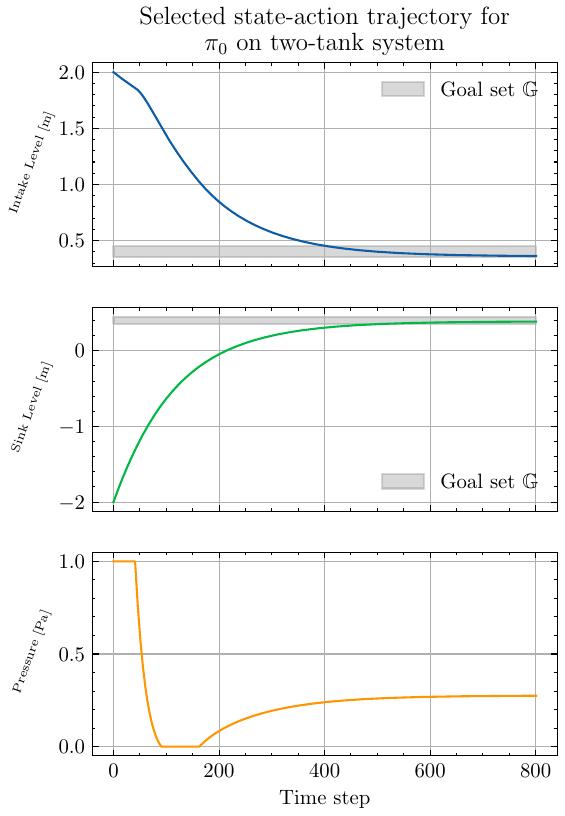}
\includegraphics[width=0.5\textwidth]{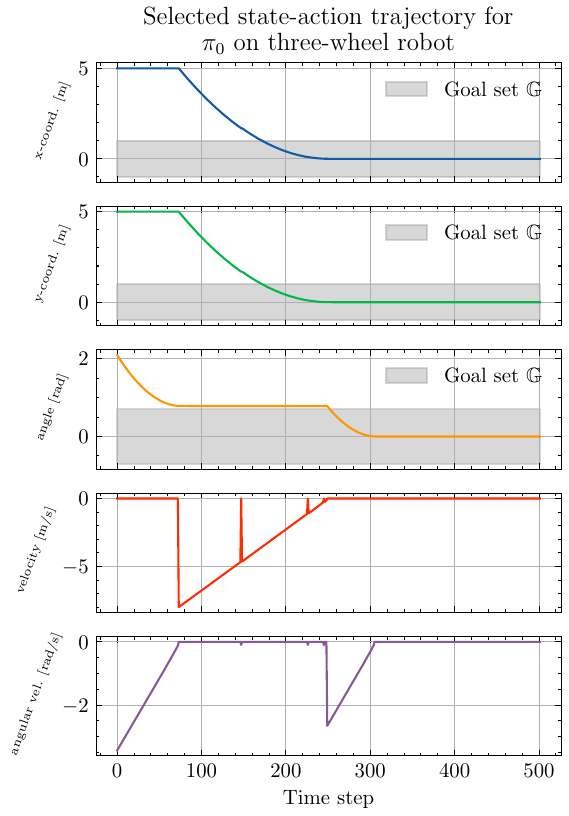}
\includegraphics[width=0.5\textwidth]{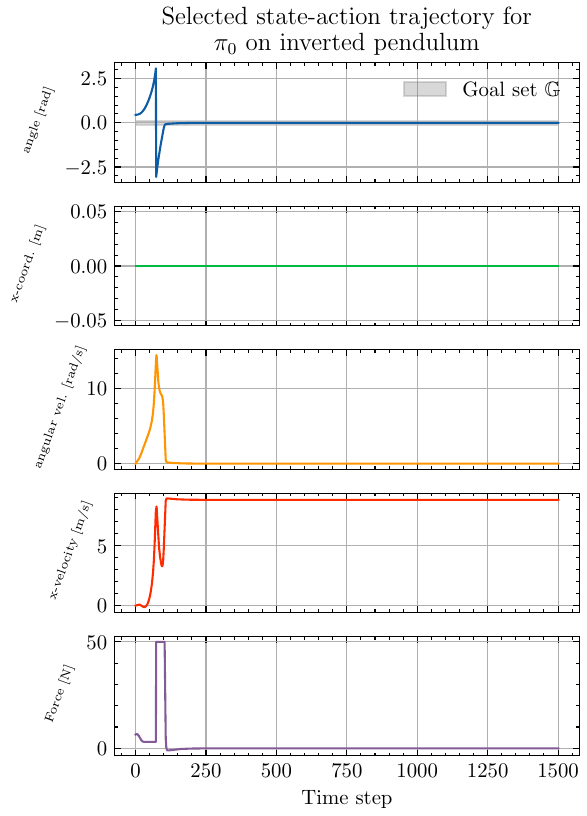}
\includegraphics[width=0.5\textwidth]{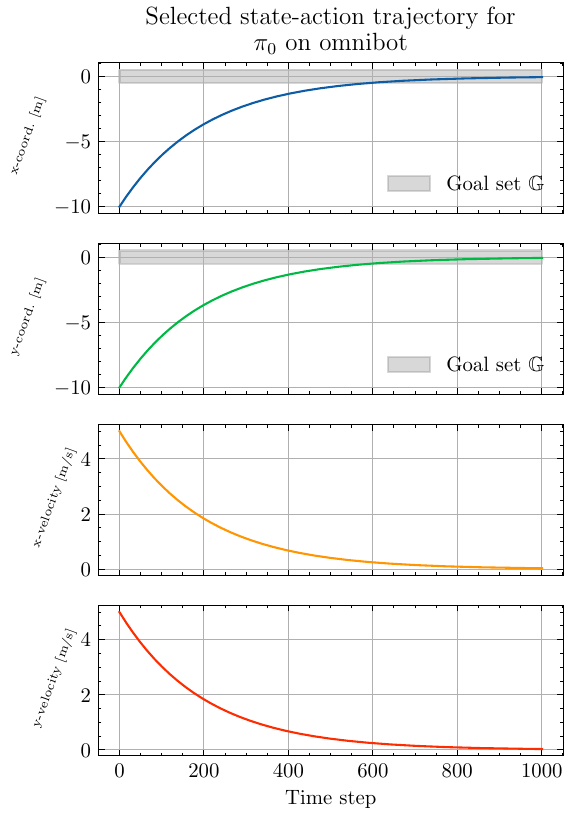}
\includegraphics[width=0.5\textwidth]{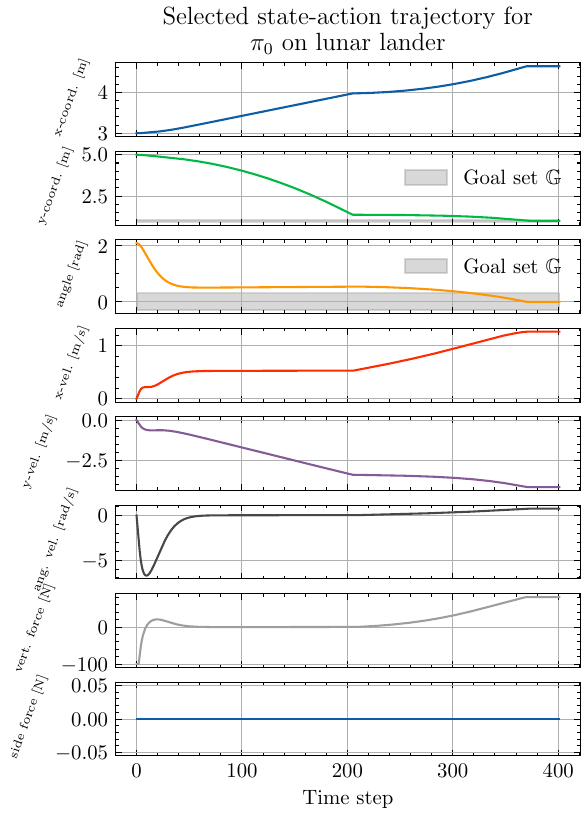}
\includegraphics[width=0.5\textwidth]{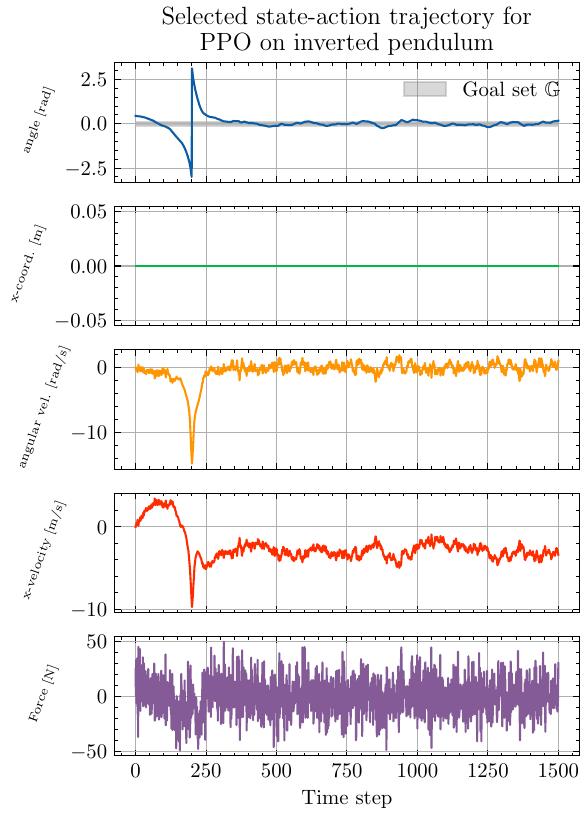}
\includegraphics[width=0.5\textwidth]{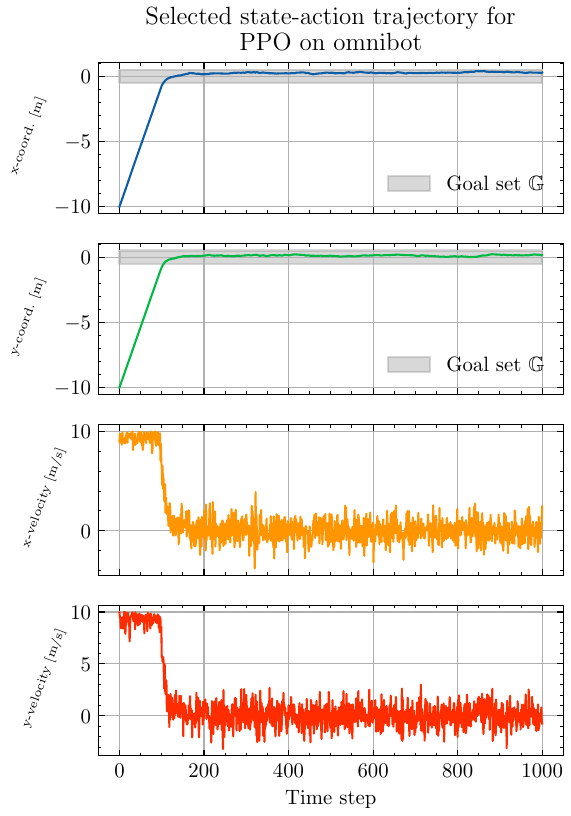}
\includegraphics[width=0.5\textwidth]{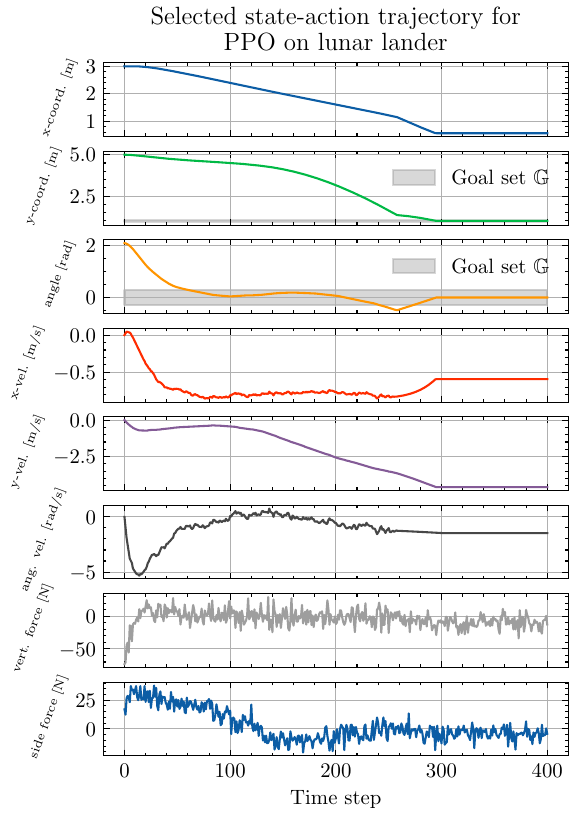}
\includegraphics[width=0.5\textwidth]{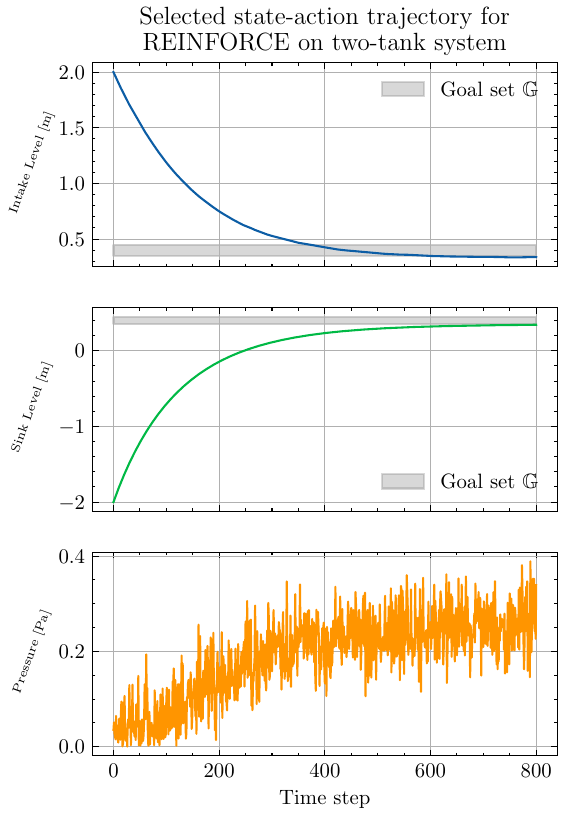}
\includegraphics[width=0.5\textwidth]{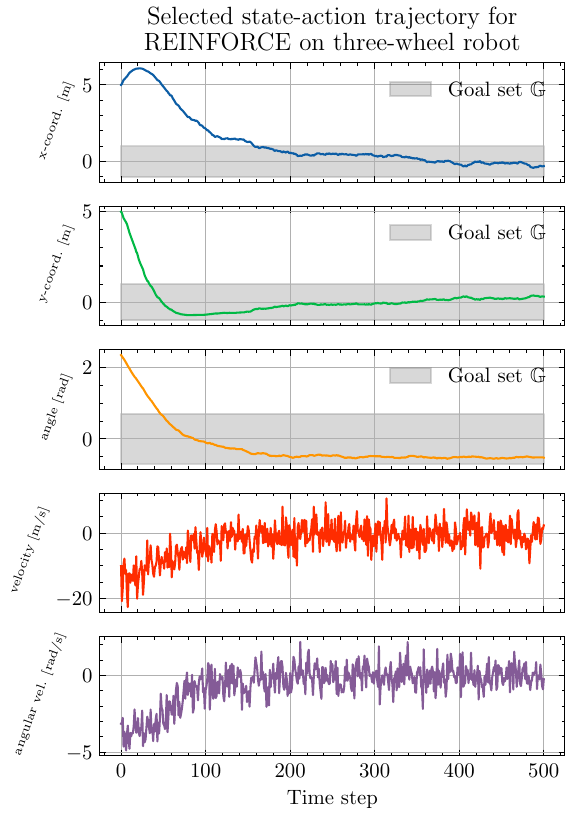}
\includegraphics[width=0.5\textwidth]{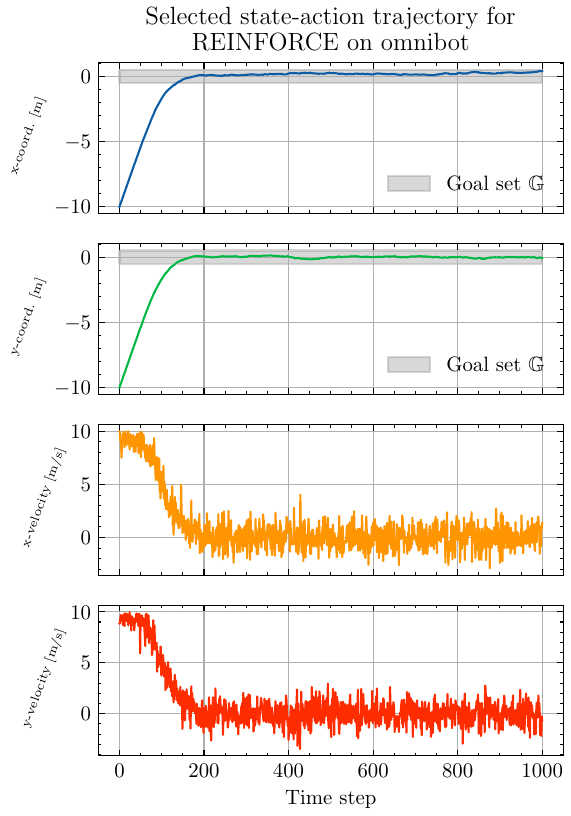}
\includegraphics[width=0.5\textwidth]{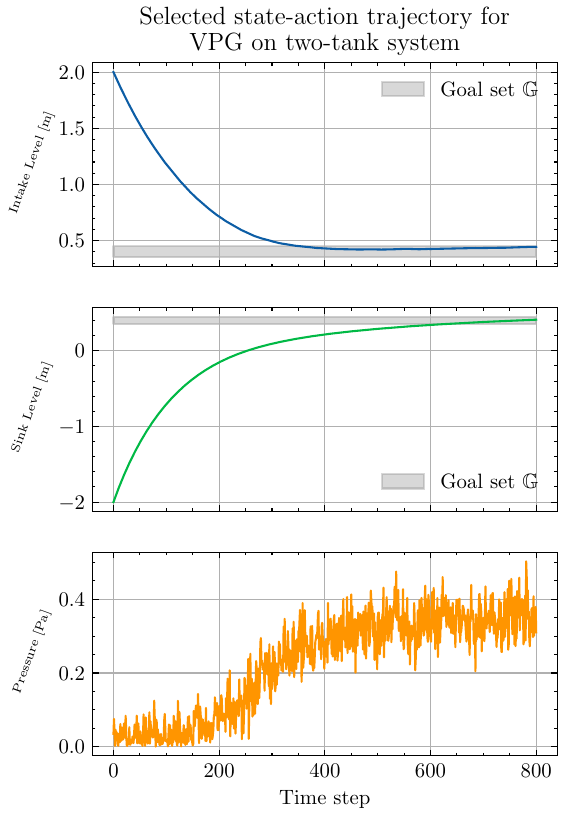}
\includegraphics[width=0.5\textwidth]{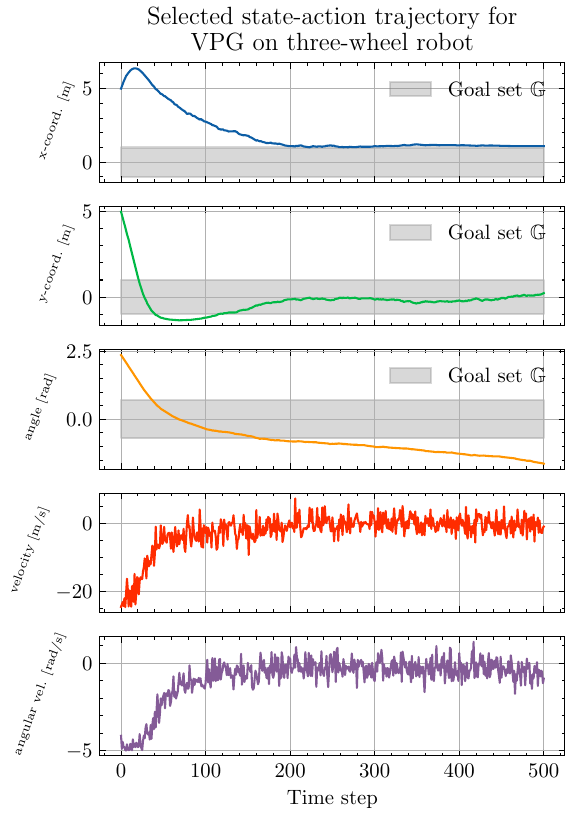}
\includegraphics[width=0.5\textwidth]{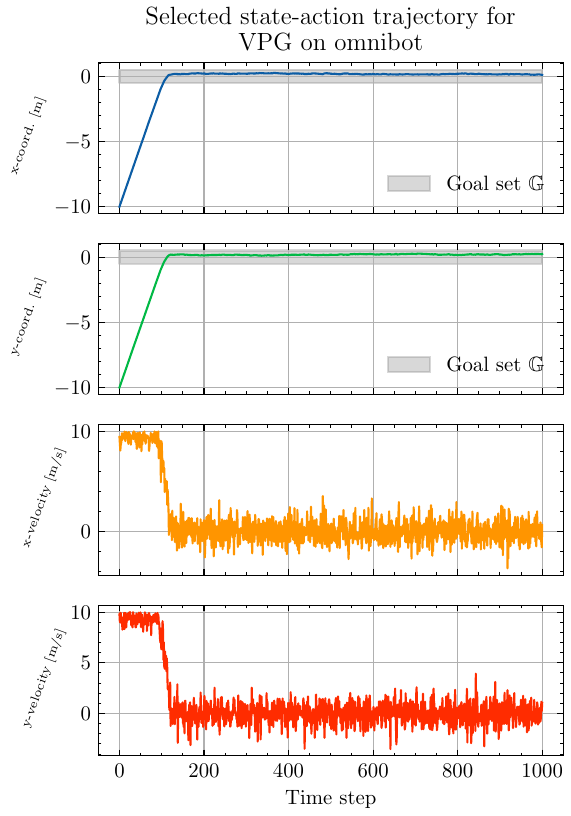}
\includegraphics[width=0.5\textwidth]{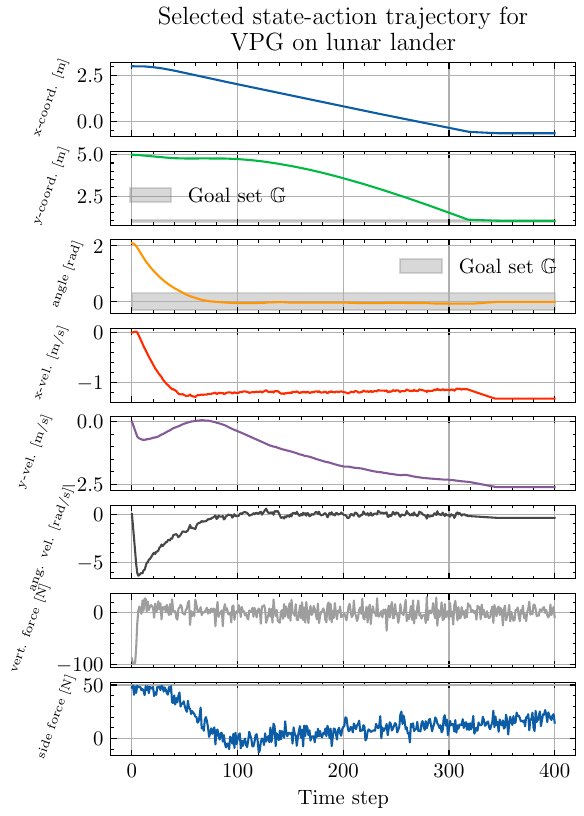}

\part{Code appendix}\label{part_codeappendix}
\section{Installation}
\label{sec_installation}
\subsection{Setup Python Virtual Environment}

Note that this project requires \textbf{Python version 3.10}.
\vspace{2\parskip}
Before running the experiments from the source, it is recommended to setup a virtual environment to isolate the project dependencies. Here's how to do it:

\begin{enumerate}
\item Install virtualenv if you haven't got it installed:
\footnotesize
\begin{verbatim}
pip install virtualenv
\end{verbatim}
\normalsize
\item Create a virtual environment:

Navigate to the root of the repo and run:
\footnotesize
\begin{verbatim}
virtualenv venv
\end{verbatim}
\normalsize
\item Activate the virtual environment:
\footnotesize
\begin{verbatim}
source venv/bin/activate
\end{verbatim}
\normalsize
\end{enumerate}

\subsection{Quick Start}

To run the experiments, follow the steps below after activating the virtual environment:

\begin{enumerate}
\item Install the required packages:

Update your system and install necessary dependencies:
\footnotesize
\begin{verbatim}
sudo apt update
sudo apt install -y libgeos-dev libqt5x11extras5 default-jre
\end{verbatim}
\normalsize 
Install the Python package in editable mode:
\footnotesize
\begin{verbatim}
pip install -e .
\end{verbatim}
\normalsize 
\item Execute the experiment script:

Run the reproduce.py script with the required parameters:
\footnotesize
\begin{verbatim}
python reproduce.py --agent={AGENT_NAME} --env={ENVIRONMENT_NAME}
\end{verbatim}
\normalsize 
Replace \{AGENT\_NAME\} and \{ENVIRONMENT\_NAME\} with the appropriate values from \Cref{sub_abbreviations}.
For example:
\footnotesize
\begin{verbatim}
python reproduce.py --agent=calf --env=pendulum
\end{verbatim}
\normalsize
The script will run the specified agent-environment configuration by executing a corresponding bash file located in the \texttt{bash/} directory. We use MLflow for experiment management. The command performs the following steps:

\begin{itemize}
\item Executes the specified run over 10 seeds by executing the corresponding bash file in the ./bash directory. Please note that this process can take a long time (several hours).
\item Extracts data from the MLflow experiment.
\item Generates plots representing the learning curve and the state-action trajectory with the greatest reward over all runs and stores the results into \texttt{goalagent/srccode\_data/plots/} directory.
\end{itemize}

If you want to rebuild the plots with already run experiment then you can use \texttt{--plots-only} flag.
\footnotesize
\begin{verbatim}
python reproduce.py --agent=calf --env=pendulum --plots-only
\end{verbatim}
\normalsize

\item To view the current progress of training, start the MLflow server:
\footnotesize
\begin{verbatim}
cd goalagent/srccode_data
mlflow ui
\end{verbatim}
\normalsize

\end{enumerate}

\section{Abbreviations}\label{sub_abbreviations}

\subsection{Environments}

\begin{itemize}
\item \texttt{--env=inverted\_pendulum} for the inverted pendulum 
\item \texttt{--env=pendulum} for the pendulum 
\item \texttt{--env=3wrobot\_kin} for the three-wheel robot
\item \texttt{--env=2tank} for the two-tank system
\item \texttt{--env=lunar\_lander} for the lunar lander
\item \texttt{--env=omnibot} for the omnibot (kinematic point)
\end{itemize}

\subsection{Agents}

\begin{itemize}
\item \texttt{--agent=calf} for our agent
\item \texttt{--agent=nominal} for $\policy_0$ 
\item \texttt{--agent=ppo} for Proximal Policy Optimization (PPO)
\item \texttt{--agent=sdpg} for Vanilla Policy Gradient (VPG)
\item \texttt{--agent=ddpg} for Deep Deterministic Policy Gradient (DDPG)
\item \texttt{--agent=reinforce} for REINFORCE
\item \texttt{--agent=sac} for Soft Actor Critic (SAC)
\item \texttt{--agent=td3} for Twin-Delayed DDPG (TD3)
\end{itemize}

\section{Code structure}
\label{sec_codestruct}

\scriptsize
\begin{verbatim} 
	goal-agent-cleanrl 
	|-- bash  -  bash files launching a specific pair agent-environment 
	|-- exman -  utils for aggregation of experiments and drawing plots. 
	|-- reproduce.py   -  main script for reproducing experiments 
	|-- presets  -  storage of configuration .yaml files for all experiments 
	|-- goalagent   -  separate run scripts for experiments
	|   
	|--calf   -  different implementations of our agent
	|--init.py
	|--agent_calf.py
	|--agent_calfq.py 
	|--utilities.py 
	|--env  -  environment specific functions and classes 
	|--utils  -  saving source code, observations, metrics evaluations
	|--calf.py   -  run script for value function variant of our agent 
	|--calfq.py  -  run script for action value function variant of our agent
	|--run_stable.py - entrypoint for PPO (our variant), VPG, DDGP, REINFORCE
	|--sac.py    - SAC run script 
	|--td3.py    - TD3 run script 
	|-- srccode  -  source code for PPO (our variant), VPG, DDGP, REINFORCE 
\end{verbatim}
\normalsize
\paragraph{./bash}

Within this subdirectory, you will find all the Bash scripts intended for running a specific pair of agent-environment configurations. You may also override seeds or any hyperparameters here to obtain different sets of runs. To override the seeds, look for the line that iterates over seeds such as 
\footnotesize
\begin{verbatim}
	+seed=0,1,2,3,...
\end{verbatim}
\normalsize	
or 
\footnotesize
\begin{verbatim}
	seed=0,1,2,3,...
\end{verbatim}
\normalsize	
and modify it as you wish.

\subsection{Comments on internal logic}

\textbf{Config files}: We utilize our own fork of Hydra to create reproducible pipelines, fully configured through a set of .yaml files that form a structured tree. To gain a deeper understanding of how our run scripts operate, we strongly recommend familiarizing yourself with Hydra and its core principles. All our run configurations are located in \texttt{presets/} folder. 
\vspace{2\parskip}
\textbf{Main script}: Every time you launch \texttt{reproduce.py} it will search and launch a respective bash script responsible for specific agent-environment pair, save learning curves and state-action trajectories. 
\vspace{2\parskip}
\textbf{Our agent}: Core logic of our agent is implemented in \texttt{./goalagent/calf} folder. 

\section{Hardware and software requirements}

\subsection{System Configuration}

All experiments were conducted on the following machine
\begin{itemize}
\item Operating system: Ubuntu 24.04 LTS
\item GPU: 2x GeForce RTX 3090
\item CPU: AMD Ryzen Threadripper 3990X 64-Core Processor
\item RAM: 128 GB
\end{itemize}

\subsection{Minimum System Requirements}

\begin{itemize}
\item RAM: 32GB or higher
\item CPU Cores: At least 16 logical cores
\item Storage space: 250 GB for all experiment data
\item GPU: Nvidia GPU with minimum of 10 GB of memory
\end{itemize}

\subsection{Software Requirements}
Although the installation steps are detailed in \Cref{sec_installation}, we formally provide a summary of the dependencies here.
\subsubsection{Python Dependencies}
All necessary Python dependencies are listed in the \texttt{pyproject.toml} file.
\subsubsection{Non-Python Dependencies}
Non-Python Dependencies can be installed using the following command:
\footnotesize
\begin{verbatim}
sudo apt install -y libgeos-dev libqt5x11extras5 default-jre
\end{verbatim}
\normalsize



\pagebreak

\printbibliography


\end{document}